\documentclass[10pt,twocolumn,letterpaper]{article}

\usepackage{cvpr}              %

\usepackage[accsupp]{axessibility}  %

\definecolor{cvprblue}{rgb}{0.21,0.49,0.74}
\usepackage[pagebackref,breaklinks,colorlinks,allcolors=cvprblue]{hyperref}

\usepackage{soul}
\usepackage{lscape}
\usepackage{amsthm}
\usepackage{dsfont}
\usepackage{amsmath}
\usepackage{amssymb}
\usepackage{amsfonts}
\usepackage{booktabs}
\usepackage{multirow}
\usepackage{subcaption}
\usepackage[table]{xcolor}

\newtheorem{theorem}{Theorem}
\newtheorem{assumption}{Assumption}

\newcommand{\approach}{{\texttt{Catalyst}}}

\def\paperID{38920} %
\def\confName{CVPR}
\def\confYear{2026}

\title{Catalyst: Out-of-Distribution Detection via Elastic Scaling}

\author{Abid Hassan, Tuan Ngo, Saad Shafiq, Nenad Medvidovic \\
University Southern California, Los Angeles\\
{\tt\small\{mdskabid, tkngo, sshafiq, neno\}@usc.edu} \\
}

\begin{document}
\maketitle

\begin{abstract}
\label{sec: abstract}
    
    Out-of-distribution (OOD) detection is critical for the safe deployment of deep neural networks. State-of-the-art post-hoc methods typically derive OOD scores from the output logits or penultimate feature vector obtained via global average pooling (GAP). We contend that this exclusive reliance on the logit or feature vector discards a rich, complementary signal: the raw channel-wise statistics of the pre-pooling feature map lost in GAP. In this paper, we introduce $\approach$, a post-hoc framework that exploits these under-explored signals. $\approach$ computes an input-dependent scaling factor ($\gamma$) on-the-fly from these raw statistics (e.g., mean, standard deviation, and maximum activation). This $\gamma$ is then fused with the existing baseline score, multiplicatively modulating it -- an ``elastic scaling'' -- to push the ID and OOD distributions further apart. We demonstrate $\approach$ is a generalizable framework: it seamlessly integrates with logit-based methods (e.g., Energy, ReAct, SCALE) and also provides a significant boost to distance-based detectors like KNN. As a result, $\approach$ achieves substantial and consistent performance gains, reducing the average False Positive Rate by 32.87\% on CIFAR-10 (ResNet-18), 27.94\% on CIFAR-100 (ResNet-18), and 22.25\% on ImageNet (ResNet-50). Our results highlight the untapped potential of pre-pooling statistics and demonstrate that $\approach$ is complementary to existing OOD detection approaches. Our code is available here: https://github.com/bingabid/Catalyst

\end{abstract}

\section{Introduction}
\label{sec: introduction}

    A deep neural network deployed in real-world environments will inevitably encounter out-of-distribution (OOD) samples drawn from novel contexts whose class labels are disjoint from the training distribution, referred as in-distribution (ID) data. Unlike ID samples that the model was trained on, these OOD instances should not be confidently classified %
    but be detected and flagged %
    for human review. Robust OOD detection is particularly crucial for safety-critical applications where erroneous predictions can have severe consequences, e.g., in medical diagnosis~\cite{medical_diagnosis_01, medical_diagnosis_02} or autonomous driving~\cite{autonomous} systems.

    Early methods to OOD detection primarily focused on designing scoring functions to distinguish ID from OOD samples. The seminal work~\cite{msp} proposed using the maximum softmax probability (MSP) as a confidence measure, based on observation that OOD samples yield lower softmax scores. However, subsequent studies~\cite{msp_failed_01, msp_failed_02} revealed a critical flaw: neural networks often produce overconfident softmax predictions even for far-OOD inputs, rendering MSP unreliable. To address this, Energy~\cite{energy} introduced the energy-based score, which maps inputs to a scalar value such that ID samples yield lower energy than OOD samples. This score provided a more robust uncertainty measure, inspiring a series of improvements aimed at enhancing ID-OOD separability. Recent advances have focused on post-hoc activation manipulation to amplify this separation. Notable methods include ReAct~\cite{ReAct}, DICE~\cite{DICE}, ASH~\cite{ASH}, SCALE~\cite{SCALE} achieving state-of-the-art performance.

    \begin{figure*}[ht]
        \centering
        \includegraphics[width=0.97\textwidth]{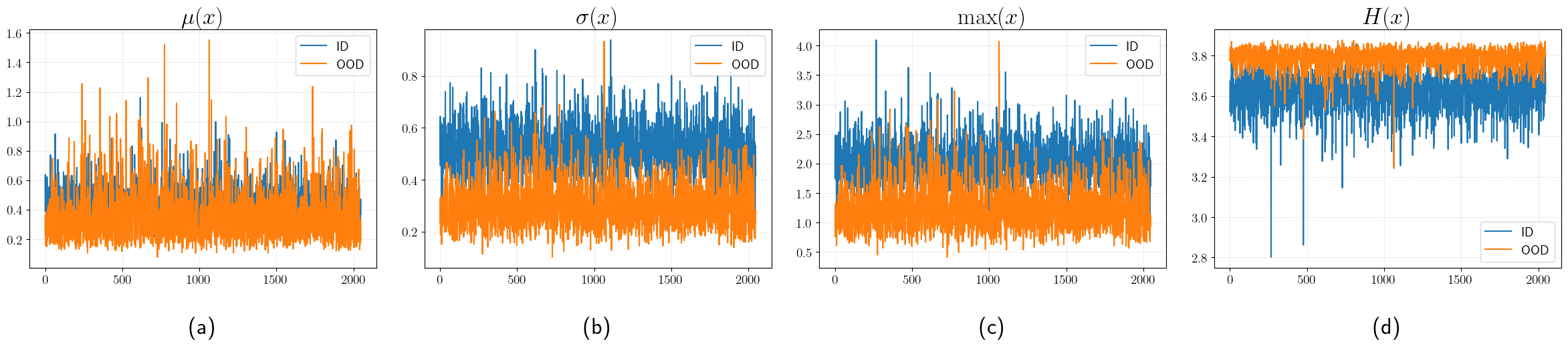}
        \vspace{-3mm}
        \caption{\textit{Information cues from each channel before the penultimate layer of a ResNet-50 trained on ImageNet-1k, evaluated with Texture as the OOD dataset. The x-axis shows channel indices; the y-axis shows cue strength. \textbf{\emph{Left to right}}: (a) $\mu(\mathbf{x})$: mean activation, (b) $\sigma(\mathbf{x})$: standard deviation, (c) $\max(\mathbf{x})$: dominant activation, and (d) $H(\mathbf{x})$: entropy per channel.}}
        \label{fig: feature_plot}
     \vspace{-2mm}
    \end{figure*}

    These methods share a common paradigm: they derive their scores using the penultimate feature vector (generally obtained via GAP) as their foundational input. These techniques process this feature vector to derive energy-based scores~\cite{energy, ReAct, SCALE} or distance-based scores~\cite{knn,ssn}. %
    We contend that exclusive reliance on the feature vector creates an information bottleneck, as it discards complementary signals, namely the raw channel-wise statistics of the pre-pooling feature map, which could otherwise be used in tandem with existing methods for improved OOD detection.

    Figure~\ref{fig: feature_plot} illustrates the distribution of these untapped information cues, extracted from the penultimate layer's pre-pooling activation map in an ImageNet-trained ResNet-50, using Textures as the OOD dataset. In exemplary visualization, we observed that pre-pooled activation map encode important channel-specific characteristics that exhibit discriminative attributes between ID (blue) and OOD (orange) samples. Each point on the x-axis corresponds to a single channel, while the y-axis represents the strength of four statistical cues: (a)~mean, (b)~standard deviation, (c)~maximum activation, and (d)~entropy values.

    The existing methods have under-explored these distinctive statistical information. The approach exclusively relies on a score derived from the output logits~\cite{msp,odin,energy,ReAct,ASH,SCALE}: discards potent raw cues (e.g., standard deviation, maximum) and fails to leverage independent discriminative power of raw mean statistics. To address this critical limitation, we propose $\approach$, a simple yet powerful framework that computes an input-dependent \emph{scaling factor} ($\gamma$) designed to be fused in tandem with an existing scoring function. This scaling factor is computed on-the-fly, leveraging these distribution-sensitive cues embedded in the pre-pooled activation maps. $\approach$ is designed to integrate seamlessly with established approaches while significantly improving their ability to distinguish between ID and OOD data. Our key contributions are:

     \begin{enumerate}
        \item $\approach$, a complementary post-hoc OOD detection  framework that leverages pre-pooling channel-wise statistics to augment existing methods, generalizing across architectures like ResNet, DenseNet, and MobileNet.
    
        \item An extensive evaluation showing $\approach$ complements and substantially improves established competitive baselines. Specifically, on the ImageNet benchmark, $\approach$ reduces average FPR95 by 22.25\% using ResNet-50. On CIFAR benchmarks, it reduces FPR95 by 32.87\% on CIFAR-10 and 27.94\% on CIFAR-100 using ResNet-18.

        \item Statistical analysis (Appendix~\ref{appendix: analysis}) and extensive ablation studies (Section~\ref{sec: ablation study}) validate our design choices.

    \end{enumerate}

\section{Preliminaries}
\label{sec: preliminaries}
    \looseness-1
    \textbf{Setup.} This paper focuses on the post-hoc analysis of multiclass classification in supervised settings. Let $\mathcal{X}$ denote the input space and $\mathcal{Y} = \{1, 2, \cdots, C\}$ the output label space. A neural network $\theta: \mathcal{X} \rightarrow \mathbb{R}^{|\mathcal{Y}|}$ is trained on a dataset $\mathcal{D} = {(\mathbf{x}_i, y_i)}_{i=1}^N$ drawn \emph{i.i.d.} from an unknown joint distribution $\mathcal{P_{XY}}$ over $\mathcal{X} \times \mathcal{Y}$. The network outputs a logit vector, which is used to predict the label of an input sample. $\mathcal{D}_{\text{in}}$ represents the marginal distribution of $\mathcal{P_{XY}}$ over $\mathcal{X}$, corresponding to the ID data.

    \noindent
    \textbf{Scoring Function.} As introduced in Section~\ref{sec: introduction}, the core challenge in OOD detection lies in designing effective scoring functions that reliably distinguish between ID and OOD samples. The evolution of scoring functions began with the MSP~\cite{msp} approach and progressed to more robust energy-based scores~\cite{energy}. While other scoring functions exist (e.g., ODIN~\cite{odin}, Mahalanobis~\cite{maha_distance}, KNN~\cite{knn}), we focus on the energy-based score $S_{\texttt{energy}}(\mathbf{x}; \theta)$ due to its prevalence, superior performance and simplicity~\cite{energy, ReAct, DICE, ASH, SCALE}. Without loss of generality, all subsequent mentions of ``score" refer to $S_{\texttt{energy}}(\mathbf{x}; \theta)$ unless specified otherwise. We adopt the negative free energy formulation from \cite{energy}. Formally, given a logit vector \(f(\mathbf{x}) \in \mathbb{R}^C\) produced by the model \(\theta\), the scoring functions is defined as:
    \begin{small}
        \vspace{-2mm}
            \begin{equation}
                S_{\texttt{energy}}(\mathbf{x}; \theta)  = \log\left(\sum_{j=1}^C e^{f_j(\mathbf{x})}\right)
                \label{eq: scoring_energy}
            \end{equation}
           \vspace{-2mm}
    \end{small}
    
    \noindent
    \textbf{Out-of-distribution Detection.} At inference time, the model $\theta$ operating in real-world will inevitably encounter OOD samples $\mathcal{D}_{\text{out}}$ whose label sets are disjoint from $\mathcal{Y}$. These samples should not be confidently predicted by $\theta$ as one of the known classes, instead necessitating robust OOD detection. Formally, we frame OOD detection as learning a decision boundary $G_{\lambda}(\mathbf{x} ; \theta)$  that classifies a test sample $\mathbf{x} \in \mathcal{X}$ as either ID or OOD:
    \vspace{-2mm}
     \begin{equation}
            \resizebox{0.90\columnwidth}{!}{$
                G_{\lambda}(\mathbf{x} ; \theta)  =
                \begin{cases}
                \mathrm{ID} & \text{if } \mathbf{x} \sim \mathcal{D}_{\text{in}} \\
                \mathrm{OOD} & \text{if } \mathbf{x} \sim \mathcal{D}_{\text{out}}
                \end{cases} \\
                  =
                \begin{cases}
                    \mathrm{ID}  & \text{if } S(\mathbf{x}; \theta) \ge \lambda \\
                    \mathrm{OOD} & \text{if } S(\mathbf{x}; \theta) < \lambda
                \end{cases}
                $}
        \label{eq: decision_boundary}
         \vspace{-2mm}
   \end{equation} 
    where $S(\mathbf{x}; \theta)$ represents a downstream OOD scoring function, and by convention~\cite{energy} $\lambda$ is a threshold calibrated such that 95\% of ID data ($\mathcal{D}_{\text{in}}$) is correctly classified. 
   
    \noindent
    \textbf{Evaluation metrics.} In line with standard evaluation protocol in OOD detection~\cite{energy},  we evaluate the performance of $\approach$ using two key metrics: FPR95 and AUROC:
    \begin{enumerate}
        \item \textit{FPR95} measures the False Positive Rate when 95\% of in-distribution (ID) samples are correctly classified. A lower FPR95 $(\downarrow)$ indicates better OOD detection performance.
    
        \item \textit{AUROC} is a threshold-free metric that computes the area under the receiver operating characteristic curve. Higher AUROC $(\uparrow)$ signifies superior discriminative capability.%
    \end{enumerate}

    \begin{figure*}[ht]
        \centering
        \includegraphics[width=0.90\textwidth]{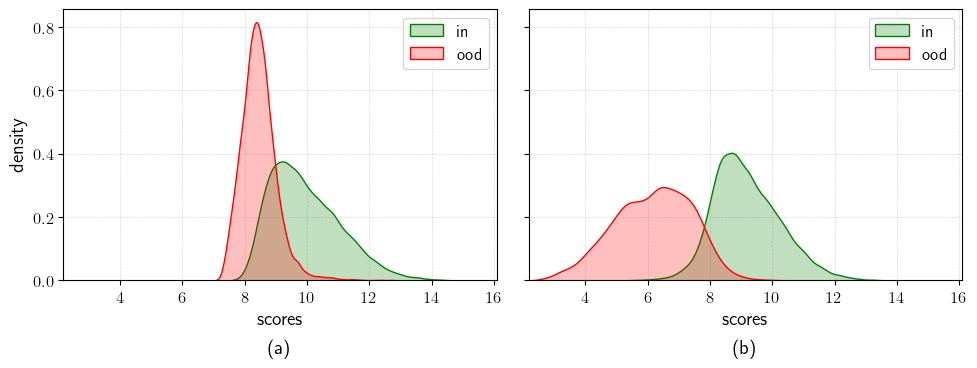}
        \vspace{-3mm}
        \caption{\textit{Illustration of $\approach$'s effectiveness. The model is ResNet-50 trained on ImageNet-1k, evaluated on Texture (OOD). Here, we apply $\gamma$ computed from the channel-maximum statistic (m) multiplicatively to the baseline ReAct. (a) The unscaled score distribution shows more significant overlap than (b) the $\approach$-scaled score distribution.}}
        \label{fig: density_plot}
        \vspace{-2mm}
    \end{figure*}

\section{Methodology}
\label{sec: method}

    The key contribution of this paper is $\approach$, a novel elastic scaling mechanism for enhanced OOD detection. We propose an input-dependent scaling factor ($\gamma$), derived from the overlooked channel-wise statistics of the penultimate layer's pre-pooling activation map. When this factor is fused with a baseline score, it significantly enhances the separability between ID and OOD samples. Figure~\ref{fig: density_plot} illustrates this effect with a trend representative of what we observe across the diverse models and OOD datasets in our evaluation. For instance, in the specific case depicted using a ResNet-50 trained on ImageNet, we see that while baseline score distributions for ID and Texture (OOD) data exhibit significant distributional overlap (Figure~\ref{fig: density_plot}a), multiplicatively fusing \(\gamma\) markedly reduces this overlap, enabling a much clearer separation (Figure~\ref{fig: density_plot}b). 

    In this section, we describe the method to compute input-dependent $\gamma$ and how it is fused with score. Finally, we discuss the compatibility of $\approach$ with other scoring functions and its integration with existing baselines.

\subsection{Computing the Scaling Factor $\gamma$}
\label{methodology: scaling factor}

    \looseness-1
    To compute  scaling factor \(\gamma\), we consider
    a trained DNN \(\theta: \mathbb{R}^d \rightarrow \mathbb{R}^C\) that maps an input \(\mathbf{x} \in \mathbb{R}^d\) to a logit vector \(f(\mathbf{x}) \in \mathbb{R}^C\), where \(C = |\mathcal{Y}|\) denotes the number of classes. The network’s penultimate layer produces a feature vector \(h(\mathbf{x}) \in \mathbb{R}^n\) by applying GAP operation to the activation map \(g(\mathbf{x}) \in \mathbb{R}^{n \times k \times k}\). Here, \(n\) is the number of channels, and each channel has spatial resolution \(k \times k\). A weight matrix \(\mathbf{W} \in \mathbb{R}^{n \times C}\) projects \(h(\mathbf{x})\) to the final logit vector.

    In this work, we deliberately focus on this activation map $g(\mathbf{x})$ as our source of statistics. As we will empirically demonstrate in our ablation study (Section~\ref{subsec: penultimate layer}, Appendix~\ref{appendix: penultimate layer}), this specific layer provides the most potent and reliable discriminative information cues for $\gamma$. The \emph{earlier layers} provide less informative signal with high ID/OOD overlap.

    $\approach$ is built upon the core insight, illustrated in Figure~\ref{fig: feature_plot}, that the existing baselines' exclusive reliance on the feature vector fails to leverage valuable, channel-specific statistical information. Building upon this, we identify and extract three key statistical cues from $g(\mathbf{x})$:

    \begin{itemize}
        \item \textit{Channel Mean} [\(\mu(\mathbf{x}) \in \mathbb{R}^n\)] is equivalent to the penultimate feature vector \(h(\mathbf{x})\) obtained via GAP.\footnote{We use \(\mu(\mathbf{x})\) and \(h(\mathbf{x})\) interchangeably.}
        \item \textit{Channel Standard Deviation} [\(\sigma(\mathbf{x}) \in \mathbb{R}^n\)] measures the spatial variability of activations within each channel.
        \item \textit{Channel Maximum} [\(m(\mathbf{x}) \in \mathbb{R}^n\)] captures the peak activation response in each channel.
    \end{itemize}
        
    The information cues \(\mu(\mathbf{x}), \sigma(\mathbf{x})\), and \(m(\mathbf{x}) \) for OOD samples may exhibit extreme unit activations. Prior work~\cite{ReAct} presented a similar phenomenon of abnormally high unit activations that result in overconfident predictions for OOD samples, subsequently distorting the energy score. Extreme values in \(\mu(\mathbf{x}), \sigma(\mathbf{x})\), and \(m(\mathbf{x}) \) can similarly distort scaling factor \(\gamma\) for OOD samples. To mitigate this effect, we introduce a clipping mechanism that bounds each statistic by a threshold \(c > 0\). Specifically, for each input, we compute rectified features via element-wise clipping:
    \begin{small}
        \vspace{-1mm}
        \begin{equation}
             \bar{f}(\mathbf{x}) = \min(f(\mathbf{x}), c)
             \label{eq: gamma_c}
           \vspace{-1mm}
        \end{equation}
    \end{small}
    where \(f(\mathbf{x}) \in \{\mu(\mathbf{x}), \sigma(\mathbf{x}), m(\mathbf{x})\}\). This operation ensures that activation values are capped at \(c\), preventing them from disproportionately influencing \(\gamma\). The rectified vectors are the basis for \(\gamma\)'s calculation: 
    \begin{small}
        \vspace{-3mm}
        \begin{equation}
            \gamma(\mathbf{x};f) = \sum_{i=1}^n \bar{f}_i(\mathbf{x})
            \label{eq: gamma_design}
           \vspace{-2mm}
        \end{equation} 
    \end{small}             
    where %
    the subscript \(i\) denotes the \(i\)-th channel. The selection of this clipping threshold \(c\) is discussed in Section~\ref{subsec: hyper-parameter selection} and detailed in Appendix~\ref{appendix: reproducibility}.

    While we primarily focus on $\mu(\mathbf{x}), \sigma(\mathbf{x})$ and $m(\mathbf{x})$, our framework readily accommodates other channel-wise statistics derived from $g(\mathbf{x})$, such as entropy and median. %
    We provide a detailed comparative analysis of these cues in our ablation study (Section~\ref{subsec: alternate statistics}; Appendix~\ref{appendix: alternate statistics}), which justifies our design and validates our focus on $\mu(\mathbf{x}), \sigma(\mathbf{x})$, and $m(\mathbf{x})$ as robust and generalizable set of statistics for computing $\gamma$.

\subsection{Elastic Scaling of the Score}

    To create a more discriminative score, we dynamically re-calibrate the baseline score $ S(\mathbf{x}; \theta) $ using scaling factor $\gamma$. We explore the two fusion strategies: \textit{multiplicative} and \textit{additive}. We term the multiplicative strategy  \textit{``Elastic Scaling"} because it truly scales (i.e., multiplies) the baseline score, elastically stretching or shrinking it based on the $\gamma$. The additive approach, in contrast, is a simple offset or shift, not a scaling. These are defined in Equation~\ref{eq: adaptive_scaling}: %
    \begin{small}
        \vspace{-1.5mm}
         \begin{subequations}
            \begin{align}
            S^{*}_{\texttt{mul}}(\mathbf{x}; \theta, \gamma) &= \gamma(\mathbf{x}; f) \times S(\mathbf{x}; \theta) \label{eq: adaptive_scaling_mul} \\
            S^{+}_{\texttt{add}}(\mathbf{x}; \theta, \gamma) &= \gamma(\mathbf{x}; f) + S(\mathbf{x}; \theta) \label{eq: adaptive_scaling_add}
            \end{align}
            \label{eq: adaptive_scaling}
        \end{subequations} 
    \end{small}    
    where $\gamma(\mathbf{x}; f)$ is the scaling factor computed from an information cue $f(\mathbf{x}) \in \{\mu(\mathbf{x}), \sigma(\mathbf{x}), m(\mathbf{x})\}$.

    While our analysis in Section~\ref{subsec: fusion strategy} %
    shows that both strategies can achieve similar peak performance, we adopt multiplicative fusion (Eq.~\ref{eq: adaptive_scaling_mul}) as our primary framework. This choice is not arbitrary, as we demonstrate that the additive approach, while effective, is operationally fragile due to its hyperparameter sensitivity. The multiplicative fusion provides not only competitive performance but also the practical robustness and stability required of a general-purpose usage. Therefore, in the remainder of this paper, we will refer multiplicative fusion as \textit{Elastic Scaling}. This final re-calibrated score is subsequently used in the decision rule defined in Equation~\ref{eq: decision_boundary} to classify as ID or OOD.

\subsection{Generalizability of $\approach$} 

    While our analysis primarily focuses on energy-based scoring (given its primary role for competitive methods like ReAct and SCALE), $\approach$ is a general framework. It can be seamlessly integrated with other scoring functions such as MSP~\cite{msp}, ODIN~\cite{odin}, and KNN~\cite{knn,ssn} -- by replacing the baseline score $S(\mathbf{x}; \theta)$ in equation (Eq.~\ref{eq: adaptive_scaling_mul}) with the alternate score.

    Additionally, this elastic scaling retains all advantages of post-hoc methods while transforming scores into a more discriminative metric. $\approach$ is designed to complement existing techniques, including Energy, ReAct, DICE, ASH, SCALE, and KNN. In Appendix~\ref{appendix: analysis}, we provide a formal characterization of why $\approach$ enhances ID-OOD separability, offering deeper insight.

\section{Experiments}
\label{sec: experiments}

    In this section, we evaluate the efficacy of $\approach$ across a diverse set of OOD datasets. We begin with an in-depth empirical analysis on standard CIFAR benchmarks. We then extend our evaluation to a large-scale OOD detection setting using ImageNet, demonstrating the versatility and robustness of $\approach$. Our evaluation does not assume the availability of an OOD validation set and incorporates a wide range of OOD datasets to provide a realistic assessment of $\approach$.
    
    We use the Energy score~\citep{energy} as our default baseline. For brevity, when $\approach$ is applied to Energy, we simply denote it as $\approach$. When applying $\approach$ to other baselines or scoring functions, we state it explicitly (e.g., $\approach$ + \texttt{ReAct}, $\approach$ + \texttt{KNN}).

    We ensure a fair and direct comparison against prior work. As the architectures used in our evaluation (e.g., ResNet-18 for the CIFAR benchmarks; ResNet-34 and DenseNet-121 for ImageNet) were not included in the primary baselines like ReAct, DICE, ASH, and SCALE, we undertook a rigorous re-evaluation of these methods. We carefully followed the official hyperparameter selection protocols and open-sourced implementations from their respective papers to ensure the integrity of our comparisons.

\subsection{CIFAR Evaluation}

    \begin{table}[ht]
        \centering
        \resizebox{\columnwidth}{!}{%
        \begin{tabular}{l l  cc  cc }
            \toprule
            \multirow{2}{*}{\textbf{Model}} & \multirow{2}{*}{\textbf{Method}} & \multicolumn{2}{c}{\textbf{CIFAR-10}} & \multicolumn{2}{c}{\textbf{CIFAR-100}}  \\
            \cmidrule(lr){3-4} \cmidrule(lr){5-6}
            & & \textbf{FPR95} $\downarrow$ & \textbf{AUROC} $\uparrow$ & \textbf{FPR95} $\downarrow$ & \textbf{AUROC} $\uparrow$ \\
            \midrule
            \multirow{14}{*}{\raisebox{5ex}{\rotatebox{90}{ResNet-18}}} & MSP          & 58.33 & 91.28  & 79.92 & 76.66   \\
                                                                        & ODIN         & 28.98 & 95.16 & 66.06 & 84.78  \\
                                                                        & Energy       & 35.50 & 94.17 & 70.21 & 83.54   \\
                                                                        & ReAct        & 29.76 & 95.19 & 57.76 & 87.97   \\
                                                                        & DICE         & 30.98 & 94.69 & 55.66 & 85.97   \\
                                                                        & ReAct+DICE   & 19.65 & 96.50 & 48.23 & 89.06   \\
                                                                        & ASH          & 20.96 & 95.95 & 49.52 & 86.86   \\
                                                                        & SCALE        & 21.05 & 96.19 & 48.10 & 88.70    \\
            \cmidrule(lr){2-6}
            \rowcolor[gray]{0.9}         & \textbf{$\approach(\mu)$}                   & 24.85 & 95.74 & 52.93 & 87.46 \\
            \rowcolor[gray]{0.9}         & \textbf{$\approach(\sigma)$}                & 17.72 & 96.89 & 46.29 & 89.18 \\
            \rowcolor[gray]{0.9}         & \textbf{$\approach(m)$}                     & 16.59 & 97.10 & 45.96 & 89.37 \\
            \rowcolor[gray]{0.9}         & \textbf{$\approach(\mu)+\texttt{ReAct}$}    & 19.88 & 96.41 & 41.93 & 89.99 \\
            \rowcolor[gray]{0.9}         & \textbf{$\approach(\sigma)+\texttt{ReAct}$} & 14.25 & 97.42 & 35.15 & 91.48 \\
            \rowcolor[gray]{0.9}         & \textbf{$\approach(m)+\texttt{ReAct}$}      & \textbf{13.19} & \textbf{97.59} & \textbf{34.66} & \textbf{91.70} \\
            \midrule
            \multirow{14}{*}{\raisebox{7ex}{\rotatebox{90}{DenseNet-101}}}  & MSP          & 45.43 & 92.43 & 77.47 & 74.80 \\
                                                                            & ODIN         & 19.37 & 96.06 & 57.67 & 84.00 \\
                                                                            & Energy       & 22.41 & 95.43 & 58.92 & 83.87 \\
                                                                            & ReAct        & 17.13 & 96.61 & 52.89 & 87.18 \\
                                                                            & DICE         & 14.52 & 96.74 & 40.98 & 87.92 \\
                                                                            & ReAct+DICE   & 10.26 & 97.94 & 34.64 & 91.17 \\
                                                                            & ASH          & 11.71 & 97.44 & 35.84 & 90.85 \\
                                                                            & SCALE        & 19.88 & 96.01 & 38.31 & 90.46 \\
            \cmidrule(lr){2-6}
            \rowcolor[gray]{0.9}         & \textbf{$\approach(\mu)$}                       & 13.73 & 97.14 & 41.42 & 89.45 \\
            \rowcolor[gray]{0.9}         & \textbf{$\approach(\sigma)$}                    & 10.93 & 97.71 & 37.98 & 90.48 \\
            \rowcolor[gray]{0.9}         & \textbf{$\approach(m)$}                         & 10.71 & 97.77 & 36.79 & 90.83 \\
            \rowcolor[gray]{0.9}         & \textbf{$\approach(\mu)+\texttt{ReAct}$}        & 10.24 & 97.85 & 29.36 & 92.56 \\
            \rowcolor[gray]{0.9}         & \textbf{$\approach(\sigma)+\texttt{ReAct}$}     &  8.49 & 98.21 & 29.05 & 92.78 \\
            \rowcolor[gray]{0.9}         & \textbf{$\approach(m)+\texttt{ReAct}$}          &  \textbf{8.42} & \textbf{98.26} & \textbf{28.06} & \textbf{93.06} \\
            \bottomrule
            \end{tabular}
        }
        \caption{\textit{OOD detection results on {CIFAR} benchmarks. All values are percentages, averaged across six OOD test datasets. Full results for each dataset are available in Appendix~\ref{appendix: detailed_cifar_benchmark}.  $\boldsymbol{\downarrow}$ / $\boldsymbol{\uparrow}$ indicates lower / higher values are better.}}
        \label{table: avg_cifar_benchmark}
               \vspace{-4mm}
    \end{table}

    \begin{table*}[ht]
        \centering
        \resizebox{0.85\textwidth}{!}{
        \begin{tabular}{l l cc cc  cc  cc}
        \toprule
        \multirow{2}{*}{\textbf{Method}} & \multicolumn{2}{c}{\textbf{ResNet-34}} & \multicolumn{2}{c}{\textbf{ResNet-50}} & \multicolumn{2}{c}{\textbf{MobileNet-v2}} & \multicolumn{2}{c}{\textbf{DenseNet-121}}\\
        \cmidrule(lr){2-3} \cmidrule(lr){4-5} \cmidrule(lr){6-7} \cmidrule(lr){8-9}
        & \textbf{FPR95} $\downarrow$ & \textbf{AUROC} $\uparrow$ & \textbf{FPR95} $\downarrow$ & \textbf{AUROC} $\uparrow$ & \textbf{FPR95} $\downarrow$ & \textbf{AUROC} $\uparrow$ & \textbf{FPR95} $\downarrow$ & \textbf{AUROC} $\uparrow$ \\
        \midrule
        MSP           & 68.84 & 81.19 & 64.76 & 82.82 & 70.49 & 80.67 & 63.46 & 82.65 \\
        ODIN          & 55.90 & 87.16 & 56.48 & 85.41 & 54.20 & 85.81 & 49.45 & 87.48 \\
        Energy        & 57.20 & 86.84 & 57.48 & 87.05 & 58.87 & 86.59 & 50.68 & 87.60 \\
        ReAct         & 32.24 & 93.08 & 30.77 & 93.27 & 48.91 & 88.75 & 35.99 & 92.27 \\
        DICE          & 39.12 & 89.96 & 35.65 & 90.94 & 41.07 & 89.94 & 38.67 & 89.65 \\
        ReAct+DICE    & 26.25 & 93.99 & 25.41 & 94.10 & 31.06 & 92.84 & 29.33 & 93.42 \\
        ASH           & 29.32 & 93.46 & 22.83 & 95.12 & 38.68 & 90.95 & 30.25 & 93.09 \\
        SCALE         & 27.02 & 94.14 & 21.89 & 95.32 & 34.28 & 92.52 & 28.06 & 93.45 \\
        \cmidrule(lr){1-9}
        \rowcolor[gray]{0.9} $\approach(\mu)$                   & 31.92 & 92.41 & 28.42 & 93.23 & 36.71 & 91.69 & 29.54 & 92.71 \\
        \rowcolor[gray]{0.9} $\approach(\sigma)$                & 31.91 & 92.36 & 29.75 & 92.92 & 33.63 & 92.27 & 29.12 & 92.80 \\
        \rowcolor[gray]{0.9} $\approach(m)$                     & 31.83 & 92.34 & 29.89 & 92.82 & 33.15 & 92.33 & 29.45 & 92.68 \\
        \rowcolor[gray]{0.9} $\approach(\mu)+\texttt{ReAct}$    & \textbf{19.84} & \textbf{95.56} & \textbf{17.02} & \textbf{96.18} & 30.81 & 93.31 & 25.43 & 94.56 \\
        \rowcolor[gray]{0.9} $\approach(\sigma)+\texttt{ReAct}$ & 19.91 & 95.50 & 17.46 & 96.02 & 31.56 & 92.71 & \textbf{24.26} & \textbf{94.61} \\
        \rowcolor[gray]{0.9} $\approach(m)+\texttt{ReAct}$      & 20.16 & 95.44 & 17.64 & 95.93 & \textbf{29.33} & \textbf{93.43} & 24.52 & 94.53 \\
        \bottomrule
        \end{tabular}}
         \caption{\textit{OOD detection results on ImageNet benchmarks. All values are percentages and are averaged over four common OOD benchmark datasets. Complete results for each individual dataset are available in Appendix~\ref{appendix: detailed_imagenet_benchmark}. $\boldsymbol{\downarrow}$ / $\boldsymbol{\uparrow}$ indicates lower / higher values are better.}}
        \label{table: avg_imagenet_benchmark}
     \vspace{-3mm}
   \end{table*}
     
    \noindent
    \textbf{Experimental Setup.} We evaluate on the CIFAR datasets~\citep{cifar_learning}. Following standard protocols~\citep{energy,ReAct,ASH}, we use six common OOD datasets for evaluation: Textures~\citep{texture}, SVHN~\citep{svhn}, Places365~\citep{places365}, LSUN-Crop~\citep{lsun}, LSUN-Resize~\citep{lsun}, and iSUN~\citep{isun}. To ensure fair comparison with prior work, we use a DenseNet-101 backbone~\citep{DenseNet}. To demonstrate architectural generality, we extend our evaluation to ResNet-18~\citep{ResNet}. Training details are detailed in Appendix~\ref{appendix: reproducibility}.
    
    \noindent
    \textbf{Results.} Table~\ref{table: avg_cifar_benchmark} summarizes our results on the CIFAR benchmarks. The table clearly shows the two key benefits%
    (1) $\approach$ (e.g., $\approach(m)$) significantly outperforms the standard energy score baseline, proving the inherent value of scaling factor. (2) When composed with \texttt{ReAct}, $\approach$ establishes a new benchmark. For instance, on CIFAR-10, $\approach(m) + \texttt{ReAct}$ reduces FPR95 by 32.87\%, 28.10\% with ResNet-18 and DenseNet-101 respectively. On CIFAR-100, $\approach(m) + \texttt{ReAct}$ reduces FPR95 by 27.94\% and 18.99\% with ResNet-18 and DenseNet-101 respectively. The detailed per-dataset results are provided in Appendix~\ref{appendix: detailed_cifar_benchmark}.

    \noindent
    \textbf{Near-OOD Evaluation.} We evaluate $\approach$ on the challenging near-OOD task of distinguishing CIFAR-10 from CIFAR-100~\citep{ssn}. While it yields marginal gains over \texttt{SCALE}, the improvements are less pronounced than in far-OOD settings, likely due to the high similarity of learned penultimate representations. Designing a more effective $\gamma$ for near-OOD settings remains an important direction for future work. Detailed results are provided in Appendix~\ref{appendix: near_ood_results}.

\subsection{ImageNet Evaluation}

    \looseness-1 
    \textbf{Experimental Setup.} To assess scalability in a more realistic setting, we evaluate on the ImageNet-1k benchmark. We use four  OOD datasets: iNaturalist~\citep{iNaturalist}, SUN~\citep{sun}, Places365~\citep{places365}, and Textures~\citep{texture}. These datasets are carefully curated to avoid class overlap with ImageNet, while spanning distinct semantic domains to rigorously assess generalization performance~\citep{energy,ReAct}. 
    
    Our evaluation showcases broad architectural robustness by using pre-trained ResNet-34, ResNet-50, DenseNet-121, and MobileNet-v2. Since primary baselines (e.g., ReAct, SCALE) did not originally report results on all of these architectures (such as ResNet-34 and DenseNet-121), we undertook a rigorous re-evaluation of all methods.  

    \noindent %
    \textbf{Results}. Table~\ref{table: avg_imagenet_benchmark} shows that $\approach$ yields~\mbox{consistent} improvements at ImageNet scale. Compared to~\mbox{energy} score,
    $\approach(m)$ reduces  FPR95 by 44.35\%, 47.99\%, 43.69\%, and 21.23\% using ResNet-34, \mbox{ResNet-50},~Mobile\-Net-v2, and DenseNet-121 architectures respectively. The most significant gains are achieved when composing $\approach$ with existing primary methods like \texttt{ReAct}. Specifically, $\approach(m) + \texttt{ReAct}$ improves FPR95 by 25.39\%, 19.41\%, 5.57\% and 12.62\% compared to previous best results using ResNet-34, ResNet-50, MobileNet-v2, and DenseNet-121 respectively. These results validate that the principles of $\approach$ are effective in complex, large-scale datasets and across diverse architectural families. The performance boost confirms that the scaling factor $\gamma$ provides significant discriminative information that is complementary to %
    existing competitive techniques. The detailed per-dataset results are in Appendix~\ref{appendix: detailed_imagenet_benchmark}.

    \noindent
    \textbf{Discussion.} While we acknowledge standardized benchmarks like OpenOOD~\cite{openood}, we adopted a more challenging and principled evaluation for two key reasons: \textit{(a)}~OpenOOD's  dataset selection excludes several difficult, widely-used testbeds like SUN~\cite{sun}, Places~\cite{places365}, and four complex categories (bubbly, honeycombed, cobwebbed, and spiralled) from Texture~\cite{texture, ViM}. \textit{(b)} OpenOOD's setup uses held-out OOD validation set for hyperparameter tuning. Our evaluation is conducted without assuming the availability of an OOD validation set and incorporates these difficult datasets to provide a more rigorous and realistic assessment of $\approach$.

    We also explored combining statistical cues (e.g., mean + std) to compute $\gamma$. Our empirical analysis showed these multivariate combinations did not yield significant performance gains over the best-performing single statistic. This finding reinforces our framework's simplicity and efficiency, as a single, well-chosen statistic is sufficient to provide a robust performance boost.

\subsection{Synergy with Existing Baselines}
    We evaluated the performance of existing baselines when applied in tandem with $\approach$ to demonstrate its complementary effect. The results show that $\approach$ provides consistent relative performance boost across baselines on CIFAR and ImageNet. For example, on  ImageNet $\approach(\mu) + \texttt{DICE}$ improves relative FPR95 by 22.24\% (ResNet-50) and 15.41\% (MobileNet-v2) respectively. A detailed breakdown is provided in Appendix~\ref{appendix: combined method}.

\subsection{Generalizability to Distance-Based Methods}
\label{subsec: knn}

    To validate $\approach$ as a general-purpose framework, we test its synergy with a distance-based K-Nearest Neighbors (KNN)~\cite{knn, ssn} OOD detector. %
    The results in Tables~\ref{table: avg_knn_cifar}~and~\ref{table: avg_knn_imagenet} confirm our hypothesis, showing that $\approach$ provides significant improvement over the KNN baseline across all benchmarks. For instance, %
    on CIFAR-100 (ResNet-18), $\approach(m)$ achieves a 43.84\% reduction. Similarly on, large-scale ImageNet benchmark, where $\approach(\mu)$ on a ResNet-50 results in a 52.13\% reduction in average FPR95. These performance boost highlight that $\approach$ is a general-purpose modulator, providing complementary information for both logit and distance based methods, making it a true \textit{plug-and-play} framework. The full experimental setup and detailed per-dataset results are in Appendix~\ref{appendix: knn}.

    Extending our framework to gradient-based methods~\cite{gradorth, GradNorm} remains future work due to engineering challenges. Furthermore, we omit Mahalanobis~\citep{maha_distance} as a baseline, following recent precedents~\cite{ReAct,DICE,ASH}, owing to its high computational cost and limiting performance. %

    \begin{table}[ht]
        \centering
        \resizebox{\columnwidth}{!}{%
        \begin{tabular}{l l l  cc  cc }
            \toprule
            \multirow{2}{*}{\textbf{Model}} & \multirow{2}{*}{\textbf{Method}} & \multicolumn{2}{c}{\textbf{CIFAR-10}} & \multicolumn{2}{c}{\textbf{CIFAR-100}}  \\
            \cmidrule(lr){3-4} \cmidrule(lr){5-6}
            && \textbf{FPR95} $\downarrow$ & \textbf{AUROC} $\uparrow$ & \textbf{FPR95} $\downarrow$ & \textbf{AUROC} $\uparrow$ \\
            \midrule
            \multirow{4}{*}{ResNet-18}  & KNN                   & 31.02 & 95.00 & 66.81 & 83.40 \\
                                        & + $\approach(\mu)$    & 25.54 & 96.18 & 52.77 & 87.98 \\
                                        & + $\approach(\sigma)$ & 16.87 & 97.28 & 38.28 & 90.80 \\
                                        & + $\approach(m)$      & \textbf{15.62} & \textbf{97.45} & \textbf{37.52} & \textbf{90.99} \\
            \cmidrule(lr){2-6}                                                    
            \multirow{4}{*}{DenseNet-101} & KNN                   & 13.08 & 97.51 & 41.97 & 88.29 \\
                                          & + $\approach(\mu)$    &  9.49 & 98.05 & 36.42 & 91.51 \\
                                          & + $\approach(\sigma)$ &  8.50 & 98.18 & 32.75 & 92.30 \\
                                          & + $\approach(m)$      &  \textbf{8.30} & \textbf{98.23} & \textbf{32.06} & \textbf{92.48} \\
            \bottomrule
            \end{tabular}
        }
        \caption{\textit{Generalizability of $\approach$ to KNN-based OOD detection on the CIFAR benchmarks. All values are averaged across six OOD test datasets. $\boldsymbol{\downarrow}$ / $\boldsymbol{\uparrow}$ indicates lower / higher values are better. Full per-dataset results are in Appendix~\ref{appendix: knn}.}}
     \vspace{-3mm}
       \label{table: avg_knn_cifar}
    \end{table}

    \begin{table*}[ht]
        \centering
        \resizebox{0.85\textwidth}{!}{
        \begin{tabular}{l l cc cc  cc  cc}
        \toprule
        \multirow{2}{*}{\textbf{Method}} & \multicolumn{2}{c}{\textbf{ResNet-34}} & \multicolumn{2}{c}{\textbf{ResNet-50}} & \multicolumn{2}{c}{\textbf{MobileNet-v2}} & \multicolumn{2}{c}{\textbf{DenseNet-121}}\\
        \cmidrule(lr){2-3} \cmidrule(lr){4-5} \cmidrule(lr){6-7} \cmidrule(lr){8-9}
        & \textbf{FPR95} $\downarrow$ & \textbf{AUROC} $\uparrow$ & \textbf{FPR95} $\downarrow$ & \textbf{AUROC} $\uparrow$ & \textbf{FPR95} $\downarrow$ & \textbf{AUROC} $\uparrow$ & \textbf{FPR95} $\downarrow$ & \textbf{AUROC} $\uparrow$ \\
        \midrule
        KNN                   & 73.26 & 93.47 & 64.05 & 95.56 & 75.54 & 91.75 & 74.01 & 92.11 \\
        + $\approach(\mu)$    & \textbf{34.69} & \textbf{97.99} & \textbf{31.11} & \textbf{98.46} & \textbf{46.77} & \textbf{97.10} & \textbf{43.55} & \textbf{97.52} \\
        + $\approach(\sigma)$ & 43.40 & 97.18 & 39.85 & 97.79 & 50.85 & 96.61 & 48.35 & 96.85 \\
        + $\approach(m)$      & 43.16 & 97.17 & 39.60 & 97.78 & 50.52 & 96.61 & 48.89 & 96.75 \\
        \bottomrule
        \end{tabular}}
         \caption{\textit{Generalizability of $\approach$ to KNN-based OOD detection on the ImageNet benchmarks. All values are averaged across six OOD test datasets. $\boldsymbol{\downarrow}$ / $\boldsymbol{\uparrow}$ indicates lower / higher values are better. Full per-dataset results are in Appendix~\ref{appendix: knn}}.}
        \vspace{-3mm}
        \label{table: avg_knn_imagenet}
    \end{table*}

\subsection{Hyperparameter Selection}
\label{subsec: hyper-parameter selection}

    \looseness-1 The clipping threshold $\boldsymbol{c}$ (Eq.~\ref{eq: gamma_c}) is crucial for enhanced performance, as it must be set to optimally distinguish ID from OOD data. Analogous to ReAct~\cite{ReAct}, we set $\boldsymbol{c}$ to the $p$-th percentile of the ID activation distribution. The choice of this percentile $p$ is the key hyperparameter to be tuned. To demonstrate its sensitivity, we summarize the OOD detection performance of $\approach(m)$ in Figure~\ref{fig: c}, varying $p$ from 10 to 100 at 5-point intervals. To this end, we follow established protocols~\citep{DICE,ReAct} and create a proxy OOD validation set,  generated by adding pixel-wise Gaussian noise to images from the ID validation set. We then select the percentile $p$ that yields the best OOD separation on this proxy task. This two-step procedure, which uses a percentile for the mechanism and a proxy set for tuning, is a robust tuning strategy grounded in prior work. The specific details and the selected $p$ values are provided in Appendix~\ref{appendix: reproducibility}.

    \begin{figure}[ht]
        \centering
        \includegraphics[width=0.95\columnwidth]{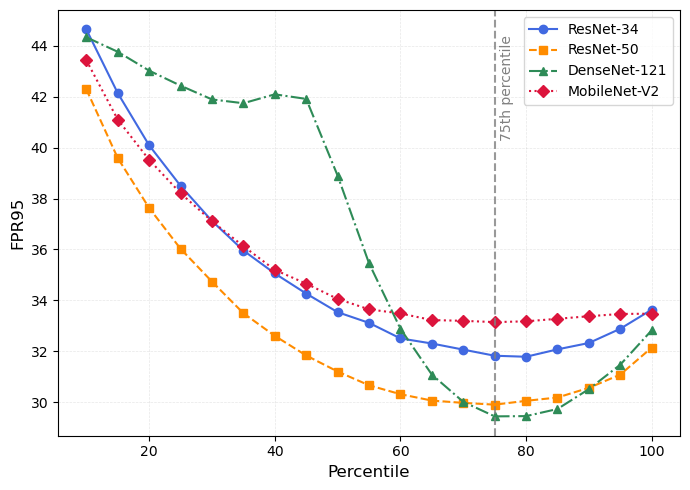}
        \vspace{-3mm}
        \caption{ \textit{Sensitivity analysis of the clipping percentile ($p$) on $\approach(m)$ performance. All values averaged over 4 OOD test datasets for a ResNet-50 (ImageNet).}}
        \label{fig: c}
        \vspace{-2mm}
    \end{figure}

\subsection{Comparison with Other Baselines}

    Comparing  $\approach$ against three contemporary methods, AdaScale~\cite{adascle}, NCI~\citep{NCI} and fDBD~\citep{fdbd}, confirms its  superiority, particularly when used as a complementary module. Aginst, AdaScale's reported reslults using DenseNet-101 on CIFAR-100, $\approach(m) + \texttt{ReAct}$ outperforms AdaScale yielding a 32.45\% gain over the best AdaScale variant. Against NCI's reported results (which were obtained on the OpenOOD settings), $\approach(m) + \texttt{ReAct}$ achieves the average FPR95 by a significant 33.43\% on CIFAR-10 (ResNet-18). This advantage is even more pronounced against fDBD, where our method achieves a substantial FPR95 reduction of 65.54\% on ImageNet (ResNet-50). We also provide a detailed comparison with 19 existing OOD detection methods in literature in Appendix~\ref{appendix: additional baselines}.

\subsection{Accuracy and Computational Overhead}

Our post-hoc method, $\approach$, maintains the original ID classification accuracy of the base model, as it does not alter its inference path. Furthermore, its computational overhead is negligible. The cost depends on the statistic used. $\approach(\mu)$ is the most efficient, as the mean is already computed by the standard GAP. The additional cost is less than 0.0001\% of a ResNet-50's forward pass. Whereas, $\approach(\sigma)$, our most complex statistic, still adds less than 0.01\% overhead. This confirms $\approach$ is lightweight and efficient framework. A detailed breakdown of accuracy and FLOPs is provided in Appendix~\ref{appendix: accuracy}.

\section{Ablation Study}
\label{sec: ablation study}

\subsection{Choice of Layer for Computing $\gamma$}
\label{subsec: penultimate layer}

    A core methodological decision is which network layer provides the most discriminative signal for $\gamma$. We conducted an analysis to locate this optimal signal source and discovered a critical and consistent trend. Using a pre-trained ResNet-50 on ImageNet-1k as a representative example, we found that the $\gamma$ distributions from the early-to-mid residual stages (Layers 1-3) are not sufficiently discriminative, exhibiting high overlap between ID and OOD data and rendering them ineffective (As shown in Figure~\ref{fig: intermediate_layers_places} of Appendix~\ref{appendix: penultimate layer}).

    This finding is intuitively aligned with the principles of hierarchical feature learning \cite{moda, low_level_features}. These initial layers learn general, low-level features such as edges, textures, and color blobs, which are fundamental properties shared by all natural images. Since both ID and OOD samples contain these common features, their activation statistics in these early layers are highly similar, resulting in the non-discriminative, overlapping $\gamma$ distributions we observed. In sharp contrast, the distribution from the final residual stage (Layer 4), immediately preceding GAP, provides a better separation, because it is trained to recognize the complex, high-level concepts and structures specific to the ID classes, which OOD samples lack. This analysis, which held true across all tested OOD datasets and architectures, empirically validates our focus: the penultimate layer's pre-pooling feature map is not a layer of convenience but the most reliable source of a potent signal for $\approach$. The complete details are presented in Appendix~\ref{appendix: penultimate layer}.

\subsection{Analysis of Fusion Strategy}
\label{subsec: fusion strategy}

    In Section~\ref{sec: method}, we alluded to two fusion strategies: multiplicative(*) and additive(+). We  investigated both to validate our design choice. Our analysis on the ImageNet (Table~\ref{table: fusion_strategy_imagenet} in Appendix~\ref{appendix: fusion strategy}) reveals that both strategies can achieve a similar high level of performance, confirming the discriminative power of the scaling factor $\gamma$ itself. 

    However, we found a critical difference in their hyperparameter robustness. The optimal additive method required tuning its clipping threshold $c^+$ at an extremely low percentile (e.g., $\le$ 1st percentile for ResNet-50), making it operationally fragile and highly sensitive to data shifts. In sharp contrast, our proposed multiplicative method tunes its threshold $c^*$ at a stable, moderate percentile, aligning with robust foundational methods like ReAct and SCALE. 
    
    Given its superior robustness and practical stability, we selected multiplicative fusion as our primary strategy. A detailed analysis of this comparison is provided in Appendix~\ref{appendix: fusion strategy}.

\subsection{Alternate Statistics: Median and Entropy}
\label{subsec: alternate statistics}

    To validate our choice of statistics (mean, std, max), we performed a rigorous analysis of two alternatives: median and Shannon entropy. This study found that median is not a viable statistic. It consistently degrades performance across all benchmarks, as its statistical signature fails to produce a discriminative $\gamma$ (see Fig.~\ref{fig: densenet-101_cifar100_svhn_md_scale_density_plot} in Appendix~\ref{appendix: alternate statistics}). The study of Shannon entropy revealed it to be inconsistent. While it provided a strong 14.65\% improvement in a specific case (MobileNet-V2 on ImageNet), this performance was not generalizable, with minimal gains on other architectures like ResNet-50. 

    This confirms our design choice: median was skipped for being ineffective, and entropy was rejected for being unreliable. Our proposed combination of mean, std, and max provides the most robust and consistently high-performing signal. Our full analysis is presented in Appendix~\ref{appendix: alternate statistics}.

\subsection{Scaling Factor $\gamma$ as a Scoring Metric}
\label{subsec: scaling factor as score}

We conducted an analysis to determine if $\gamma$ is powerful enough to serve as a standalone OOD score, similar to MSP~\cite{msp} or Energy~\cite{energy}. Our findings show that $\gamma$ computed from std ($\gamma_{\text{std}}$) and max ($\gamma_{\text{max}}$) are consistently robust signals. On both the CIFAR benchmarks (Table~\ref{table: gamma_score_cifar}) and the large-scale ImageNet benchmark (Table~\ref{table: gamma_score_imagenet}), these two statistics are consistently better than Energy baseline, proving they are viable and generalizable standalone scores. In contrast, the entropy provides a critical insight. While $\gamma_{\text{entropy}}$ appears to be the distinguishable signal on CIFAR, this trend is inconsistent on ImageNet. On this more complex benchmark, $\gamma_{\text{entropy}}$ fails to generalize, suffering a performance collapse and lagging behind Energy. This analysis confirms that entropy, while potent in some cases %
is not a reliable or generalizable statistic for a robust, all-purpose method. The complete details are provided in Appendix~\ref{appendix: scaling factor as score}.

\section{Scope and Future Work}
\label{sec: limitation and future work}

We evaluated $\approach$ using three specific statistics: mean, standard deviation, and max. As our ablations (Appendix~\ref{appendix: alternate statistics}) demonstrated, this choice was deliberate, as other statistics like median were ineffective and entropy was not generalizable. While other aggregate functions could be explored, our focus remained on this robust set.

Additionally, we provide a comprehensive analysis focused on CNN-based architectures, with $\approach$ applied within this setting. This focus is motivated by two factors: \emph{(a)}~Competitive baselines in the literature~\cite{energy, energy_01, energy_02, energy_03, ReAct, DICE, ASH, SCALE} extensively use CNN-based architectures. For fair comparison, we adopt similar architectures to evaluate $\approach$. \emph{(b)} CNN-based architectures continue to be widely used in both the research community and real-world applications. A comprehensive benchmark study carried out in prior work~\cite{Battle_of_backbones} has shown convolutional networks such as ResNet~\cite{ResNet} and ConvNeXt~\cite{ConvNext,ConvNext_v2} remain the default choice in real-world vision systems (including object detection, segmentation, retrieval, and classification) due to their strong inductive bias (translation invariance), computational efficiency, strong performance on moderate-scale data, and extensive ecosystem of pretrained models. 

The core principle of our method, leveraging statistical cues from penultimate pre-pooled activation map is a general strategy that can be extended beyond CNNs to architectures like Vision Transformers (ViTs)~\cite{ViT}. However, adapting $\approach$ to derive an effective scaling factor \(\gamma\) from the intermediate blocks of a transformer requires substantial research and engineering. We are actively exploring the extension of our framework to transformer-based models. %

\section{Related Work}
\label{sec:related work}

\noindent 
\textbf{Scoring-based OOD Detection.} Post-hoc OOD detection is dominated by the design of scoring functions. Early work on MSP~\cite{msp} and its variants~\cite{odin,G-odin,MOS} was shown to be vulnerable to model overconfidence~\cite{msp_failed_01, msp_failed_02}. This led to the development of energy-based scores~\cite{energy}, which have become the foundation for most logit-based OOD detection methods~\cite{ReAct,DICE,ASH,SCALE} due to their superior performance. Other families of scores exist, including distance-based (e.g., Mahalanobis~\cite{maha_distance}, KNN~\cite{knn,ssd,ssn,NNGuide}, fDBD~\cite{fdbd}, NCI~\cite{NCI}), gradient-based (e.g., GradNorm~\cite{GradNorm}, GradOrth~\cite{gradorth}), Virtual-logit~\cite{ViM} and Bayesian approaches~\cite{bayesian_00,bayesian_01,bayesian_02,bayesian_03,bayesian_04}, etc. $\approach$ is designed to complement and enhance these scoring paradigms.

\noindent
\textbf{Post-hoc Pruning based OOD Detection.} Recent approaches like ReAct~\cite{ReAct}, DICE~\cite{DICE}, ASH~\cite{ASH}, and SCALE~\cite{SCALE} operate post-hoc by pruning~\cite{prune_00, prune_01} or modifying feature representations, often using simple heuristics over penultimate activations. Our method, $\approach$, reveals that activation channels of layers prior to penultimate layer possess rich statistical information cues that, when exploited, can substantially improve OOD detection performance when combined with these approaches. It complements existing sparse representation techniques and is easy to integrate into standard pipelines.

\noindent
\textbf{Generative OOD Detection.} Generative models identify OOD samples by estimating data density~\cite{density_2013,density_2014,density_2014_ICLR,density_2016,density_2017,density_2017_07}, but recent work~\cite{density_ood_failed} has shown they may assign high likelihoods to OOD inputs. Moreover, these models are often harder to train and less reliable than discriminative approaches~\cite{msp, odin, GradNorm, energy, ReAct, DICE, ASH, SCALE}. Thus, we primarily focused on such discriminative approaches, while showcasing generality using KNN~\cite{knn}. However, if a generative method relies on a scalar-based scoring function, then $\approach$ can also be extended to such generative methods.

\noindent
\textbf{Training-Time OOD Detection Methods.} A distinct line of work involves modifying the model's training objective with regularization techniques to improve OOD separation~\cite{msp_failed_02, energy, bayesian_02, training_time_00, training_time_01, training_time_04, training_time_05, training_time_06, training_time_07, training_time_08, training_time_09, training_time_10}. These methods often encourage uniform predictions for outliers \cite{training_time_00, training_time_01} or explicitly penalize low energy scores for out-of-distribution samples during training \cite{energy, training_time_02, training_time_03, training_time_04, training_time_05}. In contrast, our work is entirely post-hoc, requiring no changes to the training process. This is highly practical and broadly applicable, particularly~in~scenarios involving large models where retraining is costly or infeasible.

\section{Conclusion}
\label{sec: conclusion}

$\approach$ is a simple yet powerful post-hoc framework for OOD detection that challenges the conventional paradigm of using only the pooled feature vector from the penultimate layer. We demonstrated that rich, discriminative information cues were being discarded, namely, the channel-wise statistics embedded in penultimate layer's pre-pooled feature map. $\approach$ effectively harnesses this under-explored information by computing an input-dependent scaling factor ($\gamma$) that modulates existing baseline scores, significantly enhancing the separation between ID and OOD distributions.
Extensive experiments across diverse models and datasets demonstrate that $\approach$ consistently outperforms recent competitive baselines OOD detection methods. These empirical findings are further supported by ablation studies and statistical analysis. %

{
    \small
    \bibliographystyle{ieeenat_fullname}
    \bibliography{main}
}

\appendix
\clearpage
\setcounter{page}{1}
\maketitlesupplementary

\section{Description of Baseline Methods}
\label{appendix: baseline methods}

In resonance with existing work~\cite{energy,ReAct,DICE,ASH}, for the reader’s convenience, we summarize in detail a few common techniques for defining OOD scores that measure the degree of ID-ness on the given sample. All the methods derive the score post hoc on neural networks trained with in-distribution data only. By convention, a higher score is indicative of being in-distribution, and vice versa.

\textbf{Softmax score } One of the earliest works on OOD detection considered using the maximum softmax
probability (MSP) to distinguish between $\mathcal{D}_{\text {in }}$ and $\mathcal{D}_{\text {out }}$~\cite{msp}. In detail, suppose the label space is \( \mathcal{Y} = \{1, 2, \cdots , C\}\). We assume the classifier $f$ is defined in terms of a feature extractor $f : \mathcal{X} \rightarrow \mathbb{R}^m$ and a linear multinomial regressor with weight matrix $W \in \mathbb{R}^{C \times m}$ and bias vector $\mathbf{b} \in \mathbb{R}^C$. The prediction probability for each class is given by :

\begin{small}
    \begin{equation}
        \mathbb{P}(y = c| \mathbf{x}) = \text{Softmax}(Wh(\mathbf{x}) + \mathbf{b})_{c}
    \end{equation}
\end{small}

The softmax score is defined as $S_{\text{MSP}}(\mathbf{x}; f) := \max_c\mathbb{P}(y = c| \mathbf{x})$.

\textbf{ODIN }~\cite{odin} This method introduced temperature scaling and input perturbation to improve the separation of MSP for ID and OOD data.  $\mathbf{\Tilde{x}}$ denotes perturbed input.

\begin{small}
    \begin{equation}
        \mathbb{P}(y = c| \mathbf{\Tilde{x}}) = \text{Softmax}[(Wh(\mathbf{\Tilde{x}}) + \mathbf{b})/T]_{c}
    \end{equation}
\end{small}

the ODIN score is defined as $S_{\text{ODIN}}(\mathbf{x}; f) := \max_c\mathbb{P}(y = c | \mathbf{\Tilde{x}})$.

\textbf{Energy score }  The energy function~\cite{energy} maps the output logit to a scalar $S_{\text{Energy}}(\mathbf{x}; f) \in \mathbb{R}$, which is relatively lower for ID data:

\begin{small}
    \begin{equation}
        S_{\text{Energy}}(\mathbf{x}; f) = -\text{Energy}(\mathbf{x}; f) = \log{ \left( \sum_{c=1}^{C} \exp( f_{c}(\mathbf{x}) ) \right) }
    \end{equation}
\end{small}

They used the \textit{negative energy score} for OOD detection, in order to align with the convention that $S(\mathbf{x}; f)$ is higher for ID data and vice versa.

\textbf{ReAct }  They perform post hoc modification of penultimate layer of the neural network. It works by truncating the feature activations at a threshold $c$, i.e., replacing each activation with $\min(x,c)$. This limits the influence of abnormally large activations often caused by OOD inputs.The truncation threshold is set with the validation strategy in~\cite{ReAct}.Formally,
    \[  
        \begin{split}
            h^{\text{ReAct}}(\mathbf{x}) &= \text{ReAct}(h(\mathbf{x}); c) \\
                                         &= \min(h(\mathbf{x}), c) \quad \text{(applied element-wise)}
        \end{split}
    \]
The final model output becomes:
    \[
        f^{\text{ReAct}}(\mathbf{x}) = W^\top h^{\text{ReAct}}(\mathbf{x}) + \mathbf{b}
    \]
This method also uses energy score $S_{\text{Energy}}(\mathbf{x}; f^{\text{ReAct}}) \in \mathbb{R}$ for OOD detection.

\textbf{DICE}~\cite{DICE}   It is a post hoc method to improve OOD detection by retaining only the most informative weights in the final layer of a pre-trained neural network. A \textit{contribution matrix} \( V \in \mathbb{R}^{m \times C} \) is computed, where each column is:
    \[
        \mathbf{v}_c = \mathbb{E}_{\mathbf{x} \in \mathcal{D}}[\mathbf{w}_c \odot h(\mathbf{x})]
    \]
with \( \odot \) denoting element-wise multiplication. Each entry in \( V \) quantifies the average contribution of a feature unit to class \( c \). A binary \textit{masking matrix} \( M \in \mathbb{R}^{m \times C} \) selects the top-$k$ highest-contributing weights, setting others to zero. The sparsified output is:
    \[
        f^{\text{DICE}}(\mathbf{x}; \theta) = (M \odot W)^\top h(\mathbf{x}) + \mathbf{b}
    \]
This method also uses energy score $S_{\text{Energy}}(\mathbf{x}; f^{\text{DICE}}) \in \mathbb{R}$ for OOD detection.

\textbf{ASH }~\citep{ASH}  It is also a post-hoc method that simplifies feature representations to improve OOD detection. They proposes three versions of ASH, we presented only the best performing version i.e, ASH-S.  Given an input activation vector \( h(\mathbf{x}) \) and a pruning percentile \( p \), ASH~\citep{ASH} proceeds as follows shaping the activation of penultimate layer $h(\mathbf{x})$ to get $h^{\text{ASH}}(\mathbf{x})$:

\begin{enumerate}
    \item Compute the \( p \)-th percentile threshold \( t \) of \( h(\mathbf{x}) \).
    \item Let \( s_1 = \sum h(\mathbf{x}) \), the sum of all activation values before pruning.
    \item Set all values in \( h(\mathbf{x}) \) less than \( t \) to zero.
    \item Let \( s_2 = \sum h(\mathbf{x}) \), the sum after pruning.
    \item Scale all non-zero values in \( h(\mathbf{x}) \) by \( \exp(s_1/s_2) \).
\end{enumerate}

The final model output becomes, which is then used to compute energy score $S_{\text{Energy}}(\mathbf{x}; f^{\text{ASH}}) \in \mathbb{R}$ for OOD detection :
    \[
        f^{\text{ASH}}(\mathbf{x}) = W^\top h^{\text{ASH}}(\mathbf{x}) + \mathbf{b}
    \]

\textbf{SCALE}~\citep{SCALE} It is a post-hoc method designed to enhance out-of-distribution (OOD) detection by adaptively scaling the activation of the penultimate layer $h(\mathbf{x})$ before computing the final classifier output. Given an input activation vector $h(\mathbf{x})$ and a pruning percentile $p$, SCALE~\citep{SCALE} proceeds as follows to obtain the scaled activation $h^{\text{SCALE}}(\mathbf{x})$:

\begin{enumerate}
    \item Compute the $p$-th percentile threshold $t$ of $h(\mathbf{x})$.
    \item Let $s_1 = \sum h(\mathbf{x})$, the sum of all activation values before pruning.
    \item Construct a binary mask $\mathbf{1}_{\{h(\mathbf{x}) \geq t\}}$ that keeps only the top-$p$ activations.
    \item Let $s_2 = \sum h(\mathbf{x}) \cdot \mathbf{1}_{\{h(\mathbf{x}) \geq t\}}$, the sum of the top-$p$ activations.
    \item Compute the scaling ratio $r = \tfrac{s_1}{s_2}$.
    \item Scale the original activations by $\exp(r)$:
    \[
        h^{\text{SCALE}}(\mathbf{x};\theta) = \exp(r) \cdot h(\mathbf{x};\theta).
    \]
\end{enumerate}

The final model output is then computed with the scaled activations, and the \emph{energy score} is used for OOD detection:
\[
    f^{\text{SCALE}}(\mathbf{x}) = W^\top h^{\text{SCALE}}(\mathbf{x}) + \mathbf{b}, 
    \quad
    S_{\text{Energy}}(\mathbf{x}; f^{\text{SCALE}}) \in \mathbb{R}.
\]

\textbf{KNN}~\citep{knn}. This post-hoc, feature-space method identifies OOD samples based on their distance from ID training manifold. Let $\mathcal{H}_{\text{train}} = \{h(\mathbf{x}_i) \in \mathbb{R}^d\}_{i=1}^N$ be the set of $N$ penultimate-layer feature vectors stored from the ID training set. For a new test input $\mathbf{x}$ with feature $h(\mathbf{x})$, the kNN score is computed in three steps:

\begin{enumerate}
    \item Compute Distances: The set of Euclidean distances $\{d_i\}$ between $h(\mathbf{x})$ and all stored ID features in $\mathcal{H}_{\text{train}}$ is computed:$$d_i = \| h(\mathbf{x}) - h(\mathbf{x}_i) \|_2, \quad \forall h(\mathbf{x}_i) \in \mathcal{H}_{\text{train}}$$.

    \item Identify Neighbors: The $k$ smallest distances are identified and sorted, $d_{(1)} \leq d_{(2)} \leq \dots \leq d_{(k)}$.

    \item Calculate Score: The final kNN score is the average distance to these $k$ nearest neighbors:$$S_{\text{kNN}}(\mathbf{x}) = \frac{1}{k} \sum_{j=1}^k d_{(j)}$$
\end{enumerate}

A large score $S_{\text{kNN}}(\mathbf{x})$ indicates that the sample lies far from the ID training manifold and is therefore flagged as out-of-distribution.

\section{Statistical Analysis}
\label{appendix: analysis}

    In this section, we present a detailed statistical analysis of our method, $\approach$, exhibiting how it enhances the separation between in-distribution (ID) and out-of-distribution (OOD) samples. This increased separation leads to a sharper decision boundary between ID and OOD regions. Our analysis builds on key observations commonly made in prior work on OOD detection~\citep{energy, ReAct, DICE, ASH, SCALE}. %

    \begin{figure*}[ht]
        \centering
        \includegraphics[width=0.97\textwidth]{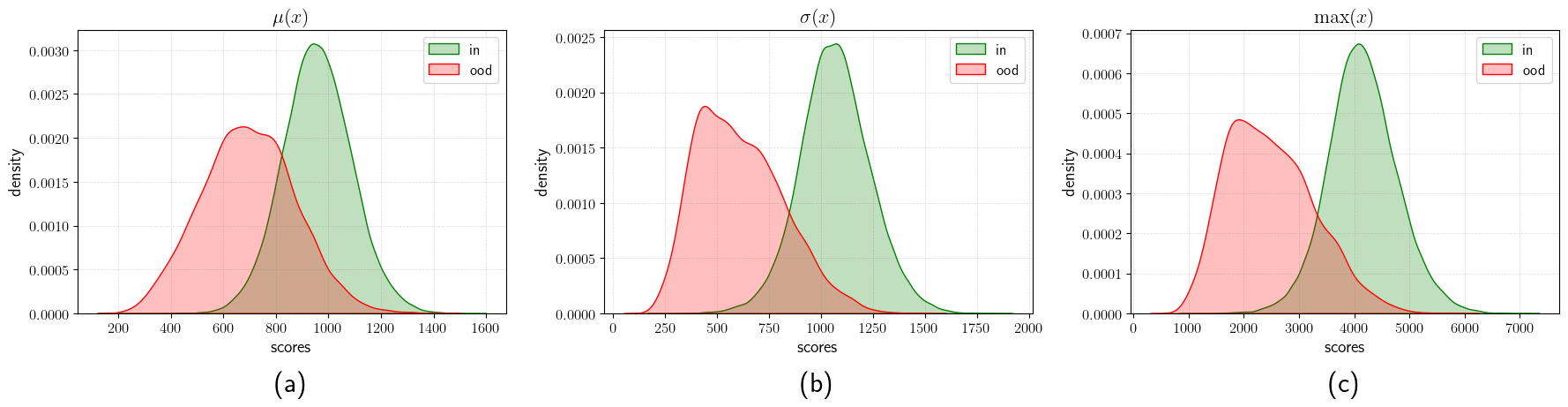}
        \caption{\textit{Distribution of scaling factor \(\gamma\) from the penultimate layer of a ResNet-50 trained on ImageNet-1k, evaluated with Texture as the OOD dataset. The scales show clear separation between ID and OOD samples. Left to right: (a) $\mu(\mathbf{x})$: mean, (b) $\sigma(\mathbf{x})$: standard deviation, (c) $\max(\mathbf{x})$: max }}
        \label{fig: scale_density_plot}
    \end{figure*}

    \subsection{Framework and Objective}
    In this section, we provide a statistical analysis demonstrating that our method, $\approach$, improves OOD detection by increasing the distributional separation between the expected scores of in-distribution (ID) and out-of-distribution (OOD) data.

    Let $S(\mathbf{x})$ be the baseline OOD score and $\gamma(\mathbf{x})$ be our input-dependent scaling factor. We analyze two fusion strategies multiplicative scaling (i.e, elastic scaling) and additive shift as described in Equation~\ref{eq: adaptive_scaling} of Section~\ref{sec: method}: 
        \[
            \begin{split}
                S^{*}(\mathbf{x}) &= \gamma(\mathbf{x})S(\mathbf{x}) \\
                S^{+}(\mathbf{x}) &= \gamma(\mathbf{x}) + S(\mathbf{x})
            \end{split}
        \]
    Our objective is to formally show that the separability of the re-calibrated scores ($\Delta_{\mathrm{scaled}}$ and $\Delta_{\mathrm{shift}}$) is greater than or equal to the separability of the original score ($\Delta_{\mathrm{original}}$). The separation $\Delta$ is defined as the difference between the expected score for in-distribution and out-of-distribution data:
        \[ 
        \begin{split}
            \Delta_{\mathrm{shift}} &= \mathbb{E}_{\mathbf{x} \sim \mathcal{D}_{\mathrm{in}}}[S^{+}(\mathbf{x})] - \mathbb{E}_{\mathbf{x} \sim \mathcal{D}_{\mathrm{out}}}[S^{+}(\mathbf{x})] \\
            \Delta_{\mathrm{scaled}} &= \mathbb{E}_{\mathbf{x} \sim \mathcal{D}_{\mathrm{in}}}[S^{*}(\mathbf{x})] - \mathbb{E}_{\mathbf{x} \sim \mathcal{D}_{\mathrm{out}}}[S^{*}(\mathbf{x})] \\
            \Delta_{\mathrm{original}} &= \mathbb{E}_{\mathbf{x} \sim \mathcal{D}_{\mathrm{in}}}[S(\mathbf{x})] - \mathbb{E}_{\mathbf{x} \sim \mathcal{D}_{\mathrm{out}}}[S(\mathbf{x})] \\
        \end{split}
        \]

    \subsection{Rationale and Assumptions}

    The rationale for OOD scoring functions~\cite{msp, energy} is to map inputs to a scalar $S(\mathbf{x})$ that separates ID from OOD data. For clarity, we will follow the convention where ID samples yield higher scores and OOD samples yield lower scores%
    The success of any post-hoc method relies on this baseline separation as a necessary condition.  

    Building on this, $\approach$ introduces a complementary scaling factor, $\gamma(\mathbf{x})$. For the fusion to be effective, $\gamma(\mathbf{x})$ must also be larger for a typical ID samples than OOD samples. This property is a necessary condition for success of $\approach$.%

    \begin{assumption}
        \label{assumption: scaling_pattern}
        The expected value of the scaling factor for in-distribution data is greater than or equal to its expected value for out-of-distribution data:
        $$\bar{\gamma}_{\text{in}} = \mathbb{E}_{\mathbf{x} \sim \mathcal{D}_{\mathrm{in}}}[\gamma(\mathbf{x})] \ge \mathbb{E}_{\mathbf{x} \sim \mathcal{D}_{\mathrm{out}}}[\gamma(\mathbf{x})] = \bar{\gamma}_{\text{out}}$$
    \end{assumption}

    In other word, by fusing these two "higher-is-ID" signals multiplicatively ($S^{*}(\mathbf{x}) = S(\mathbf{x}) \times \gamma(\mathbf{x})$), $\approach$ uses differential amplification to actively widen the ID-OOD gap:
    \begin{itemize}
        \item For a typical ID sample, the high baseline score $S(\mathbf{x})$ is amplified by the high $\gamma(\mathbf{x})$, pushing it further into the ID region.
        \item  For a typical OOD sample, the low baseline score $S(\mathbf{x})$ is suppressed by the low $\gamma(\mathbf{x})$, pushing it further into the OOD region.
    \end{itemize}
    The additive fusion, $S^{+}(\mathbf{x}) = S(\mathbf{x}) + \gamma(\mathbf{x})$, achieves a similar separation by applying a differential shift.
    
    Additionally, to simplify the theoretical analysis, we introduce the following sufficient condition, which is empirically supported by our observations (Figure~\ref{fig: density_plot}, \ref{fig: scale_density_plot}). We note that this condition is primarily for theoretical tractability; our method is empirically robust and does not strictly require this assumption to hold to achieve strong performance.
    
    \begin{assumption}
        \label{assumption: scaling pattern}
        The mean scaling factor for ID data is larger than for OOD data, and both are bounded by one as illustrated in Figure~\ref{fig: scale_density_plot}. Formally, defining $\bar{\gamma}_{\text{in}} = \mathbb{E}_{\mathbf{x} \sim \mathcal{D}_{\text{in}}}[\gamma(\mathbf{x})]$ and $\bar{\gamma}_{\text{out}} = \mathbb{E}_{\mathbf{x} \sim \mathcal{D}_{\text{out}}}[\gamma(\mathbf{x})]$, we have  
            \begin{equation}
                \bar{\gamma}_{\text{in}} \ge \bar{\gamma}_{\text{out}} \ge 1
                \label{eq:gamma_relations}
            \end{equation}
    \end{assumption}
    
    \begin{assumption}
        \label{assumption: approx uncorrelation}
        The scaling factor $\gamma(\mathbf{x})$ and baseline score $S(\mathbf{x})$ are approximately uncorrelated. This is a simplifying assumption for the analysis that the covariance is negligible for both ID and OOD data.
            \begin{equation}
                \mathrm{Cov}\big(\gamma(\mathbf{x}), S(\mathbf{x})\big) = 0
                \label{eq:co-variance}
            \end{equation}
    \end{assumption}

    \subsection{ $\approach$'s Improved Separation}
    In this section, we provide a formal characterization of how $\approach$ widens the separability between the expected ID and OOD scores under both multiplicative ($^*$) and additive ($^+$) fusion.

    \begin{theorem}
        Under Assumptions \ref{assumption: scaling pattern} and \ref{assumption: approx uncorrelation}, the distributional separation of the multiplicatively scaled score, $S^{*}(\mathbf{x})$, is at least as great as that of the original score, $S(\mathbf{x})$, i.e., $\Delta_{\text{scaled}} \ge \Delta_{\text{original}}$.
    \end{theorem}
    
    \begin{proof} By definition of \( \Delta_{\mathrm{scaled}} \):
            \[
            \resizebox{\columnwidth}{!}{$
                \begin{aligned}
                \Delta_{\text{scaled}} &= \mathbb{E}_{\mathbf{x} \sim \mathcal{D}_{\mathrm{in}}}[\gamma(\mathbf{x})S(\mathbf{x})] - \mathbb{E}_{\mathbf{x} \sim \mathcal{D}_{\mathrm{out}}}[\gamma(\mathbf{x})S(\mathbf{x})] \\
                &= \mathbb{E}_{\mathbf{x} \sim \mathcal{D}_{\mathrm{in}}}[\gamma(\mathbf{x})]\mathbb{E}_{\mathbf{x} \sim \mathcal{D}_{\mathrm{in}}}[S(\mathbf{x})]  - \mathbb{E}_{\mathbf{x} \sim \mathcal{D}_{\mathrm{out}}}[\gamma(\mathbf{x})]\mathbb{E}_{\mathbf{x} \sim \mathcal{D}_{\mathrm{out}}}[S(\mathbf{x})] \\
                &=   \bar{\gamma}_{\text{in}}\mathbb{E}_{\mathbf{x} \sim \mathcal{D}_{\mathrm{in}}}[S(\mathbf{x})] - \bar{\gamma}_{\text{out}}\mathbb{E}_{\mathbf{x} \sim \mathcal{D}_{\mathrm{out}}}[S(\mathbf{x})] \\
                &\ge  \bar{\gamma}_{\text{out}}\mathbb{E}_{\mathbf{x} \sim \mathcal{D}_{\mathrm{in}}}[S(\mathbf{x})] - \bar{\gamma}_{\text{out}}\mathbb{E}_{\mathbf{x} \sim \mathcal{D}_{\mathrm{out}}}[S(\mathbf{x})]  \\
                &= \bar{\gamma}_{\text{out}}\Big(\mathbb{E}_{\mathbf{x} \sim \mathcal{D}_{\mathrm{in}}}[S(\mathbf{x})] - \mathbb{E}_{\mathbf{x} \sim \mathcal{D}_{\mathrm{out}}}[S(\mathbf{x})] \Big) \\
                &= \bar{\gamma}_{\text{out}}\Delta_{\text{original}} \\
                \end{aligned}
                $}
            \]
            \[  
                \begin{aligned}
                     \therefore \quad \Delta_{\text{scaled}} & \ge \bar{\gamma}_{\text{out}}\Delta_{\text{original}} \\
                \end{aligned}
            \]
        \(\because \bar{\gamma}_{\text{out}} \ge 1\), we conclude scaling increases the separation between typical ID and OOD samples.     
    \end{proof}

    \begin{theorem}
        Under Assumption \ref{assumption: scaling_pattern}, the additive fusion score $S^{+}(\mathbf{x})$ increases or maintains the distributional separation compared to the baseline score, i.e., $\Delta_{\text{shift}} \ge \Delta_{\text{original}}$.
    \end{theorem}
    
    \begin{proof} By the definition of $\Delta_{\mathrm{shift}}$ and the linearity of expectation:
            \[
            \resizebox{\columnwidth}{!}{$
                \begin{aligned}
               \Delta_{\mathrm{shift}} &= \mathbb{E}_{\mathbf{x} \sim \mathcal{D}_{\mathrm{in}}}[S^{+}(\mathbf{x})] - \mathbb{E}_{\mathbf{x} \sim \mathcal{D}_{\mathrm{out}}}[S^{+}(\mathbf{x})] \\
                &= \mathbb{E}_{\mathbf{x} \sim \mathcal{D}_{\mathrm{in}}}[\gamma(\mathbf{x}) + S(\mathbf{x})]  - \mathbb{E}_{\mathbf{x} \sim \mathcal{D}_{\mathrm{out}}}[\gamma(\mathbf{x})+S(\mathbf{x})] \\
                 &= \mathbb{E}_{\mathbf{x} \sim \mathcal{D}_{\mathrm{in}}}[S(\mathbf{x})]  - \mathbb{E}_{\mathbf{x} \sim \mathcal{D}_{\mathrm{out}}}[S(\mathbf{x})] + \mathbb{E}_{\mathbf{x} \sim \mathcal{D}_{\mathrm{in}}}[\gamma(\mathbf{x})]  - \mathbb{E}_{\mathbf{x} \sim \mathcal{D}_{\mathrm{out}}}[\gamma(\mathbf{x})]\\
                 &= \Delta_{\text{original}} + \Big( \bar{\gamma}_{\text{in}} - \bar{\gamma}_{\text{out}}\Big) \\
                 &\ge \Delta_{\text{original}} \hspace{20pt} \Big( \because  \bar{\gamma}_{\text{in}} - \bar{\gamma}_{\text{out}} \ge 0 \Big)  \\
                \end{aligned}
                $}
            \]
            \[  
                \begin{aligned}
                     \therefore \quad \Delta_{\text{shift}} & \ge \Delta_{\text{original}} \\
                \end{aligned}
            \]
        We conclude shifting increases the separation between typical ID and OOD samples.     
    \end{proof}

\section{Accuracy and Computational Overhead}
\label{appendix: accuracy}

\textbf{Classification Accuracy.} Our method, $\approach$, is designed to be post-hoc. The scaling factor \(\gamma\) is computed from penultimate pre-pooled activation map without altering the network's weights or its standard forward pass. Consequently, when used as a standalone method, $\approach$ does not interfere with the model's inference process and maintains its original ID classification accuracy. We report the specific ID classification accuracy for all models used in our evaluation in Table~\ref{table: accuracy}.

    \begin{table}[th]
        \centering
        \begin{tabular}{l l l}
        \toprule
        \textbf{Dataset} & \textbf{Model} & \textbf{Accuracy} \\
        \midrule
        \multirow{2}{*}{CIFAR-10} & ResNet-18     & 93.89 \\
                                  & DenseNet-101  & 93.61 \\
        \midrule
        \multirow{2}{*}{CIFAR-100} & ResNet-18    & 75.20 \\
                                   & DenseNet-101 & 74.47 \\
        \midrule
        \multirow{4}{*}{ImageNet} & ResNet-34     & 73.31 \\
                                  & ResNet-50     & 76.13 \\
                                  & MobileNet-v2  & 71.88 \\
                                  & DenseNet-121  & 74.44 \\ 
        \bottomrule
        \end{tabular}
        \caption{\textit{In-distribution classification accuracy (\%) of the all the model used in evaluation of $\approach$.}}
        \label{table: accuracy}
    \end{table}

\noindent
\textbf{Computational Overhead.} The computational overhead introduced by $\approach$ is negligible. The primary cost is computing one of channel-wise statistics (mean, std, max) from the $n \times k \times k$ pre-pooling map, followed by the clipping (Equation~\ref{eq: gamma_c}) and summation (Equation~\ref{eq: gamma_design}).

\begin{enumerate}
    \item $\approach(\mu)$. This is the most efficient scenario. The mean statistic is simply the output of the GAP operation, which is already part of the standard forward pass. The only additional cost is the clipping (Equation~\ref{eq: gamma_c}) and summation (Equation~\ref{eq: gamma_design}) of the resulting 2048-dimensional vector. This requires only 4,096 FLOP, an overhead of less than 0.0001\% compared to the ~5.42 GFLOPs of a ResNet-50.

    \item $\approach(\sigma)$ or $\approach(m)$. These require computing a new statistic from the $n \times k \times k$ pre-pooling map. This is still negligible. For ResNet-50 (with a $2048 \times 7 \times 7$ map), computing the channel-wise maximum requires ~0.1 MFLOPs, and the standard deviation requires ~0.3 MFLOPs. In the worst-case scenario (standard deviation), the overhead is still less than 0.01\% of the full forward pass. This confirms that $\approach$ is lightweight and efficient post-hoc method.
\end{enumerate}

\section{Generalizability to Distance-Based Methods}
\label{appendix: knn}

    \begin{table*}[ht]
        \centering
        \resizebox{\textwidth}{!}{
        \begin{tabular}{l l  cc cc cc cc cc}
            \toprule
            \multirow{2}{*}{\textbf{Model}} & \multirow{2}{*}{\textbf{Method}} & \multicolumn{2}{c}{SUN} & \multicolumn{2}{c}{Places} & \multicolumn{2}{c}{Texture} & \multicolumn{2}{c}{iNaturalist} & \multicolumn{2}{c}{Average}  \\
            \cmidrule(lr){3-4} \cmidrule(lr){5-6} \cmidrule(lr){7-8} \cmidrule(lr){9-10} \cmidrule(lr){11-12}
            && \textbf{FPR95} $\downarrow$ & \textbf{AUROC} $\uparrow$ & \textbf{FPR95} $\downarrow$ & \textbf{AUROC} $\uparrow$ & \textbf{FPR95} $\downarrow$ & \textbf{AUROC} $\uparrow$ & \textbf{FPR95} $\downarrow$ & \textbf{AUROC} $\uparrow$ & \textbf{FPR95} $\downarrow$ & \textbf{AUROC} $\uparrow$ \\
            \midrule
            \multirow{4}{*}{ResNet-34}  & KNN                     & 88.42 & 91.38 & 88.21 & 90.55 & 31.08 & 98.76 & 85.31 & 93.21 & 73.26 & 93.47 \\
                                        & + $\approach(\mu)$      & \textbf{40.16} & \textbf{97.55} & \textbf{53.81} & \textbf{96.38} & 11.97 & 99.57 & \textbf{32.82} & \textbf{98.45} & \textbf{34.69} & \textbf{97.99} \\
                                        & + $\approach(\sigma)$   & 56.43 & 96.20 & 70.31 & 94.44 &  \textbf{7.77} & \textbf{99.76} & 39.11 & 98.34 & 43.40 & 97.18 \\
                                        & + $\approach(m)$        & 54.22 & 96.30 & 68.95 & 94.48 &  8.30 & 99.72 & 41.16 & 98.18 & 43.16 & 97.17 \\
            \midrule
            \multirow{4}{*}{ResNet-50}  & KNN                     & 79.08 & 94.43 & 82.21 & 93.31 & 16.29 & 99.43 & 78.61 & 95.05 & 64.05 & 95.56 \\
                                        & + $\approach(\mu)$      & \textbf{35.94} & \textbf{98.12} & \textbf{50.41} & \textbf{97.06} & 10.53 & 99.74 & \textbf{27.54} & \textbf{98.92} & \textbf{31.11} & \textbf{98.46} \\
                                        & + $\approach(\sigma)$   & 51.05 & 97.08 & 66.25 & 95.58 &  \textbf{6.88} & \textbf{99.83} & 35.23 & 98.67 & 39.85 & 97.79 \\
                                        & + $\approach(m)$        & 49.86 & 97.09 & 65.58 & 95.56 &  7.23 & 99.81 & 35.74 & 98.65 & 39.60 & 97.78 \\
            \midrule
            \multirow{4}{*}{MobileNet-v2} & KNN                   & 94.24 & 88.46 & 94.29 & 87.77 & 20.48 & 99.31 & 93.16 & 91.46 & 75.54 & 91.75 \\
                                          & + $\approach(\mu)$    & \textbf{53.00} & \textbf{96.68} & \textbf{69.59} & \textbf{94.97} & 12.48 & 99.61 & \textbf{52.00} & 97.15 & \textbf{46.77} & \textbf{97.10} \\
                                          & + $\approach(\sigma)$ & 62.90 & 95.84 & 77.39 & 93.64 &  \textbf{8.62} & \textbf{99.77} & 54.50 & \textbf{97.17} & 50.85 & 96.61 \\
                                          & + $\approach(m)$      & 61.41 & 95.88 & 76.40 & 93.66 &  8.88 & 99.76 & 55.39 & 97.14 & 50.52 & 96.61 \\
            \midrule
            \multirow{4}{*}{DenseNet-121} & KNN                   & 91.80 & 89.06 & 91.66 & 88.94 & 21.86 & 99.22 & 90.70 & 91.23 & 74.01 & 92.11 \\
                                          & + $\approach(\mu)$    & \textbf{54.87} & \textbf{96.78} & \textbf{65.96} & \textbf{95.28} & 16.45 & 99.57 & \textbf{36.91} & \textbf{98.43} & \textbf{43.55} & \textbf{97.52} \\
                                          & + $\approach(\sigma)$ & 61.96 & 95.85 & 73.53 & 93.92 & \textbf{14.11} & \textbf{99.64} & 43.79 & 97.97 & 48.35 & 96.85 \\
                                          & + $\approach(m)$      & 61.73 & 95.76 & 73.68 & 93.80 & 14.97 & 99.62 & 45.20 & 97.83 & 48.89 & 96.75 \\
            \bottomrule
            \end{tabular}
            }
            \caption{\textit{Detailed KNN-based OOD detection results for ImageNet benchmarks, using ResNet-34, ResNet-50, MobileNet-v2, and DenseNet-121. All values are percentages and are averaged over four common OOD benchmark datasets: SUN~\cite{sun}, Places~\cite{places365}, Texture~\cite{texture} and iNaturalist~\cite{iNaturalist}. The symbol $\boldsymbol{\downarrow}$ indicates lower values are better; $\boldsymbol{\uparrow}$ indicates larger values are better.}}
            \label{table: detailed_knn_imagenet}
        \end{table*}

A key question for our framework is its generalizability: is $\approach$ merely an enhancement for logit-based methods, or is it a truly general-purpose framework? To answer this, we conducted a targeted study on its synergy with an entirely different family of OOD detectors: distance-based K-Nearest Neighbors (KNN)~\cite{knn}.

\noindent
\textbf{Setup.} Our goal here is not to reproduce a specific, highly-optimized KNN baseline (which often rely on contrastive pre-training~\cite{ssd, knn, ssn} to structure the feature space). Instead, our goal is to test a hypothesis: can $\approach$ boost a generic KNN detector applied to a standard off-the-shelf pre-trained model?

To this end, we use the pre-trained models from our main experiments. Following the standard KNN OOD protocol~\cite{knn}, we build a Faiss~\cite{faiss} index of the (ID) training set's feature vectors. At inference, the baseline score $S_{\text{KNN}}(\mathbf{x})$ is the $L_2$ distance to the $k$-th nearest neighbor (we use $k=50$, a standard value from prior work~\cite{knn, ssn}). A high distance indicates an OOD sample.

We integrate $\approach$ by fusing scaling factor $\gamma$ to elastically scale this distance score. As $\gamma$ is high for ID (low distance) and low for OOD (high distance) samples, the signals are anti-correlated. We therefore use the fusion as shown in Equation~\ref{eq: knn}. This pushes ID scores even lower and OOD scores even higher, widening the separation.

    \begin{equation}
        S'_{\text{KNN}}(\mathbf{x}) = S_{\text{KNN}}(\mathbf{x}) / \gamma(\mathbf{x})
        \label{eq: knn}
    \end{equation}

\noindent
\textbf{Results.}  As shown in Table~\ref{table: detailed_knn_cifar}, $\approach$ provides an consistent improvement over the standard KNN baseline across CIFAR and ImageNet benchmark. For instance, elastically scaling using scaling factor derived from max statistics $\approach(m)$ we observed:

\begin{itemize}
    \item On CIFAR-10, $\approach$ reduces the FPR95 by 49.64\% for ResNet-18 (from 31.02\% to 15.62\%) and by 36.54\% for DenseNet-101 (from 13.08\% to 8.30\%).
    \item On CIFAR-100, $\approach$ reduces the FPR95 by 43.84\% for ResNet-18 (from 66.81\% to 37.52\%) and by 23.61\% for DenseNet-101 (from 41.97\% to 32.06\%). 
\end{itemize}

Similarly, in Table~\ref{table: detailed_knn_imagenet}, we can see an consistent improvement across the model over standard KNN baselines across all tested models on ImageNet-1k. For instance, we observe $\approach(\mu)$ reduces the FPR95 by 52.64\%, 52.13\%, 38.08\% and 41.16\% for ResNet-34, ResNet-50, MobileNet-v2, and DenseNet-121 respectively.

\noindent
\textbf{Discussion.} These performance boost demonstrate that $\approach$ is a general-purpose framework. It successfully modulates a distance-based score on a standard cross-entropy trained model, proving its utility without requiring specialized training. The discriminative signal from scaling factor $\gamma$ provides additional complementary information captured by both logit-based and distance-based methods, making it a powerful, ``plug-and-play" enhancer for diverse OOD detection paradigms.

While this principle could be extended to other families, such as gradient-based methods like GradOrth~\cite{gradorth}, we note that integrating with such methods requires substantial, non-trivial engineering to reproduce their codebases and is beyond our current scope. We therefore leave this as a promising direction for future work. Finally, we note that our evaluation omits a direct comparison to Mahalanobis~\citep{maha_distance}. This follows the precedent set by recent works~\cite{ReAct,DICE,ASH}, which has shown it to be computationally expensive while offering limiting performance on these benchmarks.

    \begin{landscape}
        \begin{table*}[ht]
        \centering
        \resizebox{\linewidth}{!}{
        \begin{tabular}{ l l l  cc cc cc cc cc cc cc }
            \toprule
             \multirow{2}{*}{\textbf{Dataset}}& \multirow{2}{*}{\textbf{Model}} & \multirow{2}{*}{\textbf{Method}} & \multicolumn{2}{c}{\textbf{SVHN}} & \multicolumn{2}{c}{\textbf{Place365}} & \multicolumn{2}{c}{\textbf{iSUN}} & \multicolumn{2}{c}{\textbf{Textures}} & \multicolumn{2}{c}{\textbf{LSUN-c}}  & \multicolumn{2}{c}{\textbf{LSUN-r}} & \multicolumn{2}{c}{\textbf{Average}}  \\
            \cmidrule(lr){4-5} \cmidrule(lr){6-7} \cmidrule(lr){8-9} \cmidrule(lr){10-11} \cmidrule(lr){12-13} \cmidrule(lr){14-15} \cmidrule(lr){16-17}
            &&& \textbf{FPR95} $\downarrow$ & \textbf{AUROC} $\uparrow$ & \textbf{FPR95} $\downarrow$ & \textbf{AUROC} $\uparrow$ & \textbf{FPR95} $\downarrow$ & \textbf{AUROC} $\uparrow$ & \textbf{FPR95} $\downarrow$ & \textbf{AUROC} $\uparrow$ & \textbf{FPR95} $\downarrow$ & \textbf{AUROC} $\uparrow$ & \textbf{FPR95} $\downarrow$ & \textbf{AUROC} $\uparrow$ & \textbf{FPR95} $\downarrow$ & \textbf{AUROC} $\uparrow$ \\
            \midrule
            \multirow{8}{*}{CIFAR-10}& \multirow{4}{*}{ResNet-18}   & KNN                      & 13.72 & 97.96 & 49.67 & 89.77 & 38.52 & 93.94 & 29.42 & 96.82 & 16.25 & 97.61 & 38.55 & 93.92 & 31.02 & 95.00 \\
                                                                    && + $\approach(\mu)$      &  9.28 & 98.51 & 52.18 & 90.16 & 31.49 & 95.81 & 25.23 & 97.67 &  4.63 & 99.20 & 30.42 & 95.76 & 25.54 & 96.18 \\
                                                                    && + $\approach(\sigma)$   &  4.83 & 99.16 & 41.62 & 91.80 & 19.54 & 97.34 & 14.45 & 98.58 &  \textbf{2.45} & 99.55 & 18.32 & 97.26 & 16.87 & 97.28 \\
                                                                    && + $\approach(m)$        &  \textbf{4.54} & \textbf{99.20} & \textbf{39.88} & \textbf{92.09} & \textbf{17.24} & \textbf{97.62} & \textbf{13.39} & \textbf{98.69} &  2.48 & \textbf{99.56} & \textbf{16.20} & \textbf{97.54} & \textbf{15.62} & \textbf{97.45} \\

            \cmidrule(lr){2-17}
            &\multirow{4}{*}{DenseNet-101}   & KNN                      & 1.51 & 99.67 & 41.83 & 90.26 & 6.89 & 98.95 & 14.36 & 98.59 & 5.59 & 98.95 & 8.28 & 98.63 & 13.08 & 97.51 \\
                                            && + $\approach(\mu)$       & 0.96 & 99.74 & 41.61 & 90.67 & 2.31 & 99.55 &  8.81 & 99.09 & \textbf{0.65} & \textbf{99.84} & 2.57 & 99.44 &  9.49 & 98.05 \\
                                            && + $\approach(\sigma)$    & \textbf{0.82} & 99.79 & 38.70 & 91.14 & \textbf{1.92} & 99.58 &  6.33 & 99.31 & 0.93 & 99.77 & 2.30 & 99.47 &  8.50 & 98.18 \\
                                            && + $\approach(m)$         & 0.84 & \textbf{99.80} & \textbf{37.51} & \textbf{91.42} & \textbf{1.92} & \textbf{99.60} &  \textbf{6.15} & \textbf{99.35} & 1.14 & 99.75 & \textbf{2.21} & \textbf{99.49} &  \textbf{8.30} & \textbf{98.23} \\

            \midrule
            \multirow{8}{*}{CIFAR-100}& \multirow{4}{*}{ResNet-18}   & KNN                      & 60.35 & 92.40 & \textbf{86.34} & \textbf{71.63} & 69.50 & 82.53 & 40.78 & 94.58 & 76.15 & 77.29 & 67.73 & 81.96 & 66.81 & 83.40 \\
                                                                    && + $\approach(\mu)$       & 34.89 & 95.51 & 90.25 & 69.69 & 65.40 & 86.66 & 40.18 & 94.72 & 21.31 & 95.50 & 64.59 & 85.78 & 52.77 & 87.98 \\
                                                                    && + $\approach(\sigma)$    &  8.33 & 98.56 & 89.69 & 70.39 & 46.61 & 91.39 & 22.22 & 97.19 & 14.48 & 96.92 & 48.38 & \textbf{90.35} & 38.28 & 90.80 \\
                                                                    && + $\approach(m)$         &  \textbf{7.67} & \textbf{98.62} & 88.64 & 71.25 & \textbf{45.87} & \textbf{91.42} & \textbf{21.92} & \textbf{97.23} & \textbf{12.95} & \textbf{97.24} & \textbf{48.05} & 90.20 & \textbf{37.52} & \textbf{90.99} \\

            \cmidrule(lr){2-17}
            &\multirow{4}{*}{DenseNet-101}   & KNN                      & 15.24 & 97.22 & \textbf{88.30} & 67.69 & 43.04 & 90.41 & 26.79 & 96.35 & 35.63 & 88.68 & 42.84 & 89.42 & 41.97 & 88.29 \\
                                            && + $\approach(\mu)$       & 11.07 & 98.20 & 89.38 & 69.47 & 44.03 & 93.20 & 23.65 & 96.70 &  \textbf{3.00} & \textbf{99.39} & 47.41 & 92.12 & 36.42 & 91.51 \\
                                            && + $\approach(\sigma)$    &  8.79 & \textbf{98.53} & 89.12 & 70.06 & 36.09 & 94.74 & 17.41 & 97.63 &  5.96 & 98.98 & 39.14 & 93.88 & 32.75 & 92.30 \\
                                            && + $\approach(m)$         &  \textbf{8.45} & \textbf{98.53} & 88.46 & \textbf{70.71} & \textbf{34.32} & \textbf{95.01} & \textbf{16.26} & \textbf{97.75} &  7.34 & 98.76 & \textbf{37.51} & \textbf{94.15} & \textbf{32.06} & \textbf{92.48} \\

            \bottomrule
            \end{tabular}
            }
            \caption{\textit{Detailed KNN-based OOD detection results for the CIFAR-10 and CIFAR-100 benchmarks, using ResNet-18 and DenseNet-101. Results are evaluated against six common OOD datasets: SVHN~\cite{svhn}, Places365~\cite{places365}, iSUN~\cite{isun}, Textures~\cite{texture}, LSUN-crop~\cite{lsun}, and LSUN-resize~\cite{lsun}. $\boldsymbol{\downarrow}$ indicates lower values are better and $\boldsymbol{\uparrow}$ indicates larger values are better.}}
            \label{table: detailed_knn_cifar}
        \end{table*}
    \end{landscape}

\section{Detailed OOD Detection Performance}
\label{appendix: detailed results}

\subsection{Near-OOD Evaluation}
\label{appendix: near_ood_results}

    We also evaluate $\approach$ on the challenging near-OOD task of distinguishing CIFAR-10 from CIFAR-100, a commonly used setup used in prior work~\citep{ssn}. As shown in Table~\ref{table: cifar_near_ood}, while the separation is inherently more difficult for all methods, $\approach(m)$ still provides a marginal performance improvement over the baselines when applied in tandem, demonstrating its robustness even in fine-grained detection scenarios. For instance, with ResNet-18, $\approach(m) + \texttt{ReAct}$ reduces the FPR95 from 52.04\% to 49.62\%, an improvement of 4.65\%. We attribute the limited improvement in near-OOD settings to the high similarity of the learned penultimate representations. A valuable direction for future research is to design a suitable scaling factor $\gamma$ in near-ood evaluation settings. %

        \begin{table}[ht]
        \centering
        \resizebox{0.75\columnwidth}{!}{
        \begin{tabular}{l c c  c}
    
        \toprule
        \multirow{1}{*}{\textbf{Method}} & \textbf{FPR95} $\downarrow$ & \textbf{AUROC} $\uparrow$ \\
        \midrule
        MSP                     & 65.85	& 88.17   \\
        + $\approach(\mu)$      & 68.10	& 71.94   \\
        + $\approach(\sigma)$   & 60.88	& 86.55   \\
        + $\approach(m)$        & 60.07	& 86.28   \\
        \midrule
        Energy                  & 52.32	& 90.14   \\
        + $\approach(\mu)$      & 52.94	& 89.82   \\
        + $\approach(\sigma)$   & 50.93	& 90.47   \\
        + $\approach(m)$        & 50.98	& 90.38   \\
        \midrule
        ReAct                   & 52.04	& 90.42   \\
        + $\approach(\mu)$      & 52.63	& 89.68   \\
        + $\approach(\sigma)$   & 50.08	& 90.47   \\
        + $\approach(m)$        & \textbf{49.62}	& \textbf{90.52}   \\
        \midrule
        DICE                    & 56.56	& 89.12   \\
        + $\approach(\mu)$      & 65.00	& 86.87   \\
        + $\approach(\sigma)$   & 57.03	& 88.69   \\
        + $\approach(m)$        & 56.89	& 88.79   \\
        \midrule
        ReAct+DICE              & 55.94	& 89.50   \\
        + $\approach(\mu)$      & 69.11	& 84.33   \\
        + $\approach(\sigma)$   & 59.34	& 87.99   \\
        + $\approach(m)$        & 58.09	& 87.99   \\
        \midrule
        ASH                     & 57.14	& 87.60   \\
        + $\approach(\mu)$      & 64.15	& 84.00   \\
        + $\approach(\sigma)$   & 56.86	& 87.38   \\
        + $\approach(m)$        & 56.33	& 87.40   \\
        \midrule
        SCALE                   & 55.58	& 88.60   \\
        + $\approach(\mu)$      & 60.96	& 85.47   \\
        + $\approach(\sigma)$   & 54.93	& 88.15   \\
        + $\approach(m)$        & 54.34	& 88.17   \\
        \bottomrule
        \end{tabular}
        }
        \caption{\textit{Near-OOD detection evaluation using ResNet-18. CIFAR-10 is ID dataset and CIFAR-100 is OOD dataset. The symbol $\boldsymbol{\downarrow}$ indicates lower values are better;  $\boldsymbol{\uparrow}$ indicates higher values are better.}}
        \label{table: cifar_near_ood}
    \end{table}

\subsection{ImageNet Evaluation.}
\label{appendix: detailed_imagenet_benchmark}

\noindent
\textbf{Evaluation.} Table~\ref{table: detailed_imagenet_benchmark} showcases detailed evaluation on ImageNet benchmark, using broad pre-trained model , ResNet-34, ResNet-50, MobileNet-v2, and DenseNet-121 for which we re-evaluated all baselines to ensure a fair comparison. Since results for ResNet-34 and DenseNet-121 were not available in the original publications of foundational baselines (e.g., ReAct, DICE, ASH, SCALE), we rigorously re-evaluated these methods ourselves. To ensure a fair and direct comparison, we carefully followed the hyperparameter selection protocols described in their respective papers (Appendix~\ref{appendix: reproducibility}).

\noindent
\textbf{Discussion.} In the Table~\ref{table: detailed_imagenet_benchmark}, $\approach$ shows limited performance improvement on the SUN and Places datasets, particularly when using the MobileNet-v2 backbone. We empirically observed that this is due to a high degree of overlap between the distribution of the scaling factor, \(\gamma\), for these datasets and for the in-distribution ImageNet-1k data. This overlap can be attributed to the high scene similarity between these datasets, a challenge previously identified by ViM~\cite{ViM}.
    
To demonstrate this, the Figure~\ref{fig:mobilenet-scale-density} presents the distributions of \(\gamma\) for the SUN, Places365, Texture, and iNaturalist datasets, generated using the pre-trained MobileNet-v2 model. A clear pattern emerges: the distributions for the scene-based datasets (SUN, Places365) exhibit a significantly greater overlap with the in-distribution data compared to the more distinct Texture and iNaturalist datasets. This effect is particularly prominent when using the standard deviation $\sigma(\mathbf{x})$ and maximum value $\max(\mathbf{x})$ as information cues. 

For brevity, we omit the \(\gamma\) distribution plots for the ResNet-34 and ResNet-50 backbones, but we confirm they exhibit the same general pattern. However, we also find that the overlap is more prominent for MobileNet-V2 than for ResNet-34, and in turn, more prominent for ResNet-34 than for ResNet-50.

        \begin{figure*}[ht]
    
        \begin{subfigure}{\textwidth}
            \centering
            \includegraphics[width=0.97\textwidth]{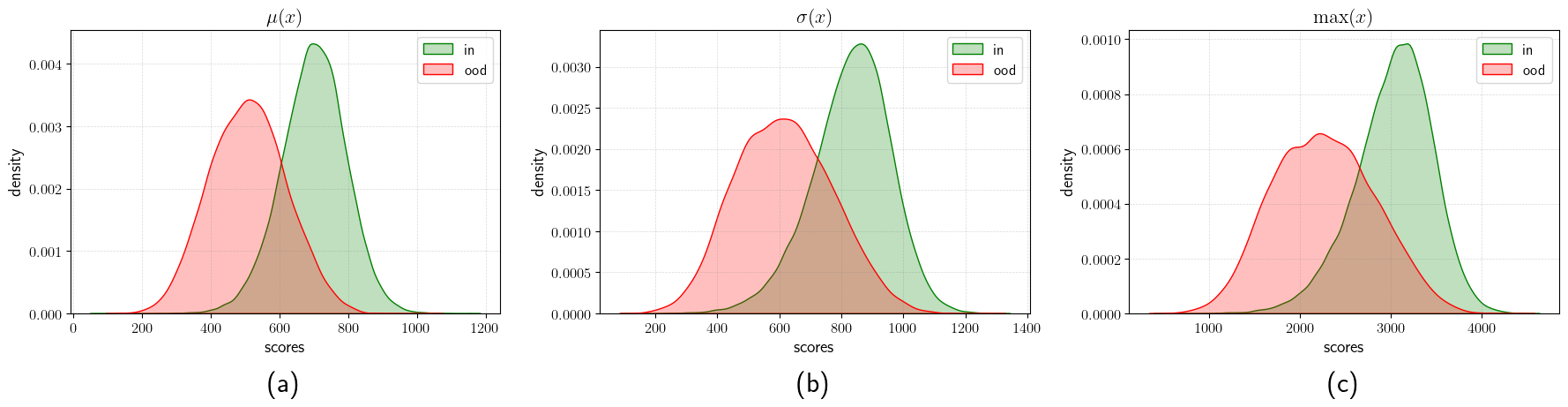}
        \end{subfigure}
    
        \begin{subfigure}{\textwidth}
            \centering
            \includegraphics[width=0.97\textwidth]{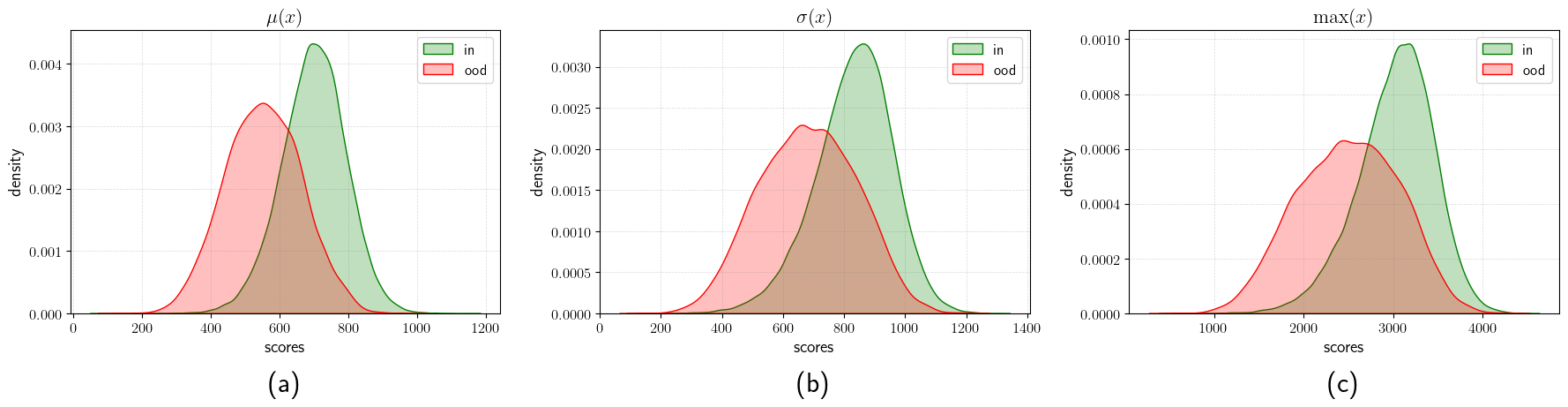}
        \end{subfigure}
    
        \begin{subfigure}{\textwidth}
            \centering
            \includegraphics[width=0.97\textwidth]{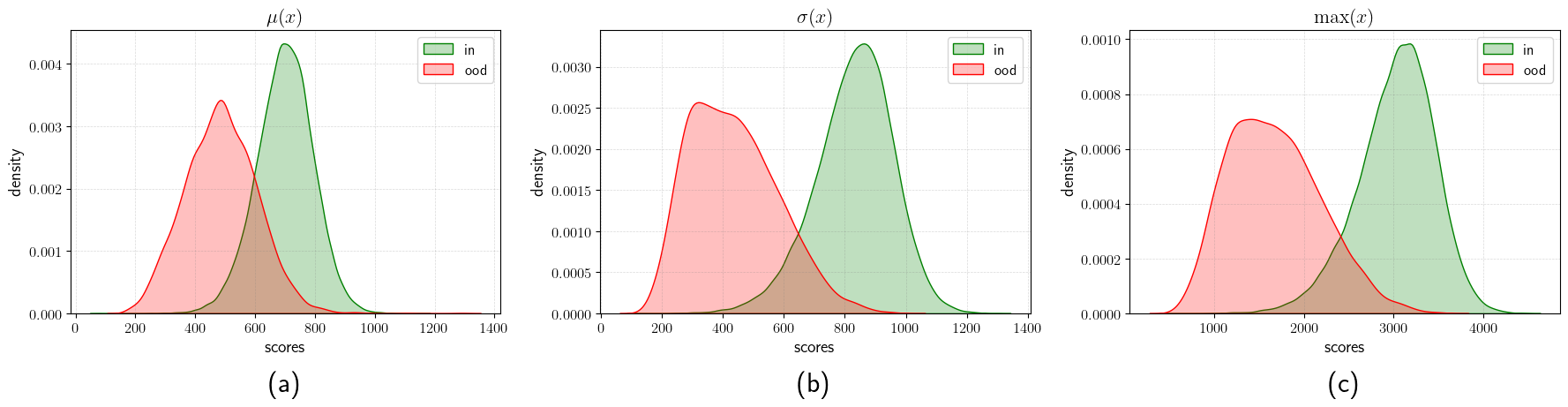}
        \end{subfigure}
    
        \begin{subfigure}{\textwidth}
            \centering
            \includegraphics[width=0.97\textwidth]{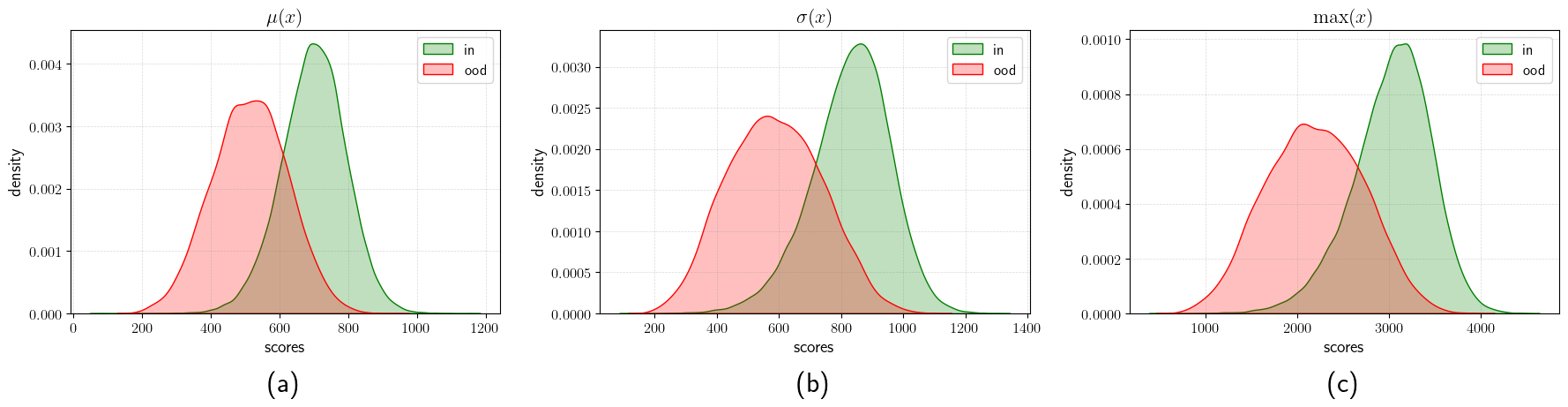}
        \end{subfigure}
    
        \caption {\textit{Distributions of the scaling factor \(\gamma\), derived from the penultimate layer of a MobileNet-V2 model trained on ImageNet-1k. The rows \emph{(top to bottom)} correspond to the OOD datasets: SUN, Places365, Texture, and iNaturalist. The columns (left to right) correspond to the statistical cue used to compute \(\gamma\): (a) mean: $\mu(\mathbf{x})$, (b) standard deviation: $\sigma(\mathbf{x})$, and (c) maximum value: $\max(\mathbf{x})$ ( we used \(\max(\mathbf{x})\) and \(m(\mathbf{x})\) interchangeably). A clear pattern emerges: the distributions for the scene-based datasets (SUN, Places365) exhibit a significantly greater overlap with the in-distribution data compared to the more distinct Texture and iNaturalist datasets. This effect is particularly prominent when using the standard deviation $\sigma(\mathbf{x})$ and maximum value $\max(\mathbf{x})$ as information cues. We observe a similar pattern for ResNet-34 and ResNet-50 backbones. However, we also find that the overlap is more prominent for MobileNet-V2 than for ResNet-34, and in turn, more prominent for ResNet-34 than for ResNet-50.}}
        \label{fig:mobilenet-scale-density}
    \end{figure*}

    \begin{table*}[ht]
        \centering
        \resizebox{\textwidth}{!}{
        \begin{tabular}{l l  cc cc cc cc cc}
            \toprule
            \multirow{2}{*}{\textbf{Model}} & \multirow{2}{*}{\textbf{Method}} & \multicolumn{2}{c}{SUN} & \multicolumn{2}{c}{Places} & \multicolumn{2}{c}{Texture} & \multicolumn{2}{c}{iNaturalist} & \multicolumn{2}{c}{Average}  \\
            \cmidrule(lr){3-4} \cmidrule(lr){5-6} \cmidrule(lr){7-8} \cmidrule(lr){9-10} \cmidrule(lr){11-12}
            && \textbf{FPR95} $\downarrow$ & \textbf{AUROC} $\uparrow$ & \textbf{FPR95} $\downarrow$ & \textbf{AUROC} $\uparrow$ & \textbf{FPR95} $\downarrow$ & \textbf{AUROC} $\uparrow$ & \textbf{FPR95} $\downarrow$ & \textbf{AUROC} $\uparrow$ & \textbf{FPR95} $\downarrow$ & \textbf{AUROC} $\uparrow$ \\
            \midrule
            \multirow{14}{*}{ResNet-34} & MSP        & 72.39 & 79.81 & 73.76 & 79.20 & 69.98 & 79.12 & 59.24 & 86.61 & 68.84 & 81.19 \\
                                        & ODIN       & 59.34 & 86.13 & 64.62 & 84.14 & 51.95 & 87.45 & 47.69 & 90.91 & 55.90 & 87.16 \\
                                        & Energy     & 57.39 & 86.59 & 62.61 & 84.59 & 54.95 & 86.45 & 53.86 & 89.73 & 57.20 & 86.84 \\
                                        & ReAct      & 25.03 & 94.28 & 34.32 & 91.67 & 46.21 & 90.63 & 23.40 & 95.75 & 32.24 & 93.08 \\
                                        & DICE       & 38.03 & 90.20 & 48.40 & 87.18 & 34.72 & 90.24 & 35.34 & 92.22 & 39.12 & 89.96 \\
                                        & ReAct+DICE & 22.33 & 94.79 & 32.85 & 91.84 & 30.39 & 93.21 & 19.42 & 96.13 & 26.25 & 93.99 \\
                                        & ASH        & 36.22 & 91.72 & 47.53 & 88.58 & 14.18 & 97.12 & 19.34 & 96.44 & 29.32 & 93.46 \\
                                        & SCALE      & 32.00 & 92.91 & 42.51 & 90.50 & 16.97 & 96.24 & 16.59 & 96.93 & 27.02 & 94.14 \\
            \cmidrule(lr){2-12}
            \rowcolor[gray]{0.9}        & \textbf{$\approach(\mu)$}                     & 33.46 & 91.85 & 43.78 & 89.39 & 24.86 & 93.39 & 25.60 & 94.99 & 31.92 & 92.41 \\
            \rowcolor[gray]{0.9}        & \textbf{$\approach(\sigma)$}                  & 37.78 & 90.74 & 48.87 & 87.82 & 16.08 & 95.80 & 24.90 & 95.10 & 31.91 & 92.36 \\
            \rowcolor[gray]{0.9}        & \textbf{$\approach(m)$}                       & 36.90 & 90.88 & 48.37 & 87.87 & 16.83 & 95.58 & 25.22 & 95.00 & 31.83 & 92.34 \\
            \rowcolor[gray]{0.9}        & \textbf{$\approach(\mu)+\texttt{ReAct}$}      & \textbf{21.44} & \textbf{95.18} & \textbf{31.74} & \textbf{92.56} & 13.39 & 97.02 & \textbf{12.81} & \textbf{97.47} & \textbf{19.84} & \textbf{95.56} \\
            \rowcolor[gray]{0.9}        & \textbf{$\approach(\sigma)+\texttt{ReAct}$}   & 21.80 & 95.06 & 32.11 & 92.41 & 12.73 & 97.11 & 13.01 & 97.41 & 19.91 & 95.50 \\
            \rowcolor[gray]{0.9}        & \textbf{$\approach(m)+\texttt{ReAct}$}        & 22.03 & 94.99 & 32.58 & 92.31 & \textbf{12.55} & \textbf{97.15} & 13.47 & 97.33 & 20.16 & 95.44 \\
            \midrule
            \multirow{14}{*}{ResNet-50} & MSP        & 68.58 & 81.75 & 71.57 & 80.63 & 66.13 & 80.46 & 52.77 & 88.42 & 64.76 & 82.82 \\
                                        & ODIN	     & 60.15 & 84.59 & 67.89 & 81.78 & 50.23 & 85.62 & 47.66 & 89.66 & 56.48 & 85.41 \\
                                        & Energy     & 58.28 & 86.73 & 65.40 & 84.13 & 52.29 & 86.73 & 53.95 & 90.59 & 57.48 & 87.05 \\
                                        & ReAct      & 23.68 & 94.44 & 33.33 & 91.96 & 46.33 & 90.30 & 19.73 & 96.37 & 30.77 & 93.27 \\
                                        & DICE       & 36.11 & 91.01 & 47.62 & 87.76 & 32.38 & 90.48 & 26.48 & 94.53 & 35.65 & 90.94 \\
                                        & ReAct+DICE & 24.05 & 94.31 & 34.28 & 91.71 & 28.40 & 93.33 & 14.90 & 97.06 & 25.41 & 94.10 \\
                                        & ASH        & 28.01 & 94.02 & 39.84 & 90.98 & 11.95 & 97.60 & 11.52 & 97.87 & 22.83 & 95.12 \\
                                        & SCALE      & 25.78 & 94.54 & 36.86 & 91.96 & 14.56 & 96.75 & 10.37 & 98.02 & 21.89 & 95.32 \\
            \cmidrule(lr){2-12}
            \rowcolor[gray]{0.9}        & \textbf{$\approach(\mu)$}                    & 30.79 & 92.67 & 42.59 & 89.78 & 22.29 & 94.01 & 18.02 & 96.46 & 28.42 & 93.23 \\
            \rowcolor[gray]{0.9}        & \textbf{$\approach(\sigma)$}                 & 35.73 & 91.47 & 48.35 & 88.04 & 15.85 & 95.94 & 19.05 & 96.21 & 29.75 & 92.92 \\
            \rowcolor[gray]{0.9}        & \textbf{$\approach(m)$}                      & 35.79 & 91.40 & 48.68 & 87.82 & 16.08 & 95.88 & 19.00 & 96.18 & 29.89 & 92.82 \\
            \rowcolor[gray]{0.9}        & \textbf{$\approach(\mu)+\texttt{ReAct}$}     & \textbf{18.46} & \textbf{95.82} & \textbf{28.98} & \textbf{93.31} & 12.11 & 97.38 &  \textbf{8.54} & \textbf{98.19} & \textbf{17.02} & \textbf{96.18} \\
            \rowcolor[gray]{0.9}        & \textbf{$\approach(\sigma)+\texttt{ReAct}$}  & 19.13 & 95.61 & 29.58 & 93.04 & \textbf{12.04} & \textbf{97.38} &  9.10 & 98.06 & 17.46 & 96.02 \\
            \rowcolor[gray]{0.9}        & \textbf{$\approach(m)+\texttt{ReAct}$}       & 19.02 & 95.52 & 29.77 & 92.92 & 12.06 & 97.31 &  9.71 & 97.97 & 17.64 & 95.93 \\
            \midrule
            \multirow{14}{*}{MobileNet-v2} & MSP     & 74.20 & 78.88 & 76.89 & 78.14 & 70.99 & 78.95 & 59.86 & 86.72 & 70.49 & 80.67 \\
                                        & ODIN       & 54.07 & 85.88 & 57.36 & 84.71 & 49.96 & 85.03 & 55.39 & 87.62 & 54.20 & 85.81 \\
                                        & Energy     & 59.36 & 86.24 & 66.27 & 83.21 & 54.54 & 86.58 & 55.31 & 90.34 & 58.87 & 86.59 \\
                                        & ReAct      & 52.46 & 87.26 & 59.89 & 84.07 & 40.25 & 90.96 & 43.05 & 92.72 & 48.91 & 88.75 \\
                                        & DICE       & 37.84 & 90.81 & 52.35 & 86.17 & 32.57 & 91.46 & 41.53 & 91.30 & 41.07 & 89.94 \\
                                        & ReAct+DICE & 30.60 & 92.98 & 45.93 & 88.29 & 16.03 & 96.33 & 31.68 & 93.76 & 31.06 & 92.84  \\
                                        & ASH        & 43.63 & 90.02 & 58.85 & 84.73 & 13.12 & 97.10 & 39.13 & 91.94 & 38.68 & 90.95 \\
                                        & SCALE      & 38.74 & 91.64 & 53.49 & 87.34 & 14.79 & 96.65 & 30.09 & 94.46 & 34.28 & 92.52 \\
            \cmidrule(lr){2-12}
            \rowcolor[gray]{0.9}        & \textbf{$\approach(\mu)$}                     & 37.74 & 91.43 & 52.21 & 87.33 & 23.42 & 94.17 & 33.47 & 93.84 & 36.71 & 91.69 \\
            \rowcolor[gray]{0.9}        & \textbf{$\approach(\sigma)$}                  & 38.20 & 91.26 & 53.04 & 86.84 & 14.02 & 96.37 & 29.25 & 94.63 & 33.63 & 92.27 \\
            \rowcolor[gray]{0.9}        & \textbf{$\approach(m)$}                       & 37.41 & 91.37 & 52.24 & 86.89 & 14.18 & 96.35 & 28.78 & 94.70 & 33.15 & 92.33 \\
            \rowcolor[gray]{0.9}        & \textbf{$\approach(\mu)+\texttt{ReAct}$}      & \textbf{32.82} & \textbf{92.93} & \textbf{48.62} & \textbf{88.59} & 13.60 & 96.83 & 28.19 & 94.89 & 30.81 & 93.31 \\
            \rowcolor[gray]{0.9}        & \textbf{$\approach(\sigma)+\texttt{ReAct}$}   & 37.53 & 91.22 & 51.32 & 87.19 & 10.18 & 97.31 & 27.21 & 95.12 & 31.56 & 92.71 \\
            \rowcolor[gray]{0.9}        & \textbf{$\approach(m)+\texttt{ReAct}$}        & 34.77 & 92.26 & 49.77 & 88.06 &  \textbf{8.69} & \textbf{97.76} & \textbf{24.08} & \textbf{95.66} & \textbf{29.33} & \textbf{93.43} \\
            \midrule
            \multirow{14}{*}{DenseNet-121} & MSP        & 67.49 & 81.41 & 69.53 & 80.95 & 67.23 & 79.18 & 49.58 & 89.05 & 63.46 & 82.65 \\
                                           & ODIN       & 54.13 & 86.33 & 60.39 & 84.14 & 50.82 & 85.81 & 32.47 & 93.66 & 49.45 & 87.48 \\
                                           & Energy     & 52.51 & 87.27 & 58.24 & 85.05 & 52.22 & 85.42 & 39.75 & 92.66 & 50.68 & 87.60 \\
                                           & ReAct      & 41.06 & 91.23 & 48.48 & 88.17 & 33.46 & 93.65 & 20.98 & 96.04 & 35.99 & 92.27 \\
                                           & DICE       & 38.75 & 89.91 & 49.29 & 86.24 & 40.85 & 88.09 & 25.78 & 94.37 & 38.67 & 89.65 \\
                                           & ReAct+DICE & 31.36 & 92.99 & 43.91 & 89.11 & 24.38 & 95.14 & 17.68 & 96.44 & 29.33 & 93.42 \\
                                           & ASH        & 37.20 & 91.51 & 46.54 & 88.79 & 21.76 & 95.04 & 15.50 & 97.03 & 30.25 & 93.09 \\
                                           & SCALE      & 33.85 & 92.16 & 42.92 & 89.62 & 22.27 & 94.63 & \textbf{13.21} & \textbf{97.40} & 28.06 & 93.45 \\
            \cmidrule(lr){2-12}
            \rowcolor[gray]{0.9}        & \textbf{$\approach(\mu)$}                     & 33.24 & 91.84 & 42.94 & 89.01 & 25.59 & 93.29 & 16.41 & 96.69 & 29.54 & 92.71 \\
            \rowcolor[gray]{0.9}        & \textbf{$\approach(\sigma)$}                  & 34.12 & 91.57 & 44.34 & 88.53 & 21.06 & 94.52 & 16.95 & 96.57 & 29.12 & 92.80 \\
            \rowcolor[gray]{0.9}        & \textbf{$\approach(m)$}                       & 34.29 & 91.47 & 44.74 & 88.35 & 21.26 & 94.43 & 17.50 & 96.46 & 29.45 & 92.68 \\
            \rowcolor[gray]{0.9}        & \textbf{$\approach(\mu)+\texttt{ReAct}$}      & 31.58 & 93.41 & 42.77 & 90.30 & 12.71 & 97.44 & 14.66 & 97.11 & 25.43 & 94.56 \\
            \rowcolor[gray]{0.9}        & \textbf{$\approach(\sigma)+\texttt{ReAct}$}   & \textbf{30.04} & \textbf{93.44} & \textbf{41.50} & \textbf{90.24} & \textbf{11.37} & \textbf{97.61} & 14.12 & 97.16 & \textbf{24.26} & \textbf{94.61} \\
            \rowcolor[gray]{0.9}        & \textbf{$\approach(m)+\texttt{ReAct}$}        & 30.25 & 93.37 & 41.61 & 90.13 & 11.74 & 97.52 & 14.48 & 97.09 & 24.52 & 94.53 \\
            \bottomrule
            \end{tabular}
            }
            \caption{\textit{Detailed OOD detection results on ImageNet benchmarks. All values are percentages and are averaged over four common OOD benchmark datasets: SUN~\cite{sun}, Places~\cite{places365}, Texture~\cite{texture} and iNaturalist~\cite{iNaturalist}. The symbol $\boldsymbol{\downarrow}$ indicates lower values are better; $\boldsymbol{\uparrow}$ indicates larger values are better.}}
            \label{table: detailed_imagenet_benchmark}
        \end{table*}

\subsection{CIFAR Evaluation}
\label{appendix: detailed_cifar_benchmark}

\noindent
\textbf{Evaluation.} We present detailed performance results across six OOD test datasets for models: ResNet-18, and DenseNet-101 trained on CIFAR-10 and CIFAR-100, in Table~\ref{table:detailed_results_cifar10} and Table~\ref{table:detailed_results_cifar100}, respectively.

\noindent
\textbf{Discussion.} As shown in Table~\ref{table:detailed_results_cifar100}, $\approach$ yields limited improvement on the Places365 dataset for models trained on CIFAR-100. We empirically attribute this to a high degree of overlap between the scaling factor \(\gamma\) distributions of the in-distribution (CIFAR-100) and Places365 samples, as shown in Figures~\ref{fig:resnet-18_cifar100_scale_density_plot} and \ref{fig:densenet-101_cifar100_scale_density_plot}. This pattern is consistent with our analysis on the ImageNet benchmark, where similar overlaps led to reduced performance. This case illustrates a key requirement for our method: its success hinges on a significant distributional separation of $\gamma$, between ID/OOD data.

    \begin{figure*}[ht]
        \centering
        \includegraphics[width=0.97\textwidth]{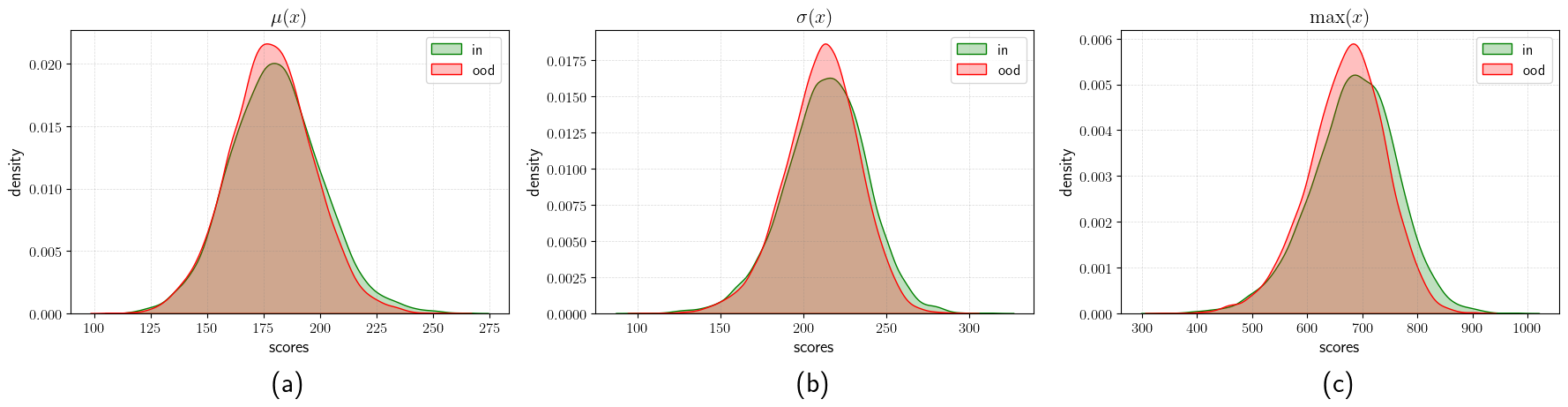}
        \caption{\textit{Distribution of scaling factor \(\gamma\) from the penultimate layer of a ResNet-18 trained on CIFAR-100, evaluated with Places365 as the OOD dataset. The scales shows high overlap between ID and OOD samples. \emph{Left to right}: (a) $\mu(\mathbf{x})$: mean, (b) $\sigma(\mathbf{x})$: standard deviation, (c) $\max(\mathbf{x})$: max }}
        \label{fig:resnet-18_cifar100_scale_density_plot}
    \end{figure*}

    \begin{figure*}[ht]
        \centering
        \includegraphics[width=0.97\textwidth]{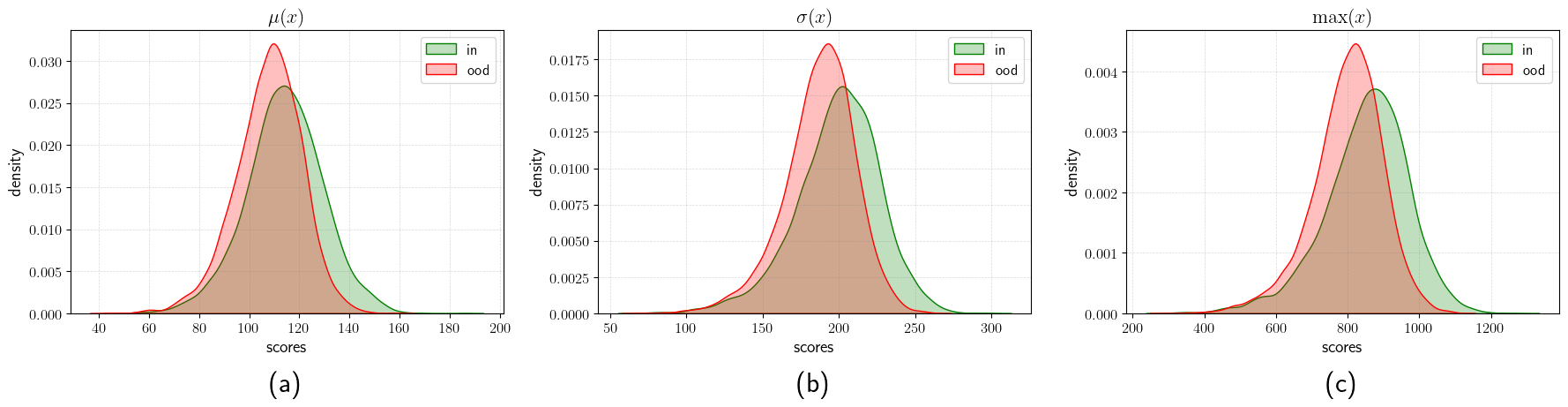}
        \caption{\textit{Distribution of scaling factor \(\gamma\) from the penultimate layer of a DenseNet-101 trained on CIFAR-100, evaluated with Places365 as the OOD dataset. The scales shows high overlap between ID and OOD samples. \emph{Left to right}: (a) $\mu(\mathbf{x})$: mean, (b) $\sigma(\mathbf{x})$: standard deviation, (c) $\max(\mathbf{x})$: max }}
        \label{fig:densenet-101_cifar100_scale_density_plot}
    \end{figure*}

    \begin{landscape}
        \begin{table*}[ht]
        \centering
        \resizebox{\linewidth}{!}{
        \begin{tabular}{ l l  cc cc cc cc cc cc cc }
            \toprule
             \multirow{2}{*}{\textbf{Model}} & \multirow{2}{*}{\textbf{Method}} & \multicolumn{2}{c}{\textbf{SVHN}} & \multicolumn{2}{c}{\textbf{Place365}} & \multicolumn{2}{c}{\textbf{iSUN}} & \multicolumn{2}{c}{\textbf{Textures}} & \multicolumn{2}{c}{\textbf{LSUN-c}}  & \multicolumn{2}{c}{\textbf{LSUN-r}} & \multicolumn{2}{c}{\textbf{Average}}  \\
            \cmidrule(lr){3-4} \cmidrule(lr){5-6} \cmidrule(lr){7-8} \cmidrule(lr){9-10} \cmidrule(lr){11-12} \cmidrule(lr){13-14} \cmidrule(lr){15-16}
            && \textbf{FPR95} $\downarrow$ & \textbf{AUROC} $\uparrow$ & \textbf{FPR95} $\downarrow$ & \textbf{AUROC} $\uparrow$ & \textbf{FPR95} $\downarrow$ & \textbf{AUROC} $\uparrow$ & \textbf{FPR95} $\downarrow$ & \textbf{AUROC} $\uparrow$ & \textbf{FPR95} $\downarrow$ & \textbf{AUROC} $\uparrow$ & \textbf{FPR95} $\downarrow$ & \textbf{AUROC} $\uparrow$ & \textbf{FPR95} $\downarrow$ & \textbf{AUROC} $\uparrow$ \\
            \midrule
            \multirow{10}{*}{ResNet-18} & MSP	      & 60.39 & 92.40 & 63.69 & 88.37 & 56.74 & 91.32 & 62.66 & 90.10 & 51.87 & 93.64 & 54.63 & 91.87 & 58.33 & 91.28 \\
                                        & ODIN        & 35.96 & 94.70 & 41.11 & 92.06 & 23.36 & 96.56 & 46.74 & 91.97 &  6.66 & 98.71 & 20.04 & 96.93 & 28.98 & 95.16 \\
                                        & Energy	  & 44.32 & 94.04 & 41.43 & 91.72 & 35.22 & 94.70 & 50.30 & 91.11 &  9.77 & 98.19 & 31.97 & 95.26 & 35.50 & 94.17 \\
                                        & ReAct	      & 42.31 & 94.12 & 40.70 & 92.25 & 23.07 & 96.37 & 40.44 & 93.69 & 12.27 & 97.90 & 19.78 & 96.80 & 29.76 & 95.19 \\
                                        & DICE	      & 17.60 & 97.09 & 46.14 & 90.66 & 39.08 & 94.32 & 44.65 & 91.80 &  1.90 & 99.57 & 36.52 & 94.70 & 30.98 & 94.69 \\
                                        & ReAct+DICE  & 11.05 & 98.07 & 47.53 & 91.14 & 17.19 & 97.04 & 24.33 & 95.91 &  1.56 & 99.66 & 16.24 & 97.19 & 19.65 & 96.50 \\
                                        & ASH	      &  6.24 & 98.80 & 53.83 & 88.05 & 21.61 & 96.44 & 21.81 & 96.41 &  1.94 & 99.52 & 20.31 & 96.49 & 20.96 & 95.95 \\
                                        & SCALE       &  7.73 & 98.54 & 50.51 & 89.81 & 21.43 & 96.62 & 22.29 & 96.27 &  4.18 & 99.18 & 20.17 & 96.75 & 21.05 & 96.19 \\
            \cmidrule(lr){2-16}
            \rowcolor[gray]{0.9}        & \textbf{$\approach(\mu)$}	    & 15.73 & 97.32 & 43.65 & 91.25 & 26.26 & 96.08 & 35.98 & 94.25 & 3.20 & 99.26 & 24.26 & 96.30 & 24.85 & 95.74 \\
            \rowcolor[gray]{0.9}        & \textbf{$\approach(\sigma)$}  & 10.33 & 98.13 & 37.74 & 92.68 & 16.90 & 97.32 & 24.31 & 96.16 & 1.38 & 99.63 & 15.64 & 97.41 & 17.72 & 96.89 \\
            \rowcolor[gray]{0.9}        & \textbf{$\approach(m)$}	    &  9.93 & 98.24 & 36.59 & 92.97 & 14.89 & 97.61 & 23.01 & 96.43 & 1.31 & 99.65 & 13.80 & 97.68 & 16.59 & 97.10 \\
            \rowcolor[gray]{0.9}        & \textbf{$\approach(\mu)+\texttt{ReAct}$}	    & 14.37 & 97.48 & 43.18 & 91.46 & 16.71 & 97.22 & 25.04 & 95.82 & 4.26 & 99.13 & 15.70 & 97.36 & 19.88 & 96.41 \\
            \rowcolor[gray]{0.9}        & \textbf{$\approach(\sigma)+\texttt{ReAct}$}   &  9.38 & 98.29 & 36.51 & 93.10 & 10.82 & 98.10 & 16.86 & 97.27 & 1.57 & 99.61 & 10.38 & 98.14 & 14.25 & 97.42 \\
            \rowcolor[gray]{0.9}        & \textbf{$\approach(m)+\texttt{ReAct}$}	    &  8.86 & 98.39 & 35.04 & 93.38 &  9.08 & 98.32 & 15.64 & 97.48 & 1.52 & 99.63 &  9.00 & 98.35 & 13.19 & 97.59 \\
            \midrule
            \multirow{14}{*}{DenseNet-101} & MSP	     & 64.76 & 88.33 & 60.30 & 88.55 & 33.57 & 95.41 & 56.67 & 90.17 & 23.41 & 96.75 & 33.87 & 95.37 & 45.43 & 92.43 \\
                                           & ODIN        & 33.09 & 94.41 & 36.68 & 92.34 &  3.22 & 99.20 & 38.49 & 91.61 &  1.84 & 99.53 &  2.89 & 99.28 & 19.37 & 96.06 \\
                                           & Energy	  & 37.91 & 93.59 & 36.42 & 92.38 &  7.33 & 98.27 & 43.87 & 90.48 &  1.95 & 99.47 &  6.97 & 98.38 & 22.41 & 95.43 \\
                                           & ReAct	      & 23.18 & 96.28 & 33.96 & 92.97 &  5.56 & 98.49 & 32.23 & 93.98 &  2.47 & 99.33 &  5.37 & 98.59 & 17.13 & 96.61 \\
                                           & DICE	      & 16.66 & 96.98 & 37.59 & 92.04 &  2.31 & 99.42 & 27.98 & 92.71 &  0.15 & 99.94 &  2.44 & 99.36 & 14.52 & 96.74 \\
                                           & ReAct+DICE  &  4.60 & 99.02 & 35.94 & 92.91 &  1.78 & 99.51 & 17.07 & 96.78 &  0.12 & 99.95 &  2.02 & 99.47 & 10.26 & 97.94 \\
                                           & ASH	      &  5.18 & 98.90 & 42.80 & 90.42 &  2.97 & 99.27 & 15.80 & 97.04 &  0.45 & 99.80 &  3.06 & 99.25 & 11.71 & 97.44 \\
                                           & SCALE       & 29.23 & 95.23 & 37.86 & 92.14 &  6.71 & 98.46 & 36.99 & 92.28 &  1.71 & 99.50 &  6.80 & 98.48 & 19.88 & 96.01 \\
            \cmidrule(lr){2-16}
            \rowcolor[gray]{0.9}        & \textbf{$\approach(\mu)$}	     & 15.12 & 97.51 & 33.93 & 92.75 & 3.32 & 98.98 & 25.71 & 94.88 & 0.61 & 99.79 & 3.70 & 98.92 & 13.73 & 97.14 \\
            \rowcolor[gray]{0.9}        & \textbf{$\approach(\sigma)$}   & 10.95 & 98.11 & 32.51 & 92.99 & 1.73 & 99.47 & 18.26 & 96.34 & 0.26 & 99.90 & 1.88 & 99.43 & 10.93 & 97.71 \\
            \rowcolor[gray]{0.9}        & \textbf{$\approach(m)$}	     & 10.86 & 98.13 & 31.69 & 93.16 & 1.64 & 99.48 & 17.93 & 96.51 & 0.30 & 99.89 & 1.83 & 99.43 & 10.71 & 97.77 \\
            \rowcolor[gray]{0.9}        & \textbf{$\approach(\mu)+\texttt{ReAct}$}	     &  5.82 & 98.76 & 31.59 & 93.50 & 2.87 & 99.15 & 16.91 & 96.83 & 0.91 & 99.75 & 3.32 & 99.09 & 10.24 & 97.85 \\
            \rowcolor[gray]{0.9}        & \textbf{$\approach(\sigma)+\texttt{ReAct}$}    &  5.82 & 98.83 & 30.35 & 93.71 & 1.49 & 99.54 & 11.26 & 97.78 & 0.34 & 99.88 & 1.69 & 99.51 &  8.49 & 98.21 \\
            \rowcolor[gray]{0.9}        & \textbf{$\approach(m)+\texttt{ReAct}$}	     &  5.86 & 98.86 & 29.97 & 93.89 & 1.49 & 99.55 & 11.06 & 97.88 & 0.48 & 99.87 & 1.68 & 99.51 &  8.42 & 98.26 \\
            \bottomrule
            \end{tabular}
            }
            \caption{\textit{Detailed results on six common OOD benchmark datasets: SVHN~\cite{svhn}, Places365~\cite{places365}, iSUN~\cite{isun}, Textures~\cite{texture}, LSUN-crop~\cite{lsun}, LSUN-resize~\cite{lsun}. We used the same ResNet-18 and DenseNet-101 pre-trained on CIFAR-10. $\boldsymbol{\downarrow}$ indicates lower values are better and $\boldsymbol{\uparrow}$ indicates larger values are better.}}
            \label{table:detailed_results_cifar10}
        \end{table*}
    \end{landscape}

    \begin{landscape}
        \begin{table*}[ht]
        \centering        
        \resizebox{\textwidth}{!}{
        \begin{tabular}{ l l  cc cc cc cc cc cc cc }
            \toprule
             \multirow{2}{*}{\textbf{Model}} & \multirow{2}{*}{\textbf{Method}} & \multicolumn{2}{c}{\textbf{SVHN}} & \multicolumn{2}{c}{\textbf{Place365}} & \multicolumn{2}{c}{\textbf{iSUN}} & \multicolumn{2}{c}{\textbf{Textures}} & \multicolumn{2}{c}{\textbf{LSUN-c}}  & \multicolumn{2}{c}{\textbf{LSUN-r}} & \multicolumn{2}{c}{\textbf{Average}}  \\
            \cmidrule(lr){3-4} \cmidrule(lr){5-6} \cmidrule(lr){7-8} \cmidrule(lr){9-10} \cmidrule(lr){11-12} \cmidrule(lr){13-14} \cmidrule(lr){15-16}
            && \textbf{FPR95} $\downarrow$ & \textbf{AUROC} $\uparrow$ & \textbf{FPR95} $\downarrow$ & \textbf{AUROC} $\uparrow$ & \textbf{FPR95} $\downarrow$ & \textbf{AUROC} $\uparrow$ & \textbf{FPR95} $\downarrow$ & \textbf{AUROC} $\uparrow$ & \textbf{FPR95} $\downarrow$ & \textbf{AUROC} $\uparrow$ & \textbf{FPR95} $\downarrow$ & \textbf{AUROC} $\uparrow$ & \textbf{FPR95} $\downarrow$ & \textbf{AUROC} $\uparrow$ \\
            \midrule
            \multirow{10}{*}{ResNet-18} & MSP	      & 74.26 & 83.20 & 82.37 & 75.31 & 84.13 & 71.57 & 85.04 & 74.02 & 70.79 & 82.78 & 82.96 & 73.10 & 79.92 & 76.66 \\
                                        & ODIN		  & 70.30 & 88.06 & 80.14 & 77.02 & 60.26 & 86.98 & 81.56 & 76.56 & 47.73 & 91.84 & 56.35 & 88.23 & 66.06 & 84.78	\\
                                        & Energy	  & 66.64 & 89.53 & 81.39 & 76.83 & 71.46 & 83.02 & 85.18 & 75.68 & 48.01 & 91.63 & 68.57 & 84.53 & 70.21 & 83.54 \\
                                        & ReAct	      & 56.62 & 91.69 & 80.38 & 77.28 & 53.40 & 89.25 & 57.27 & 88.63 & 49.29 & 90.69 & 49.59 & 90.27 & 57.76 & 87.97 \\
                                        & DICE	      & 40.89 & 92.97 & 81.33 & 76.23 & 62.61 & 85.83 & 75.28 & 76.29 & 12.44 & 97.65 & 61.39 & 86.84 & 55.66 & 85.97 \\
                                        & ReAct+DICE  & 34.16 & 94.18 & 83.57 & 74.79 & 54.50 & 89.85 & 52.96 & 87.36 & 10.40 & 97.95 & 53.78 & 90.22 & 48.23 & 89.06 \\
                                        & ASH	      & 22.00 & 96.16 & 86.10 & 69.25 & 64.55 & 84.17 & 37.87 & 91.77 & 23.39 & 95.57 & 63.19 & 84.25 & 49.52 & 86.86 \\
                                        & SCALE       & 22.12 & 96.38 & 81.96 & 74.95 & 61.62 & 86.65 & 44.50 & 90.72 & 18.62 & 96.78 & 59.76 & 86.74 & 48.10 & 88.70 \\
            \cmidrule(lr){2-16}
            \rowcolor[gray]{0.9}        & \textbf{$\approach(\mu)$}	    & 31.13 & 95.02 & 81.53 & 76.00 & 64.83 & 85.24 & 62.06 & 85.32 & 16.45 & 97.17 & 61.59 & 86.00 & 52.93 & 87.46 \\
            \rowcolor[gray]{0.9}        & \textbf{$\approach(\sigma)$}  & 20.60 & 96.54 & 82.09 & 75.57 & 55.69 & 88.63 & 54.61 & 87.27 & 10.36 & 98.19 & 54.42 & 88.87 & 46.29 & 89.18  \\
            \rowcolor[gray]{0.9}        & \textbf{$\approach(m)$}	    & 19.94 & 96.66 & 81.83 & 76.16 & 55.48 & 88.74 & 54.66 & 87.38 &  9.47 & 98.37 & 54.35 & 88.89 & 45.96 & 89.37 \\
            \rowcolor[gray]{0.9}        & \textbf{$\approach(\mu)+\texttt{ReAct}$}	     & 19.43 & 96.65 & 85.03 & 73.97 & 50.51 & 89.41 & 31.76 & 93.30 & 16.52 & 96.91 & 48.33 & 89.70 & 41.93 & 89.99 \\
            \rowcolor[gray]{0.9}        & \textbf{$\approach(\sigma)+\texttt{ReAct}$}    & 12.01 & 97.78 & 84.81 & 73.90 & 38.70 & 92.53 & 28.69 & 93.87 &  8.56 & 98.30 & 38.15 & 92.51 & 35.15 & 91.48 \\
            \rowcolor[gray]{0.9}        & \textbf{$\approach(m)+\texttt{ReAct}$}	     & 11.47 & 97.85 & 83.96 & 74.67 & 38.04 & 92.72 & 28.58 & 93.96 &  7.65 & 98.46 & 38.23 & 92.55 & 34.66 & 91.70 \\
            \midrule
            \multirow{10}{*}{DenseNet-101} & MSP	     & 81.38 & 75.71 & 82.68 & 74.06 & 82.52 & 70.50 & 87.11 & 68.39 & 51.82 & 87.93 & 79.31 & 72.21 & 77.47 & 74.80 \\
                                           & ODIN	      & 85.94 & 80.35 & 75.59 & 77.62 & 48.03 & 89.12 & 83.37 & 67.83 & 12.78 & 97.70 & 40.28 & 91.35 & 57.67 & 84.00 \\
                                           & Energy	  & 70.99 & 86.66 & 77.28 & 76.94 & 59.39 & 85.68 & 83.49 & 67.47 & 11.45 & 97.89 & 50.90 & 88.57 & 58.92 & 83.87 \\
                                           & ReAct	      & 69.82 & 86.30 & 79.23 & 74.09 & 41.50 & 92.40 & 72.09 & 80.38 & 18.14 & 96.26 & 36.53 & 93.64 & 52.89 & 87.18 \\
                                           & DICE	      & 32.93 & 94.09 & 79.90 & 75.43 & 35.50 & 92.50 & 64.84 & 71.95 &  1.93 & 99.57 & 30.81 & 93.96 & 40.98 & 87.92 \\
                                           & ReAct+DICE  & 25.10 & 95.70 & 84.17 & 73.56 & 27.98 & 95.06 & 41.79 & 87.82 &  1.06 & 99.70 & 27.76 & 95.16 & 34.64 & 91.17 \\
                                           & ASH	      & 10.32 & 97.99 & 85.80 & 71.97 & 37.68 & 92.45 & 35.48 & 91.77 &  5.43 & 98.98 & 40.35 & 91.96 & 35.84 & 90.85 \\
                                           & SCALE       & 16.26 & 97.05 & 78.54 & 76.97 & 43.56 & 91.21 & 45.60 & 87.23 &  3.23 & 99.30 & 42.69 & 91.02 & 38.31 & 90.46 \\
            \cmidrule(lr){2-16}
            \rowcolor[gray]{0.9}        & \textbf{$\approach(\mu)$}	    & 22.45 & 96.11 & 78.72 & 77.16 & 48.77 & 89.75 & 52.09 & 83.58 & 1.54 & 99.68 & 44.92 & 90.44 & 41.42 & 89.45 \\
            \rowcolor[gray]{0.9}        & \textbf{$\approach(\sigma)$}  & 21.13 & 96.30 & 78.19 & 77.16 & 42.78 & 91.48 & 44.34 & 86.19 & 1.18 & 99.72 & 40.25 & 92.02 & 37.98 & 90.48 \\
            \rowcolor[gray]{0.9}        & \textbf{$\approach(m)$}	    & 19.90 & 96.45 & 77.30 & 77.67 & 41.02 & 91.81 & 42.48 & 87.12 & 1.28 & 99.70 & 38.78 & 92.25 & 36.79 & 90.83 \\
            \rowcolor[gray]{0.9}        & \textbf{$\approach(\mu)+\texttt{ReAct}$}	     & 11.73 & 97.67 & 83.17 & 74.06 & 25.66 & 95.16 & 26.12 & 93.93 & 1.69 & 99.52 & 27.77 & 94.98 & 29.36 & 92.56 \\
            \rowcolor[gray]{0.9}        & \textbf{$\approach(\sigma)+\texttt{ReAct}$}    & 14.13 & 97.32 & 83.87 & 74.36 & 25.15 & 95.51 & 23.21 & 94.55 & 1.55 & 99.55 & 26.41 & 95.37 & 29.05 & 92.78 \\
            \rowcolor[gray]{0.9}        & \textbf{$\approach(m)+\texttt{ReAct}$}	     & 13.70 & 97.36 & 83.00 & 75.14 & 23.26 & 95.83 & 21.68 & 94.95 & 2.07 & 99.44 & 24.63 & 95.64 & 28.06 & 93.06 \\
            \bottomrule
            \end{tabular}
            }
            \caption{\textit{Detailed results on six common OOD benchmark datasets: SVHN~\cite{svhn}, Places365~\cite{places365}, iSUN~\cite{isun}, Textures~\cite{texture}, LSUN-crop~\cite{lsun}, LSUN-resize~\cite{lsun}. We used the same ResNet-18 and DenseNet-101 pre-trained on CIFAR-100. $\boldsymbol{\downarrow}$ indicates lower values are better and $\boldsymbol{\uparrow}$ indicates larger values are better.}}
            \label{table:detailed_results_cifar100}
        \end{table*}
    \end{landscape}

\section{Comparison with Other Baselines}
\label{appendix: additional baselines}

     While in the main paper we restrict our comparison to foundational representative techniques (i.e., MSP, ODIN, Energy, ReAct, DICE, ASH and SCALE), we provide a comparison of our method, $\approach$, with additional baselines  fDBD~\citep{fdbd} and NCI~\citep{NCI} in this section. A comprehensive re-evaluation of fDBD and NCI across all architectures used in our study was determined to be beyond the scope of this work due to a fundamental difference in their design philosophy.

    Methods like ReAct, DICE, ASH, SCALE and our own $\approach$ are modular, post-hoc techniques that primarily modify the penultimate feature vector itself. In contrast, fDBD and NCI introduce entirely new scoring functions derived from the geometric relationship between features and the classifier's decision boundaries (fDBD) or class weight vectors (NCI). Integrating $\approach$ into these structurally different frameworks would require significant, non-trivial engineering effort. Therefore, for these two methods, we present a comparison limited to the overlapping architectures and datasets from their original publications.

    Additionally, we conduct a large-scale benchmark comparison on ImageNet-1k against plethora of existing literature using both ResNet-50 and MobileNet-v2. As shown in Table~\ref{table: other_baselines_imagenet_benchmark}, we compare our method against 19 existing baselines for ResNet-50~\cite{msp, odin, G-odin, maha_distance, knn, gradorth, GradNorm, NNGuide, ViM, fdbd, BATS, energy, ReAct, DICE, ASH, SCALE} and 15 baselines for MobileNet-v2~\cite{msp, odin, maha_distance, gradorth, GradNorm, NNGuide, ViM, BATS, energy, ReAct, DICE, ASH, SCALE}, with all competitors' results taken directly from their original publications. This comprehensive evaluation demonstrates that $\approach$ achieves competitive and consistent performance compared to all prior post-hoc methods on this challenging benchmark.

    \begin{table*}[ht]
        \centering
        \resizebox{\textwidth}{!}{
        \begin{tabular}{l l  cc cc cc cc cc}
            \toprule
            \multirow{2}{*}{\textbf{Model}} & \multirow{2}{*}{\textbf{Method}} & \multicolumn{2}{c}{SUN} & \multicolumn{2}{c}{Places} & \multicolumn{2}{c}{Texture} & \multicolumn{2}{c}{iNaturalist} & \multicolumn{2}{c}{Average}  \\
            \cmidrule(lr){3-4} \cmidrule(lr){5-6} \cmidrule(lr){7-8} \cmidrule(lr){9-10} \cmidrule(lr){11-12}
            && \textbf{FPR95} $\downarrow$ & \textbf{AUROC} $\uparrow$ & \textbf{FPR95} $\downarrow$ & \textbf{AUROC} $\uparrow$ & \textbf{FPR95} $\downarrow$ & \textbf{AUROC} $\uparrow$ & \textbf{FPR95} $\downarrow$ & \textbf{AUROC} $\uparrow$ & \textbf{FPR95} $\downarrow$ & \textbf{AUROC} $\uparrow$ \\
            \midrule
            \multirow{25}{*}{ResNet-50} & MSP$^*$~\cite{msp}                & 68.58 & 81.75 & 71.57 & 80.63 & 66.13 & 80.46 & 52.77 & 88.42 & 64.76 & 82.82 \\
                                        & ODIN$^*$~\cite{odin}              & 60.15 & 84.59 & 67.89 & 81.78 & 50.23 & 85.62 & 47.66 & 89.66 & 56.48 & 85.41 \\
                                        & GODIN~\cite{G-odin}               & 60.83 & 85.60 & 63.70 & 83.81 & 77.85 & 73.27 & 61.91 & 85.40 & 66.07 & 82.02 \\
                                    & Mahalanobis~\cite{maha_distance}      & 68.36 & 84.35 & 73.32 & 81.46 & 16.05 & 94.96 & 39.90 & 93.76 & 49.41 & 88.63 \\
                            & KNN ($\alpha = 100\%$)~\cite{knn}             & 68.82 & 80.72 & 76.28 & 75.76 & 11.77 & 97.07 & 59.00 & 86.47 & 53.97 & 85.01 \\
                            & KNN ($\alpha = 1\%$)~\cite{knn}               & 69.53 & 80.10 & 77.09 & 74.87 & 11.56 & 97.18 & 59.08 & 86.20 & 54.32 & 84.59 \\
                                        & GradOrth~\cite{gradorth}          & 19.61 & 95.76 & 33.67 & 91.78 & \textbf{11.19} & \textbf{98.06} & 11.04 & 98.00 & 18.57 & 96.31 \\
                                        & GradNorm~\cite{GradNorm}          & 42.81 & 87.26 & 55.62 & 81.85 & 38.15 & 87.73 & 23.73 & 93.97 & 40.08 & 87.70 \\
                                        & NN-Guide~\cite{NNGuide}           & 31.62 & 91.66 & 38.88 & 90.12 & 24.93 & 91.52 & 12.02 & 97.47 & 26.86 & 92.69 \\
                                        & ViM~\cite{ViM}                    & 43.10 & 89.39 & 52.86 & 86.61 & 17.18 & 93.58 & 20.34 & 96.24 & 33.37 & 91.45 \\
                                        & fDBD~\cite{fdbd}                  & 60.60 & 86.97 & 66.40 & 84.27 & 37.50 & 92.12 & 40.24 & 93.67 & 51.19 & 89.26 \\
                                        & BATS~\cite{BATS}                  & 22.62 & 95.33 & 34.34 & 91.83 & 38.90 & 92.27 & 12.57 & 97.67 & 27.11 & 94.20 \\
                                        & LAPS~\cite{laps}                  & \textbf{15.81} & \textbf{96.18} & \textbf{24.71} & \textbf{93.64} & 41.49 & 91.81 & 12.72 & 97.50 & 23.68 & 94.78 \\
                                        & Energy$^*$~\cite{energy}          & 58.28 & 86.73 & 65.40 & 84.13 & 52.29 & 86.73 & 53.95 & 90.59 & 57.48 & 87.05 \\
                                        & ReAct$^*$~\cite{ReAct}            & 23.68 & 94.44 & 33.33 & 91.96 & 46.33 & 90.30 & 19.73 & 96.37 & 30.77 & 93.27 \\
                                        & DICE$^*$~\cite{DICE}              & 36.11 & 91.01 & 47.62 & 87.76 & 32.38 & 90.48 & 26.48 & 94.53 & 35.65 & 90.94 \\
                                        & ReAct+DICE$^*$~\cite{ReAct,DICE}  & 24.05 & 94.31 & 34.28 & 91.71 & 28.40 & 93.33 & 14.90 & 97.06 & 25.41 & 94.10 \\
                                        & ASH$^*$~\cite{ASH}                & 28.01 & 94.02 & 39.84 & 90.98 & 11.95 & 97.60 & 11.52 & 97.87 & 22.83 & 95.12 \\
                                        & SCALE$^*$~\cite{SCALE}            & 25.78 & 94.54 & 36.86 & 91.96 & 14.56 & 96.75 & 10.37 & 98.02 & 21.89 & 95.32 \\
            \cmidrule(lr){2-12}
            \rowcolor[gray]{0.9}        & \textbf{$\approach(\mu)$}                    & 30.79 & 92.67 & 42.59 & 89.78 & 22.29 & 94.01 & 18.02 & 96.46 & 28.42 & 93.23 \\
            \rowcolor[gray]{0.9}        & \textbf{$\approach(\sigma)$}                 & 35.73 & 91.47 & 48.35 & 88.04 & 15.85 & 95.94 & 19.05 & 96.21 & 29.75 & 92.92 \\
            \rowcolor[gray]{0.9}        & \textbf{$\approach(m)$}                      & 35.79 & 91.40 & 48.68 & 87.82 & 16.08 & 95.88 & 19.00 & 96.18 & 29.89 & 92.82 \\
            \rowcolor[gray]{0.9}        & \textbf{$\approach(\mu)+\texttt{ReAct}$}     & 18.46 & 95.82 & 28.98 & 93.31 & 12.11 & 97.38 & \textbf{8.54} & \textbf{98.19} & \textbf{17.02} & \textbf{96.18} \\
            \rowcolor[gray]{0.9}        & \textbf{$\approach(\sigma)+\texttt{ReAct}$}  & 19.13 & 95.61 & 29.58 & 93.04 & 12.04 & 97.38 & 9.10 & 98.06 & 17.46 & 96.02 \\
            \rowcolor[gray]{0.9}        & \textbf{$\approach(m)+\texttt{ReAct}$}       & 19.02 & 95.52 & 29.77 & 92.92 & 12.06 & 97.31 & 9.71 & 97.97 & 17.64 & 95.93 \\
            \midrule
            \multirow{21}{*}{MobileNet-v2} & MSP$^*$~\cite{msp}                 & 74.20 & 78.88 & 76.89 & 78.14 & 70.99 & 78.95 & 59.86 & 86.72 & 70.49 & 80.67 \\
                                        & ODIN$^*$~\cite{odin}                  & 54.07 & 85.88 & 57.36 & 84.71 & 49.96 & 85.03 & 55.39 & 87.62 & 54.20 & 85.81 \\
                                    & Mahalanobis~\cite{maha_distance}          & 54.79 & 86.33 & 53.77 & 83.69 & 88.72 & 37.28 & 62.04 & 82.37 & 64.83 & 72.40 \\
                                        & GradOrth~\cite{gradorth}              & 30.82 & \textbf{93.18} & 40.27 & 89.12 & 12.69 & 97.52 & 26.81 & 93.17 & 27.65 & 93.25 \\
                                        & GradNorm~\cite{GradNorm}              & 42.15 & 89.65 & 56.56 & 83.93 & 34.95 & 90.99 & 33.70 & 92.46 & 41.84 & 89.20 \\
                                        & NN-Guide~\cite{NNGuide}               & 79.57 & 76.10 & 81.87 & 74.23 & 38.78 & 89.32 & 68.24 & 82.07 & 67.12 & 80.43 \\
                                        & ViM~\cite{ViM}                        & 88.67 & 66.37 & 92.16 & 62.43 & 40.71 & 89.59 & 86.86 & 69.57 & 77.10 & 71.99 \\
                                        & BATS~\cite{BATS}                      & 41.68 & 90.21 & 52.43 & 86.26 & 38.69 & 90.76 & 31.56 & 94.33 & 41.09 & 90.39 \\
                                        & LAPS~\cite{laps}                      & \textbf{30.07} & 92.98 & \textbf{39.70} & \textbf{90.10} & 51.37 & 88.29 & \textbf{18.82} & \textbf{96.76} & 34.99 & 92.03 \\
                                        & Energy$^*$~\cite{energy}              & 59.36 & 86.24 & 66.27 & 83.21 & 54.54 & 86.58 & 55.31 & 90.34 & 58.87 & 86.59 \\
                                        & ReAct$^*$~\cite{ReAct}                & 52.46 & 87.26 & 59.89 & 84.07 & 40.25 & 90.96 & 43.05 & 92.72 & 48.91 & 88.75 \\
                                        & DICE$^*$~\cite{DICE}                  & 37.84 & 90.81 & 52.35 & 86.17 & 32.57 & 91.46 & 41.53 & 91.30 & 41.07 & 89.94 \\
                                        & ReAct+DICE$^*$~\cite{ReAct,DICE}      & 30.60 & 92.98 & 45.93 & 88.29 & 16.03 & 96.33 & 31.68 & 93.76 & 31.06 & 92.84  \\
                                        & ASH$^*$~\cite{ASH}                    & 43.63 & 90.02 & 58.85 & 84.73 & 13.12 & 97.10 & 39.13 & 91.94 & 38.68 & 90.95 \\
                                        & SCALE$^*$~\cite{SCALE}                & 38.74 & 91.64 & 53.49 & 87.34 & 14.79 & 96.65 & 30.09 & 94.46 & 34.28 & 92.52 \\
            \cmidrule(lr){2-12}
            \rowcolor[gray]{0.9}        & \textbf{$\approach(\mu)$}                     & 37.74 & 91.43 & 52.21 & 87.33 & 23.42 & 94.17 & 33.47 & 93.84 & 36.71 & 91.69 \\
            \rowcolor[gray]{0.9}        & \textbf{$\approach(\sigma)$}                  & 38.20 & 91.26 & 53.04 & 86.84 & 14.02 & 96.37 & 29.25 & 94.63 & 33.63 & 92.27 \\
            \rowcolor[gray]{0.9}        & \textbf{$\approach(m)$}                       & 37.41 & 91.37 & 52.24 & 86.89 & 14.18 & 96.35 & 28.78 & 94.70 & 33.15 & 92.33 \\
            \rowcolor[gray]{0.9}        & \textbf{$\approach(\mu)+\texttt{ReAct}$}      & 32.82 & 92.93 & 48.62 & 88.59 & 13.60 & 96.83 & 28.19 & 94.89 & 30.81 & 93.31 \\
            \rowcolor[gray]{0.9}        & \textbf{$\approach(\sigma)+\texttt{ReAct}$}   & 37.53 & 91.22 & 51.32 & 87.19 & 10.18 & 97.31 & 27.21 & 95.12 & 31.56 & 92.71 \\
            \rowcolor[gray]{0.9}        & \textbf{$\approach(m)+\texttt{ReAct}$}        & 34.77 & 92.26 & 49.77 & 88.06 &  \textbf{8.69} & \textbf{97.76} & 24.08 & 95.66 & \textbf{29.33} & \textbf{93.43} \\
            \bottomrule
            \end{tabular}
            }
            \caption{\textit{Detailed Comparison with existing OOD detection methods on the ImageNet-1k benchmark, using ResNet-50 and MobileNet-v2. Methods marked with $^*$ were reproduced by us; results for all other methods are taken from their original publications. The symbol $\boldsymbol{\downarrow}$ indicates lower values are better; $\boldsymbol{\uparrow}$ indicates larger values are better.}}
            \label{table: other_baselines_imagenet_benchmark}
        \end{table*}

\subsection{Neural Collapse Inspired (NCI) OOD Detector}
 
    As shown in Table~\ref{table: NCI_cifar_comparison} for CIFAR-10 and Table~\ref{table: NCI_imagenet_comparison} for ImageNet, in a direct comparison against NCI's~\citep{NCI} reported results, our method $\approach$ demonstrates a clear and significant advantage. On CIFAR-10 with a ResNet-18 backbone, $\approach(m) + \texttt{ReAct}$ decisively outperforms NCI, reducing the average FPR95 by 33.43\%. This strong performance is maintained on the large-scale ImageNet benchmark, where our method reduces the FPR95 by 42.83\% on ResNet-50.

    \begin{table*}[ht]
        \centering
        
        \resizebox{0.95\linewidth}{!}{
        \begin{tabular}{l  cc cc cc cc}
            \toprule
             \multirow{2}{*}{\textbf{Method}} & \multicolumn{2}{c}{SVHN} & \multicolumn{2}{c}{Places365} & \multicolumn{2}{c}{Texture} & \multicolumn{2}{c}{Average}  \\
            \cmidrule(lr){2-3} \cmidrule(lr){4-5} \cmidrule(lr){6-7} \cmidrule(lr){8-9}
            & \textbf{FPR95} $\downarrow$ & \textbf{AUROC} $\uparrow$ & \textbf{FPR95} $\downarrow$ & \textbf{AUROC} $\uparrow$ & \textbf{FPR95} $\downarrow$ & \textbf{AUROC} $\uparrow$ & \textbf{FPR95} $\downarrow$ & \textbf{AUROC} $\uparrow$ \\
            \midrule
             \texttt{NCI}                        & 28.92 & 90.81 & 34.01 & 90.74 & 26.53 & 92.18 & 29.82 & 91.24 \\
            $\approach(\mu)$                     & 15.73 & 97.32 & 43.65 & 91.25 & 35.98 & 94.25 & 31.79 & 94.27 \\
            $\approach(\sigma)$                  & 10.33 & 98.13 & 37.74 & 92.68 & 24.31 & 96.16 & 24.13 & 95.66 \\
            $\approach(m)$                       &  9.93 & 98.24 & 36.59 & 92.97 & 23.01 & 96.43 & 23.18 & 95.88 \\
            $\approach(\mu) + \texttt{ReAct}$    & 14.37 & 97.48 & 43.18 & 91.46 & 25.04 & 95.82 & 27.53 & 94.92 \\
            $\approach(\sigma) + \texttt{ReAct}$ &  9.38 & 98.29 & 36.51 & 93.10 & 16.86 & 97.27 & 20.92 & 96.22 \\
            $\approach(m) + \texttt{ReAct}$      &  \textbf{8.86} & \textbf{98.39} & \textbf{35.04} & \textbf{93.38} & \textbf{15.64} & \textbf{97.48} & \textbf{19.85} & \textbf{96.42} \\
            \bottomrule
            \end{tabular}}
            \caption{\textit{A direct comparison of $\approach$ against the NCI baseline, using their originally reported results for CIFAR-10 with a ResNet-18 backbone. The evaluation is restricted to the SVHN, Texture, and Places365 OOD datasets to ensure a fair comparison that matches the protocol from the original NCI paper.}}
            \label{table: NCI_cifar_comparison}
    \end{table*}

    \begin{table*}[ht]
        \centering
        \resizebox{0.75\linewidth}{!}{
        \begin{tabular}{l cc cc cc}
            \toprule
             \multirow{2}{*}{\textbf{Method}} & \multicolumn{2}{c}{Texture} & \multicolumn{2}{c}{iNaturalist} & \multicolumn{2}{c}{Average}  \\
            \cmidrule(lr){2-3} \cmidrule(lr){4-5} \cmidrule(lr){6-7}
            & \textbf{FPR95} $\downarrow$ & \textbf{AUROC} $\uparrow$ & \textbf{FPR95} $\downarrow$ & \textbf{AUROC} $\uparrow$ & \textbf{FPR95} $\downarrow$ & \textbf{AUROC} $\uparrow$  \\
            \midrule
            \texttt{NCI}                         & 23.79 & 96.63 & 14.31 & 96.95 & 19.05 & 96.79 \\
            $\approach(\mu)$                     & 22.29 & 94.01 & 18.02 & 96.46 & 20.16 & 95.24 \\
            $\approach(\sigma)$                  & 15.85 & 95.94 & 19.05 & 96.21 & 17.45 & 96.08 \\
            $\approach(m)$                       & 16.08 & 95.88 & 19.00 & 96.18 & 17.54 & 96.03 \\
            $\approach(\mu) + \texttt{ReAct}$    & 12.11 & \textbf{97.38} &  \textbf{8.54} & \textbf{98.19} & \textbf{10.33} & \textbf{97.79} \\
            $\approach(\sigma) + \texttt{ReAct}$ & \textbf{12.04} & \textbf{97.38} &  9.10 & 98.06 & 10.57 & 97.72 \\
            $\approach(m) + \texttt{ReAct}$      & 12.06 & 97.31 &  9.71 & 97.97 & 10.89 & 97.64 \\                      
            \bottomrule
            \end{tabular}}
            \caption{\textit{A direct comparison of $\approach$ against the NCI baseline, using their originally reported results for ImageNet with a ResNet-50 backbone. The evaluation is restricted to the iNaturalist and Texture OOD datasets to ensure a fair comparison that matches the protocol from the original NCI paper.}}
            \label{table: NCI_imagenet_comparison}
    \end{table*}

\subsection{Fast Decision Boundary based OOD Detector}

    As shown in Tables \ref{table: fDBD_cifar_comparison} and \ref{table: fDBD_imagenet_comparison}, our method, $\approach$, demonstrates a decisive and substantial performance advantage over the fDBD's reported results in all comparable, overlapping settings. The strength of $\approach$ is most apparent when it is composed with existing techniques, creating a powerful synergistic effect that dramatically improves OOD detection. On CIFAR-10, this combination is particularly effective. Using a ResNet-18, $\approach(m) + \texttt{ReAct}$ slashes the average FPR95 by 44.80\% relative to fDBD (from 31.09\% down to 17.16\%). The gains are even more pronounced on a DenseNet-101, where $\approach(m) + \texttt{ReAct}$ achieves an FPR95 reduction of 34.67\%.

        \begin{table*}[ht]
            \centering
            \resizebox{\linewidth}{!}{
            \begin{tabular}{l l  cc cc cc cc cc}
                \toprule
                \multirow{2}{*}{\textbf{Model}} & \multirow{2}{*}{\textbf{Method}} & \multicolumn{2}{c}{SVHN} & \multicolumn{2}{c}{Places365} & \multicolumn{2}{c}{iSUN} & \multicolumn{2}{c}{Texture} & \multicolumn{2}{c}{Average}  \\
                \cmidrule(lr){3-4} \cmidrule(lr){5-6} \cmidrule(lr){7-8} \cmidrule(lr){9-10} \cmidrule(lr){11-12}
                && \textbf{FPR95} $\downarrow$ & \textbf{AUROC} $\uparrow$ & \textbf{FPR95} $\downarrow$ & \textbf{AUROC} $\uparrow$ & \textbf{FPR95} $\downarrow$ & \textbf{AUROC} $\uparrow$ & \textbf{FPR95} $\downarrow$ & \textbf{AUROC} $\uparrow$ & \textbf{FPR95} $\downarrow$ & \textbf{AUROC} $\uparrow$ \\
                \midrule
                \multirow{7}{*}{\textbf{ResNet-18}} & \texttt{fDBD}                    & 22.58 & 96.07 & 46.59 & 90.40 & 23.96 & 95.85 & 31.24 & 94.48 & 31.09 & 94.20 \\
                                                & $\approach(\mu)$                     & 15.73 & 97.32 & 43.65 & 91.25 & 26.26 & 96.08 & 35.98 & 94.25 & 30.41 & 94.73 \\
                                                & $\approach(\sigma)$                  & 10.33 & 98.13 & 37.74 & 92.68 & 16.90 & 97.32 & 24.31 & 96.16 & 22.32 & 96.07 \\
                                                & $\approach(m)$                       &  9.93 & 98.24 & 36.59 & 92.97 & 14.89 & 97.61 & 23.01 & 96.43 & 21.11 & 96.31 \\
                                                & $\approach(\mu) + \texttt{ReAct}$    & 14.37 & 97.48 & 43.18 & 91.46 & 16.71 & 97.22 & 25.04 & 95.82 & 24.83 & 95.99 \\
                                                & $\approach(\sigma) + \texttt{ReAct}$ &  9.38 & 98.29 & 36.51 & 93.10 & 10.82 & 98.10 & 16.86 & 97.27 & 18.89 & 96.69 \\
                                                & $\approach(m) + \texttt{ReAct}$      &  \textbf{8.86} & \textbf{98.39} & \textbf{35.04} & \textbf{93.38} &  \textbf{9.08} & \textbf{98.32} & \textbf{15.64} & \textbf{97.48} & \textbf{17.16} & \textbf{96.89} \\
                \midrule
                \multirow{7}{*}{\textbf{DenseNet-101}} & \texttt{fDBD}                 &  5.89 & 98.67 & 39.52 & 91.53 & 5.90 & 98.75 & 22.75 & 95.81 & 18.52 & 96.19 \\
                                                & $\approach(\mu)$                     & 15.12 & 97.51 & 33.93 & 92.75 & 3.32 & 98.98 & 25.71 & 94.88 & 19.52 & 96.53 \\
                                                & $\approach(\sigma)$                  & 10.95 & 98.11 & 32.51 & 92.99 & 1.73 & 99.47 & 18.26 & 96.34 & 15.86 & 96.73 \\
                                                & $\approach(m)$                       & 10.86 & 98.13 & 31.69 & 93.16 & 1.64 & 99.48 & 17.93 & 96.51 & 15.53 & 96.82 \\
                                                & $\approach(\mu) + \texttt{ReAct}$    &  5.82 & 98.76 & 31.59 & 93.50 & 2.87 & 99.15 & 16.91 & 96.83 & 14.30 & \textbf{97.56} \\
                                                & $\approach(\sigma) + \texttt{ReAct}$ &  \textbf{5.82} & 98.83 & 30.35 & 93.71 & \textbf{1.49} & 99.54 & 11.26 & 97.78 & 12.23 & 97.47 \\
                                                & $\approach(m) + \texttt{ReAct}$      &  5.86 & \textbf{98.86} & \textbf{29.97} & \textbf{93.89} & \textbf{1.49} & \textbf{99.55} & \textbf{11.06} & \textbf{97.88} & \textbf{12.10} & 97.55 \\
                \bottomrule
                \end{tabular}}
                \caption{\textit{Direct comparison of $\approach$ against the fDBD baseline on CIFAR-10. To ensure a fair comparison, the evaluation is restricted to the four OOD datasets reported in the original fDBD paper: SVHN, Places365, iSUN, and Texture.}}
                \label{table: fDBD_cifar_comparison}
        \end{table*}

    As demonstrated in Table~\ref{table: fDBD_imagenet_comparison}, this commanding performance extends to the large-scale ImageNet benchmark. While fDBD struggles with a high average FPR95 of 51.19\%, our $\approach(m) + \texttt{ReAct}$ achieves an FPR95 of just 17.64\% -- a massive 65.54\% relative reduction. These results validate $\approach$ performed significantly better than recent baseliens like NCI and fDBD.

        \begin{table*}[ht]
            \centering
            \resizebox{\linewidth}{!}{
            \begin{tabular}{l l  cc cc cc cc cc}
                \toprule
                \multirow{2}{*}{\textbf{Method}} & \multicolumn{2}{c}{SUN} & \multicolumn{2}{c}{Places365} & \multicolumn{2}{c}{Texture} & \multicolumn{2}{c}{iNaturalist} & \multicolumn{2}{c}{Average}  \\
                \cmidrule(lr){2-3} \cmidrule(lr){4-5} \cmidrule(lr){6-7} \cmidrule(lr){8-9} \cmidrule(lr){10-11}
                & \textbf{FPR95} $\downarrow$ & \textbf{AUROC} $\uparrow$ & \textbf{FPR95} $\downarrow$ & \textbf{AUROC} $\uparrow$ & \textbf{FPR95} $\downarrow$ & \textbf{AUROC} $\uparrow$ & \textbf{FPR95} $\downarrow$ & \textbf{AUROC} $\uparrow$ & \textbf{FPR95} $\downarrow$ & \textbf{AUROC} $\uparrow$ \\
                \midrule
                 \texttt{fDBD}                      & 60.60 & 86.97 & 66.40 & 84.27 & 37.50 & 92.12 & 40.24 & 93.67 & 51.19 & 89.26 \\
                $\approach(\mu)$                    & 30.79 & 92.67 & 42.59 & 89.78 & 22.29 & 94.01 & 18.02 & 96.46 & 28.42 & 93.23 \\
                $\approach(\sigma)$                 & 35.73 & 91.47 & 48.35 & 88.04 & 15.85 & 95.94 & 19.05 & 96.21 & 29.75 & 92.92 \\
                $\approach(m)$                      & 35.79 & 91.40 & 48.68 & 87.82 & 16.08 & 95.88 & 19.00 & 96.18 & 29.89 & 92.82 \\
                $\approach(\mu)+\texttt{ReAct}$     & \textbf{18.46} & \textbf{95.82} & \textbf{28.98} & \textbf{93.31} & 12.11 & \textbf{97.38} &  \textbf{8.54} & \textbf{98.19} & \textbf{17.02} & \textbf{96.18} \\
                $\approach(\sigma)+\texttt{ReAct}$  & 19.13 & 95.61 & 29.58 & 93.04 & \textbf{12.04} & \textbf{97.38} &  9.10 & 98.06 & 17.46 & 96.02 \\
                $\approach(m)+\texttt{ReAct}$       & 19.02 & 95.52 & 29.77 & 92.92 & 12.06 & 97.31 &  9.71 & 97.97 & 17.64 & 95.93 \\                              
                \bottomrule
                \end{tabular}}
                \caption{\textit{This table presents a direct comparison of $\approach$ against the fDBD baseline on ImageNet using a ResNet-50 backbone. To ensure a fair comparison, the evaluation is restricted to the iNaturalist and Texture OOD datasets, matching the protocol in the original fDBD paper.}}
                \label{table: fDBD_imagenet_comparison}
        \end{table*}

\subsection{AdaSCALE OOD Detection}

    AdaSCALE is a post-hoc OOD detection method that replaces fixed activation-scaling strategies with an adaptive, sample-dependent mechanism. Existing approaches (ASH, SCALE, LTS) prune activations using a static percentile threshold, which cannot reliably distinguish ID from OOD data. AdaSCALE leverages the observation that OOD samples experience larger shifts in their top activated neurons under small pixel perturbations, while ID activations remain stable. It measures this activation shift (Q), adjusts it with a correction term (Co), and maps the resulting OODness score through a CDF to produce a dynamic pruning percentile. This causes ID samples to receive stronger scaling and OOD samples weaker scaling, yielding more separated energy scores and improved detection performance.

    We compare $\approach$ against AdaScale in Tables~\ref{table: adascale_cifar_comparison} and \ref{table: adascale_imagenet_comparison}, strictly following the restricted dataset protocol from the original AdaScale paper for a fair comparison. On the CIFAR (DenseNet-101) benchmark, $\approach$ consistently outperforms AdaScale. For instane, on CIFAR-10, the best baseline AdaScale-L achieves an average FPR95 of 40.03\%. Our standalone $\approach(m)$ is already significantly better at 20.16\%, and our combined $\approach(m) + \texttt{ReAct}$ further extends this lead to 15.63\%. On CIFAR-100, our $\approach(m) + \texttt{ReAct}$ (FPR95 39.46\%) likewise outperforms the best AdaScale-A baseline (58.42\%). This trend holds on the ImageNet (ResNet-50) benchmark. As shown in Table~\ref{table: adascale_imagenet_comparison}, our $\approach(\mu) + \texttt{ReAct}$ (FPR95 16.54\%) achieves similar level of performance compared to best AdaScale-L (16.92\%).

    \begin{table*}[ht]
        \centering
        
        \resizebox{0.95\linewidth}{!}{
        \begin{tabular}{l l  cc cc cc cc}
            \toprule
             \multirow{2}{*}{\textbf{Dataset}} & \multirow{2}{*}{\textbf{Method}} & \multicolumn{2}{c}{SVHN} & \multicolumn{2}{c}{Places365} & \multicolumn{2}{c}{Texture} & \multicolumn{2}{c}{Average}  \\
            \cmidrule(lr){3-4} \cmidrule(lr){5-6} \cmidrule(lr){7-8} \cmidrule(lr){9-10}
            && \textbf{FPR95} $\downarrow$ & \textbf{AUROC} $\uparrow$ & \textbf{FPR95} $\downarrow$ & \textbf{AUROC} $\uparrow$ & \textbf{FPR95} $\downarrow$ & \textbf{AUROC} $\uparrow$ & \textbf{FPR95} $\downarrow$ & \textbf{AUROC} $\uparrow$ \\
            \midrule
            \multirow{8}{*}{CIFAR-10} & \texttt{AdaScale-L}     & 25.04 & 94.05 & 36.77 & 91.20 & 58.28 & 87.35 & 40.03 & 90.87 \\
                        & \texttt{AdaScale-A}                   & 26.43 & 93.87 & 37.03 & 91.25 & 58.59 & 87.19 & 40.68 & 90.77 \\
                        & $\approach(\mu)$                      & 15.12 & 97.51 & 33.93 & 92.75 & 25.71 & 94.88 & 24.92 & 95.05 \\
                        & $\approach(\sigma)$                   & 10.95 & 98.11 & 32.51 & 92.99 & 18.26 & 96.34 & 20.57 & 95.15 \\
                        & $\approach(m)$                        & 10.86 & 98.13 & 31.69 & 93.16 & 17.93 & 96.51 & 20.16 & 95.27 \\
                        & $\approach(\mu)+\texttt{ReAct}$       &  \textbf{5.82} & 98.76 & 31.59 & 93.50 & 16.91 & 96.83 & 18.77 & 96.36 \\
                        & $\approach(\sigma)+\texttt{ReAct}$    &  \textbf{5.82} & 98.83 & 30.35 & 93.71 & 11.26 & 97.78 & 15.81 & 96.77 \\
                        & $\approach(m)+\texttt{ReAct}$         &  5.86 & \textbf{98.86} & \textbf{29.97} & \textbf{93.89} & \textbf{11.06} & \textbf{97.88} & \textbf{15.63} & \textbf{96.88} \\
            \cmidrule(lr){2-10}
            \multirow{8}{*}{CIFAR-100} & \texttt{AdaScale-L}    & 46.29 & 84.31 & \textbf{61.70} & \textbf{78.86} & 71.40 & 76.59 & 59.13 & 79.25 \\
                        & \texttt{AdaScale-A}                   & 43.97 & 85.30 & 61.97 & 78.69 & 69.31 & 77.71 & 58.42 & 80.57 \\
                        & $\approach(\mu)$                      & 22.45 & 96.11 & 78.72 & 77.16 & 52.09 & 83.58 & 51.09 & 85.62 \\
                        & $\approach(\sigma)$                   & 21.13 & 96.30 & 78.19 & 77.16 & 44.34 & 86.19 & 47.89 & 86.55 \\
                        & $\approach(m)$                        & 19.90 & 96.45 & 77.30 & 77.67 & 42.48 & 87.12 & 46.56 & 87.08 \\
                        & $\approach(\mu)+\texttt{ReAct}$       & \textbf{11.73} & \textbf{97.67} & 83.17 & 74.06 & 26.12 & 93.93 & 40.34 & 88.55 \\
                        & $\approach(\sigma)+\texttt{ReAct}$    & 14.13 & 97.32 & 83.87 & 74.36 & 23.21 & 94.55 & 40.40 & 88.74 \\
                        & $\approach(m)+\texttt{ReAct}$         & 13.70 & 97.36 & 83.00 & 75.14 & \textbf{21.68} & \textbf{94.95} & \textbf{39.46} & \textbf{89.15} \\
            \bottomrule

            \end{tabular}}
            \caption{\textit{A direct comparison of $\approach$ against the AdaScale baseline, using their originally reported results for CIFAR with a DenseNet-101 backbone. The evaluation is restricted to the SVHN, Texture, and Places365 OOD datasets to ensure a fair comparison that matches the protocol from the original AdaScale paper~\cite{adascle}.}}
            \label{table: adascale_cifar_comparison}
    \end{table*}

    \begin{table*}[ht]
        \centering
        \resizebox{0.75\linewidth}{!}{
        \begin{tabular}{l cc cc cc cc}
            \toprule
             \multirow{2}{*}{\textbf{Method}} & \multicolumn{2}{c}{Places} & \multicolumn{2}{c}{Texture} & \multicolumn{2}{c}{iNaturalist} & \multicolumn{2}{c}{Average}  \\
            \cmidrule(lr){2-3} \cmidrule(lr){4-5} \cmidrule(lr){6-7} \cmidrule(lr){8-9}
            & \textbf{FPR95} $\downarrow$ & \textbf{AUROC} $\uparrow$ & \textbf{FPR95} $\downarrow$ & \textbf{AUROC} $\uparrow$ & \textbf{FPR95} $\downarrow$ & \textbf{AUROC} $\uparrow$ & \textbf{FPR95} $\downarrow$ & \textbf{AUROC} $\uparrow$  \\
            \midrule
            \texttt{AdaScale-L}                  & 32.60 & 92.74 & 10.57 & 97.88 &  7.61 & 98.31 & 16.92 & 96.98 \\
            \texttt{AdaScale-A}                  & 32.97 & 92.63 & \textbf{10.33} & \textbf{97.92} &  \textbf{7.78} & \textbf{98.29} & 17.03 & 96.95 \\
            $\approach(\mu)$                     & 42.59 & 89.78 & 22.29 & 94.01 & 18.02 & 96.46 & 27.63 & 93.42 \\
            $\approach(\sigma)$                  & 48.35 & 88.04 & 15.85 & 95.94 & 19.05 & 96.21 & 27.08 & 93.40 \\
            $\approach(m)$                       & 48.68 & 87.82 & 16.08 & 95.88 & 19.00 & 96.18 & 27.25 & 93.29 \\
            $\approach(\mu) + \texttt{ReAct}$    & \textbf{28.98} & \textbf{93.31} & 12.11 & 97.38 &  8.54 & 98.19 & \textbf{16.54} & \textbf{96.96 }\\
            $\approach(\sigma) + \texttt{ReAct}$ & 29.58 & 93.04 & 12.04 & 97.38 &  9.10 & 98.06 & 16.91 & 96.83 \\
            $\approach(m) + \texttt{ReAct}$      & 29.77 & 92.92 & 12.06 & 97.31 &  9.71 & 97.97 & 17.18 & 96.73 \\ 
            \bottomrule
            \end{tabular}}
            \caption{\textit{A direct comparison of $\approach$ against the \texttt{AdaScale} baseline, using their originally reported results for ImageNet with a ResNet-50 backbone. The evaluation is restricted to the Places, Texture, and iNaturalist OOD datasets to ensure a fair comparison that matches the protocol from the original AdaScale paper~\cite{adascle}.}}
            \label{table: adascale_imagenet_comparison}
    \end{table*}

\section{Reproducibility Statement}
\label{appendix: reproducibility}

    We are committed to ensuring the reproducibility of our research. To this end, we provide detailed information regarding our code, experimental setup, hyperparameter selection, and computational environment.

    \subsection{Code and Data Availability}
    The complete source code used in $\approach$, along with the scripts used to run all experiments and generate figures, is publicly available on GitHub.\footnote{https://github.com/bingabid/Catalyst} We will also provide the model weights for our trained CIFAR models. All datasets used in this work (CIFAR-10, CIFAR-100, ImageNet-1k, and all OOD benchmarks) are publicly available and were used without modification, following the standard preprocessing steps described in their original publications and common benchmarks.

    \noindent
    \textbf{Experimental Setup.}
    \begin{itemize}
        \item \textbf{CIFAR Benchmarks:} Our primary models include ResNet-18 and DenseNet-101. Following established protocols~\citep{ReAct,DICE,ASH, SCALE, cider}, all models were trained from scratch for 100 epochs using SGD with a momentum of 0.9, a weight decay of 0.0001, and a batch size of 64. The learning rate was initialized at 0.1 and decayed by a factor of 10 at epochs 50, 75, and 90.
        
        \item \textbf{ImageNet Benchmark:} For our large-scale experiments, we used the official pre-trained models provided by PyTorch for ResNet-34, ResNet-50, MobileNet-v2, and DenseNet-121. No fine-tuning was performed.
    \end{itemize}

\subsection{Hyperparameter Selection} 

    The clipping threshold $\boldsymbol{c}$ (Eq.~\ref{eq: gamma_c}) is crucial for enhanced performance, as it must be set to optimally distinguish ID from OOD data. Analogous to ReAct~\cite{ReAct}, we do not tune $\boldsymbol{c}$ directly; instead, we control it by setting it to the $p$-th percentile of the ID activation distribution (e.g, when $p = 95$, it indicates that 95\% of the ID activations are less than the threshold $\boldsymbol{c}$). The choice of this percentile $p$ is the key hyperparameter to be tuned. To select the optimal $p$, we follow established protocols from prior work~\citep{DICE,ReAct} and create a proxy OOD validation set, which is generated by adding pixel-wise Gaussian noise $\mathcal{N}(0, 0.2)$ to images from the ID validation set. We then select the percentile $p$ that yields the best OOD separation on this proxy task. This two-step procedure -- using a percentile for the mechanism and a proxy set for tuning -- is a robust tuning strategy grounded in prior work. The selected $p$ values are: 
    \begin{itemize}
        \item \textbf{CIFAR:} We found it optimal to tune the percentile for each statistic individually. These values are fixed for all CIFAR models (ResNet-18, DenseNet-101) and across all baselines (e.g., Energy, ReAct). The selected percentiles are $p_{\text{mean}} = 60$, $p_{\text{std}} = 95$, and $p_{\text{max}} = 95$.
    
        \item \textbf{ImageNet:} For our default method ($\approach + \text{Energy}$), a single percentile of $p = 75$ is used for all ImageNet models. When combining with baselines like ReAct ($\approach + \texttt{ReAct}$), we found it optimal to use a single, shared percentile $p$ across all three statistics (mean, std, and max). The optimal shared percentile $p$ varies by model: $p = 15$ for ResNet-34 and ResNet-50, $p = 35$ for MobileNet-v2, and $p = 52$ for DenseNet-121.  %
    \end{itemize}

    As our hyperparameter search for ImageNet demonstrated, the optimal shared percentile $p$ varies across different architectures (e.g., $p=15$ for ResNet-50 vs. $p=52$ for DenseNet-121). This empirical finding is highly intuitive and aligns with our method's design. The optimal $p$ (which sets the clipping threshold $c$) is naturally coupled with the model's architecture, particularly the dimension ($n$) of the penultimate layer. This is because our scaling factor $\gamma$ (Equation~\ref{eq: gamma_design}) is an aggregation (a sum) over all $n$ channels. A model with a larger channel dimension (e.g., ResNet-50, $n=2048$) will produce a sum of a very different magnitude than a model with a smaller dimension (e.g., DenseNet-121, $n=1024$). Therefore, a different clipping percentile $p$ is required for each architecture to produce the most discriminative $\gamma$ signal.

\subsection{Baseline Hyperparameter Tuning}

    A core principle of our evaluation is to ensure a fair and rigorous comparison against all baselines. For all methods (ODIN, ReAct, DICE, ASH, SCALE, KNN), we strictly followed the hyperparameter selection protocols described in their respective papers.

    When re-evaluating these baselines on new architectures not present in their original work, we performed a new hyperparameter search using the same validation procedures and search spaces they described. This ensures that every baseline is as strong as possible for each specific model. Key hyperparameters for these methods are summarized below:

    \begin{itemize}
        \item \textbf{ODIN:} We adopted the optimal hyperparameter values reported in the original publication. Accordingly, we set the temperature to $T = 1000$, with a noise magnitude $\epsilon$ of 0.004 for CIFAR and 0.0015 for ImageNet.
        
        \item \textbf{ReAct:} The clipping percentile $p$ was selected from $\{85, 90, 95\}$. While we found $p=90$ to be optimal for the standalone ReAct baseline, consistent with the original paper, the optimal value shifted to $p=95$ when ReAct was combined with our $\approach$.
        
        \item \textbf{DICE:} We selected the sparsity ratio $p$ from $\{70, 75, 80, 85, 90, 95\}$. Our validation process consistently identified $p=70\%$ as the optimal value.
        
        \item \textbf{ASH:} The pruning percentile $p$ was selected from $\{80, 85, 90\}$. The optimal value was found to be dependent on the dataset and architecture. We report the specific optimal value for each major setting to ensure the strongest and fairest possible comparison.
        
        \begin{itemize}
            \item For \textbf{ImageNet}, the optimal value was consistently $p=90$ for most architectures, with the exception of EfficientNet-b0, which required a less aggressive pruning of $p=50$.
            
           \item For \textbf{CIFAR-10}, the optimal values were $p=80$ for both ResNet models, $p=90$ for DenseNet, and $p=70$ for MobileNet-v2. These values held for both the standalone baseline and when combined with $\approach$.
            
            \item For \textbf{CIFAR-100}, the optimal value for the ResNet models was consistently $p=80$. For other architectures, we observed an interaction effect: the optimal percentile for DenseNet shifted from $p=90$ (baseline) to $p=80$ (with $\approach$), and for MobileNet-v2, it shifted from $p=90$ to $p=85$.
        \end{itemize}
    
        \item \textbf{SCALE:} For the SCALE baseline, the pruning percentile p was set to a fixed value of $p=85$ across all experiments. We adopted this value directly from the original SCALE paper~\citep{SCALE} to ensure our re-implementation was consistent with the authors' reported optimal setting, providing a fair comparison.
    \end{itemize}

\subsection{Computational Environment}
    All CIFAR model training and OOD detection experiments were conducted on an Apple M2 Max system with 96 GB of RAM. The experiments were implemented in Python using PyTorch (v2.1) and the Torchvision library.

\section{Choice of Layer for Computing $\gamma$}
\label{appendix: penultimate layer}

\begin{figure*}[ht]
    \centering
    \includegraphics[width=0.97\textwidth]{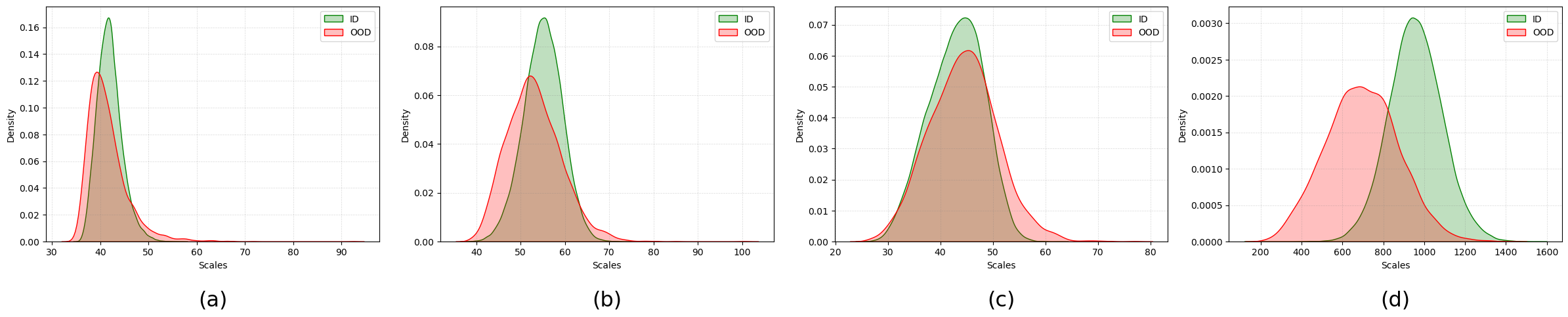}
    \caption{\textit{Distributions of the scaling factor ($\gamma$) are computed from the four residual stages. The model was trained on ImageNet-1K (ID) and evaluated against Texture (OOD). \textit{(a-c)} The $\gamma$ distributions from the early-to-mid stages (Layer 1 to Layer 3) show significant overlap between ID and OOD samples, rendering them ineffective as a discriminative signal. \textit{(d)} In sharp contrast, the distribution from the final residual stage (Layer 4) provides a clear and distinct separation. This result, consistent across OOD datasets, validates our methodological focus on the penultimate layer's pre-pooling feature map as the most potent and reliable signal source.}}
    \label{fig: intermediate_layers_places}
\end{figure*}

A necessary condition for Catalyst to be effective is that its scaling factor ($\gamma$) must be inherently distinguishable between ID and OOD samples. Consequently, a core methodological decision is to identify which network stage produces the most discriminative information cues. To justify our focus on the penultimate layer, we conducted an analysis to locate the most potent source of information cues within the network. We used a ResNet-50 model trained on ImageNet (ID) and computed $\gamma$ using the channel-wise average activation from each of its four main residual stages (Layer 1-4). We then compared the ID $\gamma$ distribution against multiple OOD datasets, including Places, SUN, Texture, and iNaturalist. 

A clear and consistent trend emerged, as illustrated representatively in Figure~\ref{fig: intermediate_layers_places} for the Texture dataset. The analysis reveals that the $\gamma$ distributions from the early-to-mid stages (Layer 1-3) are not sufficiently discriminative. As shown in the Figure~\ref{fig: intermediate_layers_places}, they exhibit high overlap overlap between ID (ImageNet) and OOD (Texture) samples, rendering them ineffective for our scaling purposes. In sharp contrast, the signal from the final residual stage (Layer 4) demonstrates a sufficiently clear between the two distributions. This finding was consistent across all tested OOD datasets.

This analysis empirically validates our methodological focus. The penultimate layer's pre-pooling feature map is not just a layer of convenience; it is the most reliable source of a potent signal for constructing $\gamma$. While we only illustrate this with ResNet-50 for brevity, we observed this same trend with the final feature map consistently providing the most signal separation across all tested architectures. This confirms our choice is a generalizable one, not specific to a single model.

\section{Analysis of Fusion Strategy}
\label{appendix: fusion strategy}

In Section~\ref{sec: method}, we introduced two potential fusion strategies for integrating our scaling factor $\gamma(\mathbf{x})$ with a baseline score $S(\mathbf{x})$: multiplicative (Eq.~\ref{eq: adaptive_fusion_a}) and additive (Eq.~\ref{eq: adaptive_fusion_b}).

    \begin{subequations}
        \begin{align}
            S^{*}(\mathbf{x}; \theta, \gamma) &= \gamma(\mathbf{x}; f) \times S(\mathbf{x}; \theta) \label{eq: adaptive_fusion_a} \\
            S^{+}(\mathbf{x}; \theta, \gamma) &= \gamma(\mathbf{x}; f) + S(\mathbf{x}; \theta) \label{eq: adaptive_fusion_b}
        \end{align}
        \label{eq: adaptive_fusion}
    \end{subequations}

As shown in Table~\ref{table: fusion_strategy_imagenet}, a comparative evaluation on the ImageNet benchmark reveals that both strategies can achieve a similar level of performance. This confirms that the core discriminative power originates from the $\gamma$ signal itself, not the specific mathematical operator. 

However, a critical distinction emerged when analyzing their hyperparameter sensitivity and robustness. The scaling factors $\gamma^*$ (multiplicative) and $\gamma^+$ (additive) are both derived from Eq.~\ref{eq: gamma_design}, but they require different optimal settings for the clipping threshold, $c^*$ and $c^+$ respectively, as detailed in Equation~\ref{eq: fusion_gamma}. 

    \begin{subequations}
        \begin{align}
            \gamma^{*}(\mathbf{x};f) = \sum_{i=1}^n \min(f_{i}(\mathbf{x}), c^{*})  \label{eq: fusion_gamma_a} \\
            \gamma^{+}(\mathbf{x};f) = \sum_{i=1}^n \min(f_{i}(\mathbf{x}), c^{+}) \label{eq: fusion_gamma_b}
        \end{align}
        \label{eq: fusion_gamma}
    \end{subequations}

For proposed multiplicative strategy, the optimal threshold $c^*$ is set by selecting a moderate percentile as detailed in reproducibility section in Appendix~\ref{appendix: reproducibility} of the ID training data's $f_i(\mathbf{x})$ values. This procedure is robust, stable, and aligns with foundational post-hoc methods like ReAct and SCALE (details in Appendix~\ref{appendix: reproducibility}). For the additive strategy, we empirically found that the optimal threshold $c^+$ was consistently an order of magnitude smaller, (e.g., $c^+ \approx 0.1 \times c^*$). On ImageNet, this optimal $c^+$ value corresponds to an extremely low percentile of the ID data (e.g., below the 1st percentile for ResNet-50).

This low-percentile tuning makes the additive strategy operationally fragile. Tuning at the extreme low activation threshold could be fragile and sensitive to small shifts in data or model, making it a poor choice for a general-purpose method. Therefore, while both methods achieve similar performance, we chose multiplicative fusion as our primary strategy. It provides not only competitive performance but also the practical robustness and hyperparameter stability required of a \textit{plug-and-play} framework. This choice aligns conceptually with our \textit{elastic scaling} narrative, where $\gamma$ acts as an input-dependent modulator of the baseline score.

\begin{table*}[ht]
\centering
\resizebox{\textwidth}{!}{
\begin{tabular}{l l l  cc cc cc cc cc}
    \toprule
    \multirow{2}{*}{\textbf{Model}} & \multirow{2}{*}{\textbf{Fusion}} & \multirow{2}{*}{\textbf{Method}} & \multicolumn{2}{c}{SUN} & \multicolumn{2}{c}{Places365} & \multicolumn{2}{c}{Texture} & \multicolumn{2}{c}{iNaturalist} & \multicolumn{2}{c}{Average}  \\
    \cmidrule(lr){4-5} \cmidrule(lr){6-7} \cmidrule(lr){8-9} \cmidrule(lr){10-11} \cmidrule(lr){12-13}
    &&& \textbf{FPR95} $\downarrow$ & \textbf{AUROC} $\uparrow$ & \textbf{FPR95} $\downarrow$ & \textbf{AUROC} $\uparrow$ & \textbf{FPR95} $\downarrow$ & \textbf{AUROC} $\uparrow$ & \textbf{FPR95} $\downarrow$ & \textbf{AUROC} $\uparrow$ & \textbf{FPR95} $\downarrow$ & \textbf{AUROC} $\uparrow$ \\
    \midrule
    \multirow{14}{*}{\rotatebox{90}{ResNet-34}} & \multirow{7}{*}{\rotatebox{90}{$\approach(+)$}} & Energy       & 57.39 & 86.59 & 62.61 & 84.59 & 54.95 & 86.45 & 53.86 & 89.73 & 57.20 & 86.84 \\
                                                                && \textbf{$\approach(\mu)$}                     & 35.50 & 91.69 & 46.05 & 89.07 & 12.66 & 96.66 & 20.03 & 96.20 & 28.56 & 93.41 \\
                                                                && \textbf{$\approach(\sigma)$}                  & 39.93 & 90.29 & 49.61 & 87.65 & 14.45 & 95.83 & 25.53 & 95.03 & 32.38 & 92.20 \\
                                                                && \textbf{$\approach(m)$}                       & 45.07 & 89.94 & 57.08 & 86.80 & 10.46 & 97.61 & 26.55 & 95.17 & 34.79 & 92.38 \\
                                                                && \textbf{$\approach(\mu)+\texttt{ReAct}$}      & 22.23 & 94.80 & 32.07 & 92.09 & 18.19 & 95.91 & 15.96 & 96.89 & 22.11 & 94.92 \\
                                                                && \textbf{$\approach(\sigma)+\texttt{ReAct}$}   & 23.10 & 94.66 & 33.18 & 91.92 & 17.43 & 96.07 & 17.18 & 96.75 & 22.72 & 94.85 \\
                                                                && \textbf{$\approach(m)+\texttt{ReAct}$}        & 37.74 & 92.96 & 49.69 & 89.86 & 11.21 & 97.60 & 23.36 & 96.01 & 30.50 & 94.11 \\
    \cmidrule(lr){2-13}
                                & \multirow{7}{*}{\rotatebox{90}{$\approach(*)$}} & Energy       & 57.39 & 86.59 & 62.61 & 84.59 & 54.95 & 86.45 & 53.86 & 89.73 & 57.20 & 86.84 \\
                                                && \textbf{$\approach(\mu)$}                     & 33.46 & 91.85 & 43.78 & 89.39 & 24.86 & 93.39 & 25.60 & 94.99 & 31.92 & 92.41 \\
                                                && \textbf{$\approach(\sigma)$}                  & 37.78 & 90.74 & 48.87 & 87.82 & 16.08 & 95.80 & 24.90 & 95.10 & 31.91 & 92.36 \\
                                                && \textbf{$\approach(m)$}                       & 36.90 & 90.88 & 48.37 & 87.87 & 16.83 & 95.58 & 25.22 & 95.00 & 31.83 & 92.34 \\
                                                && \textbf{$\approach(\mu)+\texttt{ReAct}$}      & 21.44 & 95.18 & 31.74 & 92.56 & 13.39 & 97.02 & 12.81 & 97.47 & 19.84 & 95.56 \\
                                                && \textbf{$\approach(\sigma)+\texttt{ReAct}$}   & 21.80 & 95.06 & 32.11 & 92.41 & 12.73 & 97.11 & 13.01 & 97.41 & 19.91 & 95.50 \\
                                                && \textbf{$\approach(m)+\texttt{ReAct}$}        & 22.03 & 94.99 & 32.58 & 92.31 & 12.55 & 97.15 & 13.47 & 97.33 & 20.16 & 95.44 \\
    \midrule
    \multirow{14}{*}{\rotatebox{90}{ResNet-50}} & \multirow{7}{*}{\rotatebox{90}{$\approach(+)$}} & Energy       & 58.28 & 86.73 & 65.40 & 84.13 & 52.29 & 86.73 & 53.95 & 90.59 & 57.48 & 87.05 \\
                                                                && \textbf{$\approach(\mu)$}                     & 30.99 & 93.13 & 43.36 & 89.97 &  9.72 & 97.71 & 12.94 & 97.45 & 24.25 & 94.57 \\
                                                                && \textbf{$\approach(\sigma)$}                  & 34.07 & 91.87 & 45.43 & 88.64 & 10.44 & 97.19 & 15.25 & 96.88 & 26.30 & 93.64 \\
                                                                && \textbf{$\approach(m)$}                       & 42.08 & 91.70 & 55.43 & 88.17 &  9.47 & 98.10 & 20.81 & 96.44 & 31.95 & 93.60 \\
                                                                && \textbf{$\approach(\mu)+\texttt{ReAct}$}      & 19.73 & 95.28 & 29.47 & 92.73 & 12.68 & 97.10 &  9.99 & 97.89 & 17.97 & 95.75 \\
                                                                && \textbf{$\approach(\sigma)+\texttt{ReAct}$}   & 20.64 & 95.12 & 30.94 & 92.51 & 10.53 & 97.56 & 10.25 & 97.79 & 18.09 & 95.74 \\
                                                                && \textbf{$\approach(m)+\texttt{ReAct}$}        & 37.59 & 93.66 & 50.32 & 90.69 &  9.88 & 98.07 & 18.78 & 96.93 & 29.14 & 94.84 \\
    \cmidrule(lr){2-13}
                                & \multirow{7}{*}{\rotatebox{90}{$\approach(*)$}} & Energy       & 58.28 & 86.73 & 65.40 & 84.13 & 52.29 & 86.73 & 53.95 & 90.59 & 57.48 & 87.05 \\
                                                && \textbf{$\approach(\mu)$}                     & 30.79 & 92.67 & 42.59 & 89.78 & 22.29 & 94.01 & 18.02 & 96.46 & 28.42 & 93.23 \\
                                                && \textbf{$\approach(\sigma)$}                  & 35.73 & 91.47 & 48.35 & 88.04 & 15.85 & 95.94 & 19.05 & 96.21 & 29.75 & 92.92 \\
                                                && \textbf{$\approach(m)$}                       & 35.79 & 91.40 & 48.68 & 87.82 & 16.08 & 95.88 & 19.00 & 96.18 & 29.89 & 92.82 \\
                                                && \textbf{$\approach(\mu)+\texttt{ReAct}$}      & 18.46 & 95.82 & 28.98 & 93.31 & 12.11 & 97.38 & 8.54 & 98.19 & 17.02 & 96.18 \\
                                                && \textbf{$\approach(\sigma)+\texttt{ReAct}$}   & 19.13 & 95.61 & 29.58 & 93.04 & 12.04 & 97.38 & 9.10 & 98.06 & 17.46 & 96.02 \\
                                                && \textbf{$\approach(m)+\texttt{ReAct}$}        & 19.02 & 95.52 & 29.77 & 92.92 & 12.06 & 97.31 & 9.71 & 97.97 & 17.64 & 95.93 \\
    \midrule
    \multirow{14}{*}{\rotatebox{90}{MobileNet-v2}}  & \multirow{7}{*}{\rotatebox{90}{$\approach(+)$}} & Energy   & 59.36 & 86.24 & 66.27 & 83.21 & 54.54 & 86.58 & 55.31 & 90.34 & 58.87 & 86.59 \\
                                                                && \textbf{$\approach(\mu)$}                     & 39.84 & 90.64 & 54.28 & 86.57 & 12.20 & 96.47 & 32.62 & 94.00 & 34.73 & 91.92 \\
                                                                && \textbf{$\approach(\sigma)$}                  & 42.03 & 90.52 & 56.96 & 86.35 &  9.73 & 97.27 & 34.94 & 93.66 & 35.92 & 91.95 \\
                                                                && \textbf{$\approach(m)$}                       & 43.30 & 89.94 & 57.90 & 85.84 & 10.07 & 97.07 & 36.45 & 93.34 & 36.93 & 91.55 \\
                                                                && \textbf{$\approach(\mu)+\texttt{ReAct}$}      & 40.64 & 90.34 & 53.27 & 86.40 & 11.05 & 96.97 & 29.73 & 94.68 & 33.67 & 92.10 \\
                                                                && \textbf{$\approach(\sigma)+\texttt{ReAct}$}   & 41.41 & 90.61 & 55.85 & 86.46 &  8.17 & 97.78 & 31.46 & 94.37 & 34.22 & 92.31 \\
                                                                && \textbf{$\approach(m)+\texttt{ReAct}$}        & 41.98 & 90.31 & 55.76 & 86.19 &  8.19 & 97.71 & 31.75 & 94.27 & 34.42 & 92.12 \\
    \cmidrule(lr){2-13}
                                    & \multirow{7}{*}{\rotatebox{90}{$\approach(*)$}} & Energy   & 59.36 & 86.24 & 66.27 & 83.21 & 54.54 & 86.58 & 55.31 & 90.34 & 58.87 & 86.59 \\
                                                && \textbf{$\approach(\mu)$}                     & 37.74 & 91.43 & 52.21 & 87.33 & 23.42 & 94.17 & 33.47 & 93.84 & 36.71 & 91.69 \\
                                                && \textbf{$\approach(\sigma)$}                  & 38.20 & 91.26 & 53.04 & 86.84 & 14.02 & 96.37 & 29.25 & 94.63 & 33.63 & 92.27 \\
                                                && \textbf{$\approach(m)$}                       & 37.41 & 91.37 & 52.24 & 86.89 & 14.18 & 96.35 & 28.78 & 94.70 & 33.15 & 92.33 \\
                                                && \textbf{$\approach(\mu)+\texttt{ReAct}$}      & 32.82 & 92.93 & 48.62 & 88.59 & 13.60 & 96.83 & 28.19 & 94.89 & 30.81 & 93.31 \\
                                                && \textbf{$\approach(\sigma)+\texttt{ReAct}$}   & 37.53 & 91.22 & 51.32 & 87.19 & 10.18 & 97.31 & 27.21 & 95.12 & 31.56 & 92.71 \\
                                                && \textbf{$\approach(m)+\texttt{ReAct}$}        & 34.77 & 92.26 & 49.77 & 88.06 &  8.69 & 97.76 & 24.08 & 95.66 & 29.33 & 93.43 \\
    \midrule
    \multirow{14}{*}{\rotatebox{90}{DenseNet-121}}  & \multirow{7}{*}{\rotatebox{90}{$\approach(+)$}} & Energy   & 52.51 & 87.27 & 58.24 & 85.05 & 52.22 & 85.42 & 39.75 & 92.66 & 50.68 & 87.60 \\
                                                                && \textbf{$\approach(\mu)$}                     & 36.68 & 90.59 & 45.60 & 87.91 & 20.62 & 94.00 & 19.43 & 96.07 & 30.58 & 92.14 \\
                                                                && \textbf{$\approach(\sigma)$}                  & 35.10 & 90.97 & 44.54 & 88.22 & 16.90 & 95.11 & 18.40 & 96.24 & 28.73 & 92.64 \\
                                                                && \textbf{$\approach(m)$}                       & 38.70 & 91.31 & 50.39 & 88.13 & 11.86 & 97.36 & 21.68 & 95.89 & 30.66 & 93.17 \\
                                                                && \textbf{$\approach(\mu)+\texttt{ReAct}$}      & 38.24 & 91.84 & 47.00 & 88.61 & 14.34 & 97.03 & 17.80 & 96.50 & 29.35 & 93.49 \\
                                                                && \textbf{$\approach(\sigma)+\texttt{ReAct}$}   & 32.72 & 92.71 & 43.39 & 89.39 & 10.69 & 97.76 & 15.21 & 96.95 & 25.50 & 94.20 \\
                                                                && \textbf{$\approach(m)+\texttt{ReAct}$}        & 37.89 & 92.34 & 51.97 & 88.63 &  7.32 & 98.41 & 20.80 & 96.28 & 29.50 & 93.92 \\
    \cmidrule(lr){2-13}
                                    & \multirow{7}{*}{\rotatebox{90}{$\approach(*)$}} & Energy   & 52.51 & 87.27 & 58.24 & 85.05 & 52.22 & 85.42 & 39.75 & 92.66 & 50.68 & 87.60 \\
                                                && \textbf{$\approach(\mu)$}                     & 33.24 & 91.84 & 42.94 & 89.01 & 25.59 & 93.29 & 16.41 & 96.69 & 29.54 & 92.71 \\
                                                && \textbf{$\approach(\sigma)$}                  & 34.12 & 91.57 & 44.34 & 88.53 & 21.06 & 94.52 & 16.95 & 96.57 & 29.12 & 92.80 \\
                                                && \textbf{$\approach(m)$}                       & 34.29 & 91.47 & 44.74 & 88.35 & 21.26 & 94.43 & 17.50 & 96.46 & 29.45 & 92.68 \\
                                                && \textbf{$\approach(\mu)+\texttt{ReAct}$}      & 31.58 & 93.41 & 42.77 & 90.30 & 12.71 & 97.44 & 14.66 & 97.11 & 25.43 & 94.56 \\
                                                && \textbf{$\approach(\sigma)+\texttt{ReAct}$}   & 30.04 & 93.44 & 41.50 & 90.24 & 11.37 & 97.61 & 14.12 & 97.16 & 24.26 & 94.61 \\
                                                && \textbf{$\approach(m)+\texttt{ReAct}$}        & 30.25 & 93.37 & 41.61 & 90.13 & 11.74 & 97.52 & 14.48 & 97.09 & 24.52 & 94.53 \\
    \bottomrule
    \end{tabular}
    }
    \caption{\textit{Analysis of fusion strategies on detection performance on ImageNet benchmarks. All values are percentages and are averaged over four common OOD benchmark datasets. $\approach(+)$ represents additive strategy and $\approach(*)$ represents multiplicative strategy. The symbol $\boldsymbol{\downarrow}$ indicates lower values are better; $\boldsymbol{\uparrow}$ indicates larger values are better.}}
    \label{table: fusion_strategy_imagenet}
\end{table*}

\section{Alternate Statistics: Median and Entropy}
\label{appendix: alternate statistics}

To justify our final methodological choice of using mean, standard deviation, and maximum statistics, we conducted a rigorous analysis of two common alternatives: median and Shannon entropy. For a statistic to be viable for our framework, it must produce a scaling factor $\gamma$ that has a distinct and reliable signature for ID versus OOD samples.

\noindent
\textbf{Setup.} As we discussed in Section~\ref{sec: method} of main paper, we compute the scaling factor \(\gamma\) using a trained deep neural network \(\theta: \mathbb{R}^d \rightarrow \mathbb{R}^C\) that maps an input \(\mathbf{x} \in \mathbb{R}^d\) to a logit vector \(f(\mathbf{x}) \in \mathbb{R}^C\), where \(C = |\mathcal{Y}|\) denotes the number of output classes. The network’s penultimate layer produces a feature vector \(h(\mathbf{x}) \in \mathbb{R}^n\) by applying global average pooling to the activation map \(g(\mathbf{x}) \in \mathbb{R}^{n \times k \times k}\). Here, \(n\) is the number of channels, and each channel has spatial resolution \(k \times k\). A weight matrix \(\mathbf{W} \in \mathbb{R}^{n \times C}\) projects \(h(\mathbf{x})\) to the final logit vector.

\noindent
\textbf{Median.} We begin by extracting the \textit{median} from each activation map of \( g(\mathbf{x}) \in \mathbb{R}^{n \times k \times k} \), transforming it into an \(n\)-dimensional feature vector \( h(\mathbf{x}) \in \mathbb{R}^n \) using global median pooling, as defined in Equation~\ref{eq: median_pooling}:

\begin{equation}
    h(\mathbf{x}) = \texttt{median}\left( g(\mathbf{x}) \right)
    \label{eq: median_pooling}
\end{equation}

Here, \texttt{median} denotes a global median pooling operation applied independently to each of the \( n \) activation maps in \( g(\mathbf{x}) \). 

\noindent
\textbf{Shannon Entropy}. In addition to the median, we compute the \textit{Shannon entropy} for each activation map. For the \( i \)-th channel activation \( g_i(\mathbf{x}) \in \mathbb{R}^{k \times k} \), the entropy is computed as shown in Equation~\ref{eq: entropy}. To do so, we first flatten \( g_i(\mathbf{x}) \) into a vector of length \( k^2 \), and normalize it to define a discrete probability distribution \( p_{ij} \), as described in Equation~\ref{eq: prob}. By collecting the entropy values across all channels, we obtain the final feature representation \( h(\mathbf{x}) \in \mathbb{R}^n \), as defined in Equation~\ref{eq: entropy_pooling} .

\begin{small}
    \begin{equation}
        p_{ij} = \frac{g_i(\mathbf{x})_j}{\sum_{l=1}^{k^2} g_i(\mathbf{x})_l}, \quad j = 1, \dots, k^2
        \label{eq: prob}
    \end{equation}  
\end{small}
\begin{small}
    \begin{equation}
        \texttt{entropy}_i(\mathbf{x}) = -\sum_{j=1}^{k^2} p_{ij} \log p_{ij}
        \label{eq: entropy}
\end{equation}  
\end{small}
\begin{small}
    \begin{equation}
        \begin{split}
             h(\mathbf{x}) &= \texttt{entropy}\left( g(\mathbf{x}) \right) \\
             &= \left[ \texttt{entropy}_1(\mathbf{x}), \dots, \texttt{entropy}_n(\mathbf{x}) \right]^\top
        \end{split}
        \label{eq: entropy_pooling}
    \end{equation}
\end{small}

\noindent
\textbf{Computing the Scaling Factor $\gamma$.} The \(\gamma(\mathbf{x})\) computation using the median follows the same principle described in Section~\ref{sec: method}: a higher median activation is assumed to indicate an ID sample. The Shannon entropy, however, exhibits the opposite behavior. As alluded to in our motivation (Section~\ref{sec: introduction}) and empirically demonstrated (Figure~\ref{fig: feature_plot}), ID samples typically have lower entropy (i.e., less uncertainty) than OOD samples. So, to maintain the convention that an ID sample exhibits higher score than OOD samples, the entropy-based scaling factor must be inverted (i.e, \(\frac{1}{\gamma(\mathbf{x})}\) ) before it is applied to the baseline score.

\noindent
\textbf{Evaluation.} Using Shannon entropy as the information cue for our scaling factor \(\gamma\) yields a notable performance improvement, particularly when combined with strong baselines like ReAct+DICE. This enhancement is especially pronounced for the MobileNet-V2 architecture. As shown in Table~\ref{table: combination_MobileNet_imagenet}, our entropy-based scaling improves upon the vanilla ReAct+DICE baseline by 14.65\%, achieving superior performance among all foundational methods compared. A slight improvement of 5.90\% is also observed for the ResNet-50 backbone. For brevity, the table presents results for these two representative architectures. 

We present the performance of our method, $\approach$, across both ImageNet and CIFAR benchmarks when the scaling factor ($\gamma$) is computed using the median and entropy statistics. For the ImageNet evaluation, detailed results for the ResNet-50 and MobileNet-V2 architectures are shown in Table~\ref{table: combination_ResNet_imagenet} and Table~\ref{table: combination_MobileNet_imagenet}, respectively. For the CIFAR benchmarks, we present detailed results for two architectures. For ResNet-18, the performance on CIFAR-10 and CIFAR-100 is shown in Table~\ref{table: combination_ResNet-18_CIFAR-10} and Table~\ref{table: combination_ResNet-18_CIFAR-100}, respectively. Similarly, for DenseNet-101, the results for CIFAR-10 and CIFAR-100 are in Table~\ref{table: combination_DenseNet-101_CIFAR-10} and Table~\ref{table: combination_DenseNet-101_CIFAR-100}.

\noindent
\textbf{Discussion.} Across all evaluated benchmarks (Tables \ref{table: combination_ResNet_imagenet}–\ref{table: combination_DenseNet-101_CIFAR-100}), the results for the median statistic are conclusive. Across all evaluated benchmarks, using the median to compute $\gamma$ consistently degrades performance. This degradation occurs because the median violates our method's core assumption: its statistical signature fails to separate ID and OOD samples, resulting in a $\gamma$ with high overlap. Figure~\ref{fig: densenet-101_cifar100_svhn_md_scale_density_plot} provides a representative example of this distribution collapse. Given its consistent failure, the median was conclusively rejected as a viable statistic.

    \begin{figure}[ht]
        \centering
        \includegraphics[width=\columnwidth]{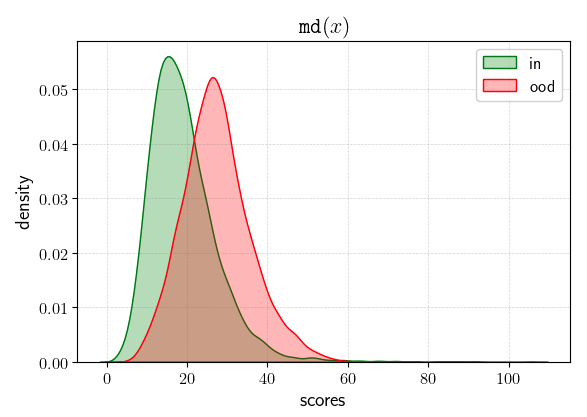}
        \caption{\textit{Distribution of the scaling factor, $\gamma$, computed using the median statistic. The model is a DenseNet-101 trained on CIFAR-100 (ID), evaluated against the SVHN dataset (OOD). The plot reveals that the OOD distribution is shifted to the right of the ID distribution, indicating that OOD samples produce a higher $\gamma$ value than ID samples. This contradicts the core assumption of our method, leading to degraded OOD detection performance.}}
        \label{fig: densenet-101_cifar100_svhn_md_scale_density_plot}
    \end{figure}

The analysis of Shannon entropy is more nuanced and reveals a critical insight. Entropy is not consistently ineffective; rather, it is inconsistent. In specific cases, entropy can be very effective. For example, on the ImageNet benchmark (Table~\ref{table: combination_MobileNet_imagenet}), the entropy-based $\gamma$ improves the \texttt{ReAct+DICE} baseline by 14.65\% on the MobileNet-V2 architecture, achieving the best performance for that specific model. However, this strong performance is not generalizable. On the same dataset but with a ResNet-50 backbone (Table~\ref{table: combination_ResNet_imagenet}), the improvement is minimal (5.90\%) and lags behind our proposed mean/std/max combination. This inconsistency can also be seen in ablation study of using scaling factor standalone as scoring metric in Appendix~\ref{appendix: scaling factor as score}, where we show that $\gamma_{\text{entropy}}$ as a standalone score was dominant on CIFAR but failed to generalize to ImageNet.

This rigor confirms that our chosen statistics (mean, standard deviation, maximum) are the most effective choice, providing a robust and consistently high-performing signal across all models and datasets.

\section{Analysis of $\gamma$ as a Standalone OOD Score}
\label{appendix: scaling factor as score}

    The core hypothesis of our work is that the scaling factor $\gamma(\mathbf{x})$, contains significant discriminative information. While our primary method uses $\gamma$ as a modulator for existing OOD scores (e.g., Energy, MSP, ODIN), an important question is whether $\gamma$ is powerful enough to serve as a standalone OOD scoring function. Furthermore, this analysis allows us to identify which of its component statistics are the most robust and generalizable.

    To investigate this, we conducted a standalone analysis of $\gamma$, comparing it directly to the strong Energy score. We computed $\gamma$ individually from four distinct channel-wise statistics: mean, standard deviation, maximum, and entropy. (We omit median as an initial analysis, detailed in Appendix~\ref{appendix: alternate statistics}, showed insufficient discriminative power). For entropy, we found its reciprocal ($1/\gamma_{\text{entropy}}$) without thresholding was its most potent configuration, and we use that for this analysis.

    \noindent
    \textbf{Analysis on ImageNet.} To test the generalizability of these findings, we repeated the analysis on the large-scale ImageNet benchmark (Table~\ref{table: gamma_score_imagenet}). Here, the trend dramatically reversed. The scores derived from $\gamma_{\text{std}}$, $\gamma_{\text{max}}$, and even $\gamma_{\text{mean}}$ remained robust and generalizable, consistently outperforming the Energy baseline by a significant margin (e.g., 19.32\% for $\gamma_{\text{std}}$ on ResNet-50).

    In sharp contrast, the $\gamma_{\text{entropy}}$ score, which was dominant on CIFAR, failed to generalize. It not only performed worse than the other $\gamma$ statistics but also failed to consistently beat the Energy baseline. For instance, on DenseNet-121, it scored a poor 53.09\% FPR95 compared to Energy's 50.68\% (a ~2.4-point gap), and on ResNet-50, it merely matched the Energy score (56.73\% vs. 57.48\%). This demonstrates that entropy, while powerful on simpler datasets, is not a reliable or generalizable statistic for OOD detection on more complex, large-scale tasks.

    This analysis provides a critical insight and directly justifies our final methodological design (Equation~\ref{eq: gamma_c} and \ref{eq: gamma_design}). Our Catalyst framework is constructed by combining the statistics that proved to be consistently robust across all benchmarks (mean, standard deviation, maximum), while entropy is deliberately excluded due to its clear lack of generalizability.

    \begin{figure*}[ht]
        \centering
        \includegraphics[width=0.95\textwidth]{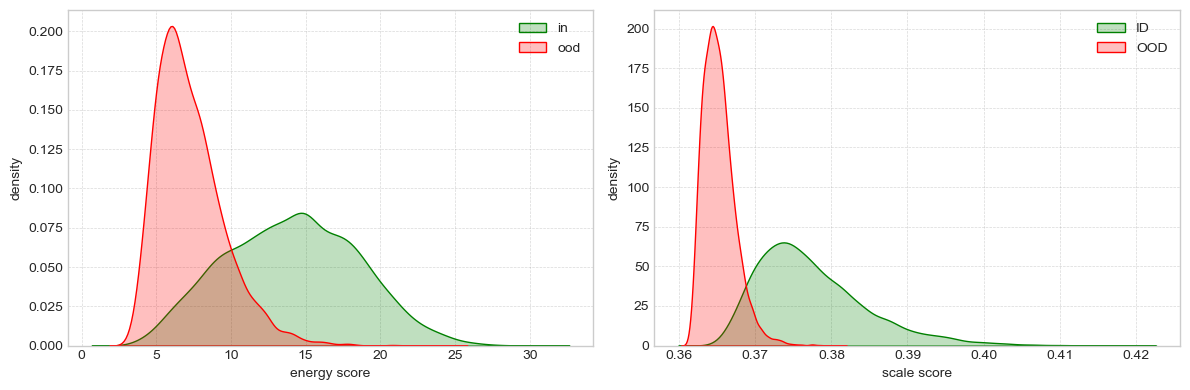}
        \caption{\textit{Superior OOD separation of $\gamma_{\text{entropy}}$ as a standalone score on CIFAR-100. The model is a ResNet-18 trained on CIFAR-100 (ID), evaluated against the Texture dataset (OOD). (Left) Significant distribution overlap between ID and OOD using the baseline Energy score. (Right) Dramatically improved separation using the standalone $\gamma_{\text{entropy}}$ score.This visualization confirms the finding from Table~\ref{table: gamma_score_cifar} that entropy is an exceptionally powerful standalone signal on the CIFAR benchmarks.}}
        \label{fig: ResNet-18_cifar100_textrue_enrgy_scale_density_plot}
    \end{figure*}

    \begin{table*}[ht]
    \centering
    \resizebox{\textwidth}{!}{
    \begin{tabular}{l l  cc cc cc cc cc}
        \toprule
        \multirow{2}{*}{\textbf{Model}} & \multirow{2}{*}{\textbf{Method}} & \multicolumn{2}{c}{SUN} & \multicolumn{2}{c}{Places} & \multicolumn{2}{c}{Texture} & \multicolumn{2}{c}{iNaturalist} & \multicolumn{2}{c}{Average}  \\
        \cmidrule(lr){3-4} \cmidrule(lr){5-6} \cmidrule(lr){7-8} \cmidrule(lr){9-10} \cmidrule(lr){11-12}
        && \textbf{FPR95} $\downarrow$ & \textbf{AUROC} $\uparrow$ & \textbf{FPR95} $\downarrow$ & \textbf{AUROC} $\uparrow$ & \textbf{FPR95} $\downarrow$ & \textbf{AUROC} $\uparrow$ & \textbf{FPR95} $\downarrow$ & \textbf{AUROC} $\uparrow$ & \textbf{FPR95} $\downarrow$ & \textbf{AUROC} $\uparrow$ \\
        \midrule
        \multirow{5}{*}{ResNet-34}  & Energy  & 57.39 & 86.59 & 62.61 & 84.59 & 54.95 & 86.45 & 53.86 & 89.73 & 57.20 & 86.84 \\
                                    & Mean    & 44.83 & 96.01 & 56.88 & 94.29 & 27.99 & 98.88 & 35.96 & 97.69 & 41.42 & 96.72 \\
                                    & Std     & 56.73 & 94.09 & 69.36 & 91.74 & 18.81 & 99.22 & 42.05 & 97.14 & 46.74 & 95.55 \\
                                    & Max     & 54.94 & 94.12 & 68.05 & 91.63 & 19.11 & 99.19 & 42.87 & 97.01 & 46.24 & 95.49 \\
                                    & Entropy & 63.75 & 91.93 & 75.47 & 88.55 & 20.64 & 99.03 & 52.63 & 95.25 & 53.12 & 93.69 \\
    
        \midrule
        \multirow{5}{*}{ResNet-50}  & Energy  & 58.28 & 86.73 & 65.40 & 84.13 & 52.29 & 86.73 & 53.95 & 90.59 & 57.48 & 87.05 \\
                                    & Mean    & 43.71 & 96.52 & 56.03 & 94.69 & 27.23 & 99.00 & 30.67 & 98.22 & 39.41 & 97.11 \\
                                    & Std     & 56.57 & 94.83 & 69.33 & 92.40 & 20.36 & 99.20 & 39.24 & 97.61 & 46.37 & 96.01 \\
                                    & Max     & 56.12 & 94.63 & 69.41 & 92.03 & 20.28 & 99.19 & 39.67 & 97.53 & 46.37 & 95.84 \\
                                    & Entropy & 71.43 & 90.98 & 81.43 & 87.52 & 21.84 & 98.90 & 52.21 & 95.38 & 56.73 & 93.20 \\
    
        \midrule
        \multirow{5}{*}{MobileNet-v2}   & Energy  & 59.36 & 86.24 & 66.27 & 83.21 & 54.54 & 86.58 & 55.31 & 90.34 & 58.87 & 86.59 \\
                                        & Mean    & 44.23 & 96.84 & 60.08 & 94.60 & 21.14 & 99.40 & 45.29 & 96.98 & 42.68 & 96.96 \\
                                        & Std     & 50.97 & 96.02 & 67.22 & 93.28 & 14.54 & 99.60 & 46.03 & 97.19 & 44.69 & 96.52 \\
                                        & Max     & 50.33 & 95.96 & 66.70 & 93.10 & 14.56 & 99.60 & 46.58 & 97.22 & 44.54 & 96.47 \\
                                        & Entropy & 56.30 & 94.59 & 71.04 & 90.89 & 15.43 & 99.52 & 50.40 & 96.28 & 48.29 & 95.32 \\
    
        \midrule
        \multirow{5}{*}{DenseNet-121}   & Energy  & 52.51 & 87.27 & 58.24 & 85.05 & 52.22 & 85.42 & 39.75 & 92.66 & 50.68 & 87.60 \\
                                        & Mean    & 56.90 & 95.02 & 68.16 & 93.11 & 36.99 & 98.58 & 40.51 & 97.42 & 50.64 & 96.03 \\
                                        & Std     & 58.44 & 94.80 & 69.69 & 92.64 & 29.63 & 98.89 & 41.84 & 97.36 & 49.90 & 95.92 \\
                                        & Max     & 58.35 & 94.65 & 70.09 & 92.37 & 29.68 & 98.89 & 43.02 & 97.21 & 50.29 & 95.78 \\
                                        & Entropy & 61.31 & 92.74 & 73.05 & 89.28 & 29.86 & 98.56 & 48.13 & 95.38 & 53.09 & 93.99 \\
    
        \bottomrule
        \end{tabular}
        }
        \caption{\textit{Detailed OOD detection results on ImageNet benchmarks using scaling factor $\gamma$ as a standalone scoring metric. $\boldsymbol{\downarrow}$ indicates lower values are better and $\boldsymbol{\uparrow}$ indicates larger values are better.}}
        \label{table: gamma_score_imagenet}
    \end{table*}

    \noindent
    \textbf{Analysis on CIFAR.} As shown in Table~\ref{table: gamma_score_cifar}, the standalone performance of $\gamma$ on CIFAR is remarkably strong. The scores from $\gamma_{\text{std}}$ and $\gamma_{\text{max}}$ are highly competitive, matching or exceeding the Energy score baseline. The $\gamma_{\text{mean}}$ score is less effective, which aligns with the poor signal separation observed in our motivational analysis (Figure~\ref{fig: feature_plot}). 
    
    Notably, the $\gamma_{\text{entropy}}$ score is effective on these benchmarks. It dramatically outperforms the Energy baseline by 69.03\% (CIFAR-10) and 31.62\% (CIFAR-100) on ResNet-18, and by 33.69\% (CIFAR-10) and 15.29\% (CIFAR-100) on DenseNet-101. Based on this initial finding, entropy would appear to be the most powerful standalone signal. As a representative, Figure~\ref{fig: ResNet-18_cifar100_textrue_enrgy_scale_density_plot} visually demonstrates the superior distribution separation achieved by the standalone $\gamma_{\text{entropy}}$ score compared to the Energy baseline on the CIFAR-100 benchmark.

    \begin{landscape}
        \begin{table*}[ht]
        \centering
        \resizebox{\linewidth}{!}{
        \begin{tabular}{ l l l  cc cc cc cc cc cc cc }
            \toprule
             \multirow{2}{*}{\textbf{Dataset}} & \multirow{2}{*}{\textbf{Model}} & \multirow{2}{*}{\textbf{Method}} & \multicolumn{2}{c}{\textbf{SVHN}} & \multicolumn{2}{c}{\textbf{Place365}} & \multicolumn{2}{c}{\textbf{iSUN}} & \multicolumn{2}{c}{\textbf{Textures}} & \multicolumn{2}{c}{\textbf{LSUN-c}}  & \multicolumn{2}{c}{\textbf{LSUN-r}} & \multicolumn{2}{c}{\textbf{Average}}  \\
            \cmidrule(lr){4-5} \cmidrule(lr){6-7} \cmidrule(lr){8-9} \cmidrule(lr){10-11} \cmidrule(lr){12-13} \cmidrule(lr){14-15} \cmidrule(lr){16-17}
            &&& \textbf{FPR95} $\downarrow$ & \textbf{AUROC} $\uparrow$ & \textbf{FPR95} $\downarrow$ & \textbf{AUROC} $\uparrow$ & \textbf{FPR95} $\downarrow$ & \textbf{AUROC} $\uparrow$ & \textbf{FPR95} $\downarrow$ & \textbf{AUROC} $\uparrow$ & \textbf{FPR95} $\downarrow$ & \textbf{AUROC} $\uparrow$ & \textbf{FPR95} $\downarrow$ & \textbf{AUROC} $\uparrow$ & \textbf{FPR95} $\downarrow$ & \textbf{AUROC} $\uparrow$ \\
            \midrule
            \multirow{10}{*}{CIFAR-10}  & \multirow{5}{*}{ResNet-18} & Energy  & 44.32 & 94.04 & 41.43 & 91.72 & 35.22 & 94.70 & 50.30 & 91.11 &  9.77 & 98.19 & 31.97 & 95.26 & 35.50 & 94.17 \\
                                                                    && Mean    & 29.84 & 85.75 & 97.48 & 38.81 & 84.09 & 57.59 & 58.63 & 79.92 & 30.37 & 91.33 & 87.08 & 53.23 & 64.58 & 67.77 \\
                                                                    && Std     &  5.56 & 98.97 & 56.34 & 88.34 & 17.21 & 97.48 & 12.43 & 98.56 &  0.35 & 99.87 & 20.93 & 96.74 & 18.80 & 96.66 \\
                                                                    && Max     &  5.39 & 99.00 & 54.87 & 88.51 & 14.49 & 97.81 & 11.42 & 98.62 &  0.57 & 99.84 & 17.56 & 97.16 & 17.38 & 96.82 \\
                                                                    && Entropy &  6.14 & 98.81 & 36.31 & 93.29 &  7.28 & 98.83 &  8.09 & 99.07 &  0.73 & 99.81 &  8.65 & 98.55 & 11.20 & 98.06 \\
            \cmidrule(lr){2-17}
            & \multirow{6}{*}{DenseNet-101}  & Energy  & 37.91 & 93.59 & 36.42 & 92.38 &  7.33 & 98.27 & 43.87 & 90.48 &  1.95 & 99.47 &  6.97 & 98.38 & 22.41 & 95.43 \\
                                            && Mean    &  7.66 & 98.49 & 87.02 & 71.43 & 50.85 & 93.13 & 28.39 & 94.51 &  9.86 & 98.39 & 59.84 & 91.56 & 40.60 & 91.25 \\
                                            && Std     & 10.35 & 97.69 & 77.55 & 74.92 & 15.17 & 97.59 & 18.28 & 96.94 &  1.54 & 99.62 & 16.99 & 97.06 & 23.31 & 93.97 \\
                                            && Max     & 10.47 & 97.62 & 77.68 & 74.83 & 16.35 & 97.47 & 17.52 & 97.11 &  2.73 & 99.42 & 18.48 & 96.87 & 23.87 & 93.89 \\
                                            && Entropy &  8.80 & 97.67 & 53.28 & 86.95 &  7.13 & 98.79 &  9.79 & 98.56 &  2.68 & 99.43 &  7.54 & 98.59 & 14.86 & 96.67 \\
            \midrule                              
            \multirow{12}{*}{CIFAR-100}  &\multirow{5}{*}{ResNet-18} & Energy  & 66.64 & 89.53 & 81.39 & 76.83 & 71.46 & 83.02 & 85.18 & 75.68 & 48.01 & 91.63 & 68.57 & 84.53 & 70.21 & 83.54 \\
                                                                    && Mean    & 88.94 & 80.05 & 98.54 & 42.35 & 96.46 & 58.22 & 63.79 & 79.95 & 57.49 & 79.04 & 97.24 & 54.51 & 83.74 & 65.68 \\
                                                                    && Std     & 36.12 & 93.51 & 96.38 & 49.35 & 76.44 & 82.19 & 46.88 & 88.67 & 26.59 & 95.08 & 81.57 & 79.82 & 60.66 & 81.43 \\
                                                                    && Max     & 35.18 & 93.69 & 95.98 & 50.69 & 77.05 & 83.24 & 46.92 & 88.95 & 25.07 & 95.52 & 82.80 & 80.45 & 60.50 & 82.09 \\
                                                                    && Entropy & 22.71 & 95.73 & 93.66 & 62.23 & 56.89 & 90.97 & 36.17 & 93.60 & 16.57 & 97.24 & 62.06 & 89.39 & 48.01 & 88.19 \\
             \cmidrule(lr){2-17}
            & \multirow{6}{*}{DenseNet-101}  & Energy  & 70.99 & 86.66 & 77.28 & 76.94 & 59.39 & 85.68 & 83.49 & 67.47 & 11.45 & 97.89 & 50.90 & 88.57 & 58.92 & 83.87 \\
                                            && Mean    & 31.77 & 94.25 & 95.11 & 55.02 & 78.44 & 81.61 & 40.41 & 92.03 & 18.11 & 97.11 & 87.01 & 77.82 & 58.47 & 82.97 \\
                                            && Std     & 30.65 & 94.73 & 93.83 & 60.93 & 65.01 & 87.40 & 34.13 & 94.21 & 11.03 & 98.26 & 73.04 & 85.11 & 51.28 & 86.77 \\
                                            && Max     & 30.81 & 94.70 & 93.85 & 62.60 & 64.63 & 88.43 & 33.24 & 94.68 & 14.43 & 97.74 & 73.70 & 86.22 & 51.77 & 87.39 \\
                                            && Entropy & 34.67 & 93.78 & 93.64 & 65.11 & 58.01 & 89.88 & 30.53 & 95.32 & 16.15 & 97.34 & 66.51 & 88.03 & 49.91 & 88.24 \\
            \bottomrule
            \end{tabular}
            }
            \caption{\textit{Detailed OOD detection results on CIFAR benchmarks using scaling factor $\gamma$ as a standalone scoring metric.$\boldsymbol{\downarrow}$ indicates lower values are better and $\boldsymbol{\uparrow}$ indicates larger values are better.}}
            \label{table: gamma_score_cifar}
        \end{table*}
    \end{landscape}

\section{Societal Impact}
\label{appendix: social impact}

    The reliable detection of out-of-distribution (OOD) inputs is a fundamental requirement for safe and trustworthy deployment of machine learning systems. This capability is critical in high-stakes domains such as autonomous transportation, where an unexpected object on the road must be identified as anomalous, and in medical diagnostics, where a model must recognize that a scan presents features of an unseen disease. By improving the separation between in-distribution (ID) and OOD data, our work directly contributes to building more robust and dependable AI. The primary benefit of our approach, $\approach$, is its potential to reduce critical failure rates, which is crucial for ensuring user safety and earning public trust in automated systems.

    The broader impact of this research lies in enhancing the safety and reliability of AI. Our work adheres to ethical research standards, does not involve human subjects, and uses publicly available datasets. While any powerful technology can have unforeseen applications, our work is fundamentally aimed at mitigating the harm that arises from brittle AI models that fail silently or unpredictably when faced with novel inputs. By releasing our code to the public, we hope to foster further research, encourage reproducibility, and accelerate the development of more robust AI systems that can be deployed responsibly in society.

\section{Synergy with Existing Methods}
\label{appendix: combined method}

$\approach$ is designed to be fully compatible with existing post-hoc OOD detection techniques, enabling seamless integration with widely used methods such as \texttt{MSP}~\cite{msp}, \texttt{ODIN}~\cite{odin}, \texttt{Energy}~\cite{energy}, \texttt{ReAct}~\cite{ReAct}, \texttt{DICE}~\cite{DICE}, \texttt{ASH}~\cite{ASH} and \texttt{KNN}~\cite{knn}. Rather than replacing these techniques, $\approach$ acts as a complementary module. It enhances their ability to separate in-distribution and out-of-distribution samples through an elastic scaling mechanism, introducing an additional degree of freedom that works in tandem with them. In our evaluation, we omit ODIN~\cite{odin} from our analysis due to its high computational cost and limited performance on large-scale datasets like ImageNet. The method's expense stems from requiring an FGSM-based perturbation for every input sample.

To demonstrate the effectiveness of this synergy on the ImageNet benchmark, we present detailed experimental results in Table~\ref{table: combination_ResNet_imagenet} and Table~\ref{table: combination_MobileNet_imagenet}, showing the performance of each baseline method when combined with $\approach$. The results consistently indicate performance improvements, validating the benefit of integrating $\approach$ with established methods. For brevity, the table presents results for these two representative architectures, ResNet-50 and MobileNet-V2. 

\begin{table*}[ht]
\centering

\resizebox{0.97\textwidth}{!}{
\begin{tabular}{ c l  cc cc cc cc cc }
    \toprule
     \multirow{2}{*}{\textbf{Model}} & \multirow{2}{*}{\textbf{Combined Method}} & \multicolumn{2}{c}{\textbf{SUN}} & \multicolumn{2}{c}{\textbf{Place365}} & \multicolumn{2}{c}{\textbf{Textures}} & \multicolumn{2}{c}{\textbf{iNaturalist}} & \multicolumn{2}{c}{\textbf{Average}}  \\
    \cmidrule(lr){3-4} \cmidrule(lr){5-6} \cmidrule(lr){7-8} \cmidrule(lr){9-10} \cmidrule(lr){11-12}
    && \textbf{FPR95} $\downarrow$ & \textbf{AUROC} $\uparrow$ & \textbf{FPR95} $\downarrow$ & \textbf{AUROC} $\uparrow$ & \textbf{FPR95} $\downarrow$ & \textbf{AUROC} $\uparrow$ & \textbf{FPR95} $\downarrow$ & \textbf{AUROC} $\uparrow$ & \textbf{FPR95} $\downarrow$ & \textbf{AUROC} $\uparrow$ \\
    \midrule
    \multirow{42}{*}{ResNet-50} & \texttt{MSP}                   & 68.58 & 81.75 & 71.57 & 80.63 & 66.13 & 80.46 & 52.77 & 88.42 & 64.76 & 82.82 \\
                                & + $\approach(\mu)$             & 56.58 & 87.33 & 62.93 & 85.10 & 52.48 & 87.81 & 39.20 & 92.41 & 52.80 & 88.16 \\
                                & + $\approach(\sigma)$          & 59.13 & 85.63 & 65.72 & 82.89 & 42.66 & 91.76 & 39.26 & 92.31 & 51.69 & 88.15 \\
                                & + $\approach(m)$               & 58.92 & 85.51 & 65.71 & 82.53 & 42.82 & 91.67 & 39.02 & 92.29 & 51.62 & 88.00 \\
                                & + $\approach(\texttt{md})$     & 53.93 & 88.34 & 59.01 & 86.99 & 63.49 & 79.42 & 42.50 & 91.27 & 54.73 & 86.51 \\
                                & + $\approach(e)$               & 67.59 & 82.41 & 70.90 & 80.79 & 63.87 & 83.21 & 51.05 & 89.05 & 63.35 & 83.86 \\

    \cmidrule(lr){2-12}
                                & \texttt{Energy}                & 58.28 & 86.73 & 65.40 & 84.13 & 52.29 & 86.73 & 53.95 & 90.59 & 57.48 & 87.05 \\
                                & + $\approach(\mu)$             & 30.79 & 92.67 & 42.59 & 89.78 & 22.29 & 94.01 & 18.02 & 96.46 & 28.42 & 93.23 \\
                                & + $\approach(\sigma)$          & 35.73 & 91.47 & 48.35 & 88.04 & 15.85 & 95.94 & 19.05 & 96.21 & 29.75 & 92.92 \\
                                & + $\approach(m)$               & 35.79 & 91.40 & 48.68 & 87.82 & 16.08 & 95.88 & 19.00 & 96.18 & 29.89 & 92.82 \\
                                & + $\approach(\texttt{md})$     & 30.23 & 93.24 & 38.47 & 91.31 & 47.70 & 87.77 & 25.46 & 95.29 & 35.46 & 91.90 \\
                                & + $\approach(e)$               & 52.40 & 87.75 & 62.06 & 84.76 & 41.35 & 89.27 & 44.22 & 92.20 & 50.01 & 88.50 \\
    \cmidrule(lr){2-12}
                                & \texttt{ReAct}                 & 23.68 & 94.44 & 33.33 & 91.96 & 46.33 & 90.30 & 19.73 & 96.37 & 30.77 & 93.27 \\
                                & + $\approach(\mu)$             & 18.46 & 95.82 & 28.98 & 93.31 & 12.11 & 97.38 & 8.54 & 98.19 & 17.02 & 96.18 \\
                                & + $\approach(\sigma)$          & 19.13 & 95.61 & 29.58 & 93.04 & 12.04 & 97.38 & 9.10 & 98.06 & 17.46 & 96.02 \\
                                & + $\approach(m)$               & 19.02 & 95.52 & 29.77 & 92.92 & 12.06 & 97.31 & 9.71 & 97.97 & 17.64 & 95.93 \\
                                & + $\approach(\texttt{md})$     & 24.41 & 95.55 & 31.75 & 93.63 & 57.62 & 87.82 & 22.06 & 96.18 & 33.96 & 93.30 \\
                                & + $\approach(e)$               & 22.64 & 94.55 & 34.09 & 91.58 & 22.27 & 94.91 & 13.94 & 97.25 & 23.23 & 94.57 \\
    \cmidrule(lr){2-12}
                                & \texttt{DICE}                  & 36.11 & 91.01 & 47.62 & 87.76 & 32.38 & 90.48 & 26.48 & 94.53 & 35.65 & 90.94 \\
                                & + $\approach(\mu)$             & 31.28 & 92.16 & 43.09 & 89.11 & 19.91 & 94.46 & 16.59 & 96.58 & 27.72 & 93.08 \\
                                & + $\approach(\sigma)$          & 32.86 & 91.85 & 45.67 & 88.45 & 20.05 & 94.45 & 20.18 & 95.82 & 29.69 & 92.65 \\
                                & + $\approach(m)$               & 33.56 & 91.77 & 47.04 & 88.18 & 18.62 & 95.06 & 20.63 & 95.82 & 29.96 & 92.71 \\
                                & + $\approach(\texttt{md})$     & 30.26 & 93.53 & 38.80 & 91.09 & 46.90 & 88.00 & 25.14 & 95.17 & 35.27 & 91.95 \\
                                & + $\approach(e)$               & 36.11 & 90.98 & 47.72 & 87.72 & 32.38 & 90.48 & 26.60 & 94.50 & 35.70 & 90.92 \\
    \cmidrule(lr){2-12}
                                & \texttt{ReAct+DICE}            & 24.05 & 94.31 & 34.28 & 91.71 & 28.40 & 93.33 & 14.90 & 97.06 & 25.41 & 94.10 \\
                                & + $\approach(\mu)$             & 23.47 & 94.82 & 34.08 & 92.18 & 14.45 & 96.91 & 10.86 & 97.86 & 20.72 & 95.44 \\
                                & + $\approach(\sigma)$          & 23.76 & 94.65 & 34.58 & 92.06 & 15.07 & 96.70 & 11.40 & 97.75 & 21.20 & 95.29 \\
                                & + $\approach(m)$               & 25.33 & 94.46 & 36.92 & 91.67 & 13.81 & 97.02 & 12.35 & 97.61 & 22.10 & 95.19 \\
                                & + $\approach(\texttt{md})$     & 24.25 & 95.25 & 32.08 & 93.19 & 46.68 & 90.50 & 18.80 & 96.63 & 30.45 & 93.89 \\
                                & + $\approach(e)$               & 22.07 & 94.78 & 31.61 & 92.32 & 28.00 & 93.54 & 13.98 & 97.25 & 23.91 & 94.47 \\
    \cmidrule(lr){2-12}
                                & \texttt{ASH}                   & 28.01 & 94.02 & 39.84 & 90.98 & 11.95 & 97.60 & 11.52 & 97.87 & 22.83 & 95.12 \\
                                & + $\approach(\mu)$             & 28.97 & 93.75 & 41.04 & 90.53 & 11.47 & 97.79 & 12.08 & 97.74 & 23.39 & 94.95 \\
                                & + $\approach(\sigma)$          & 29.76 & 93.73 & 41.75 & 90.77 & 11.56 & 97.57 & 12.24 & 97.75 & 23.83 & 94.95 \\
                                & + $\approach(m)$               & 28.23 & 93.97 & 40.20 & 90.90 & 11.49 & 97.73 & 11.60 & 97.85 & 22.88 & 95.11 \\
                                & + $\approach(\texttt{md})$     & 26.32 & 94.43 & 36.36 & 91.67 & 19.52 & 96.17 & 13.95 & 97.35 & 24.04 & 94.91 \\
                                & + $\approach(e)$               & 27.96 & 94.01 & 39.81 & 90.97 & 11.93 & 97.60 & 11.50 & 97.87 & 22.80 & 95.11 \\
    \cmidrule(lr){2-12}
                                & \texttt{SCALE}                 & 25.78 & 94.54 & 36.86 & 91.96 & 14.56 & 96.75 & 10.37 & 98.02 & 21.89 & 95.32 \\
                                & + $\approach(\mu)$             & 25.38 & 94.57 & 36.55 & 91.83 & 11.90 & 97.45 & 10.11 & 98.06 & 20.98 & 95.48 \\
                                & + $\approach(\sigma)$          & 25.58 & 94.51 & 36.99 & 91.77 & 11.83 & 97.48 & 10.31 & 98.03 & 21.18 & 95.45 \\
                                & + $\approach(m)$               & 25.60 & 94.51 & 37.09 & 91.76 & 11.79 & 97.48 & 10.32 & 98.02 & 21.20 & 95.44 \\
                                & + $\approach(\texttt{md})$     & 23.67 & 95.12 & 33.19 & 92.83 & 23.37 & 95.15 & 12.91 & 97.56 & 23.28 & 95.17 \\
                                & + $\approach(e)$               & 25.77 & 94.47 & 37.04 & 91.81 & 14.01 & 96.84 & 10.34 & 98.02 & 21.79 & 95.29  \\
    \bottomrule
    \end{tabular}
    }
    \caption{\textit{Detailed results of post-hoc methods combined with $\approach$ on four OOD benchmarks: SUN, Places365, Textures, and iNaturalist using ResNet-50 trained on ImageNet-1K. $\boldsymbol{\uparrow}$ indicates higher is better; $\boldsymbol{\downarrow}$ indicates lower is better. The symbols denote the statistic used: $\mu$ (mean), $\sigma$ (std. deviation), $m$ (maximum), $\texttt{md}$ (median), and $e$ (Shannon entropy).}}
    \label{table: combination_ResNet_imagenet}
\end{table*}

\begin{table*}[ht]
\centering

\resizebox{0.97\textwidth}{!}{
\begin{tabular}{ c l  cc cc cc cc cc }
    \toprule
     \multirow{2}{*}{\textbf{Model}} & \multirow{2}{*}{\textbf{Combined Method}} & \multicolumn{2}{c}{\textbf{SUN}} & \multicolumn{2}{c}{\textbf{Place365}} & \multicolumn{2}{c}{\textbf{Textures}} & \multicolumn{2}{c}{\textbf{iNaturalist}} & \multicolumn{2}{c}{\textbf{Average}}  \\
    \cmidrule(lr){3-4} \cmidrule(lr){5-6} \cmidrule(lr){7-8} \cmidrule(lr){9-10} \cmidrule(lr){11-12}
    && \textbf{FPR95} $\downarrow$ & \textbf{AUROC} $\uparrow$ & \textbf{FPR95} $\downarrow$ & \textbf{AUROC} $\uparrow$ & \textbf{FPR95} $\downarrow$ & \textbf{AUROC} $\uparrow$ & \textbf{FPR95} $\downarrow$ & \textbf{AUROC} $\uparrow$ & \textbf{FPR95} $\downarrow$ & \textbf{AUROC} $\uparrow$ \\
    \midrule
    \multirow{42}{*}{MobileNet-V2} & \texttt{MSP}                & 74.20 & 78.88 & 76.89 & 78.14 & 70.99 & 78.95 & 59.86 & 86.72 & 70.49 & 80.67 \\
                                & + $\approach(\mu)$             & 63.48 & 85.01 & 69.78 & 82.57 & 57.46 & 86.55 & 48.69 & 90.06 & 59.85 & 86.05 \\
                                & + $\approach(\sigma)$          & 63.13 & 84.81 & 69.71 & 81.73 & 43.67 & 91.44 & 45.33 & 90.88 & 55.46 & 87.22 \\
                                & + $\approach(m)$               & 63.19 & 84.78 & 69.63 & 81.84 & 46.12 & 90.69 & 46.05 & 90.73 & 56.25 & 87.01 \\
                                & + $\approach(\texttt{md})$     & 66.47 & 83.52 & 70.73 & 82.29 & 73.14 & 76.11 & 57.43 & 87.31 & 66.94 & 82.31 \\
                                & + $\approach(e)$               & 72.66 & 80.27 & 75.86 & 78.90 & 67.94 & 81.88 & 57.41 & 87.59 & 68.47 & 82.16 \\

    \cmidrule(lr){2-12}
                                & \texttt{Energy}                & 59.36 & 86.24 & 66.27 & 83.21 & 54.54 & 86.58 & 55.31 & 90.34 & 58.87 & 86.59 \\
                                & + $\approach(\mu)$             & 37.74 & 91.43 & 52.21 & 87.33 & 23.42 & 94.17 & 33.47 & 93.84 & 36.71 & 91.69 \\
                                & + $\approach(\sigma)$          & 38.20 & 91.26 & 53.04 & 86.84 & 14.02 & 96.37 & 29.25 & 94.63 & 33.63 & 92.27 \\
                                & + $\approach(m)$               & 37.41 & 91.37 & 52.24 & 86.89 & 14.18 & 96.35 & 28.78 & 94.70 & 33.15 & 92.33 \\
                                & + $\approach(\texttt{md})$     & 52.89 & 88.64 & 62.06 & 85.82 & 67.66 & 82.51 & 62.71 & 87.50 & 61.33 & 86.12 \\
                                & + $\approach(e)$               & 52.16 & 87.95 & 61.69 & 84.32 & 41.17 & 89.95 & 45.70 & 92.08 & 50.18 & 88.58 \\
    \cmidrule(lr){2-12}
                                & \texttt{ReAct}                 & 52.46 & 87.26 & 59.89 & 84.07 & 40.25 & 90.96 & 43.05 & 92.72 & 48.91 & 88.75 \\
                                & + $\approach(\mu)$             & 32.82 & 92.93 & 48.62 & 88.59 & 13.60 & 96.83 & 28.19 & 94.89 & 30.81 & 93.31 \\
                                & + $\approach(\sigma)$          & 37.53 & 91.22 & 51.32 & 87.19 & 10.18 & 97.31 & 27.21 & 95.12 & 31.56 & 92.71 \\
                                & + $\approach(m)$               & 34.77 & 92.26 & 49.77 & 88.06 &  8.69 & 97.76 & 24.08 & 95.66 & 29.33 & 93.43 \\
                                & + $\approach(\texttt{md})$     & 50.96 & 89.62 & 59.73 & 86.80 & 71.45 & 82.57 & 63.14 & 86.49 & 61.32 & 86.37 \\
                                & + $\approach(e)$               & 36.32 & 91.14 & 50.91 & 86.16 & 13.71 & 96.34 & 23.20 & 95.70 & 31.03 & 92.33 \\
    \cmidrule(lr){2-12}
                                & \texttt{DICE}                  & 37.84 & 90.81 & 52.35 & 86.17 & 32.57 & 91.46 & 41.53 & 91.30 & 41.07 & 89.94 \\
                                & + $\approach(\mu)$             & 36.22 & 91.56 & 51.48 & 87.17 & 16.45 & 95.78 & 34.79 & 93.38 & 34.74 & 91.97 \\
                                & + $\approach(\sigma)$          & 36.31 & 91.29 & 51.20 & 87.03 & 17.15 & 95.40 & 35.07 & 93.25 & 34.93 & 91.74 \\
                                & + $\approach(m)$               & 34.90 & 91.80 & 50.45 & 87.32 & 14.86 & 96.22 & 32.41 & 93.85 & 33.15 & 92.30 \\
                                & + $\approach(\texttt{md})$     & 47.96 & 89.30 & 58.91 & 85.37 & 61.90 & 83.26 & 61.27 & 84.74 & 57.51 & 85.67 \\
                                & + $\approach(e)$               & 37.96 & 90.76 & 52.11 & 86.28 & 28.88 & 92.59 & 36.19 & 93.07 & 38.79 & 90.68 \\

    \cmidrule(lr){2-12}
                                & \texttt{ReAct+DICE}            & 30.60 & 92.98 & 45.93 & 88.29 & 16.03 & 96.33 & 31.68 & 93.76 & 31.06 & 92.84 \\
                                & + $\approach(\mu)$             & 33.90 & 92.65 & 49.33 & 88.22 & 9.52  & 97.87 & 31.04 & 94.44 & 30.95 & 93.30 \\
                                & + $\approach(\sigma)$          & 32.67 & 92.57 & 47.68 & 88.29 & 9.88  & 97.61 & 29.20 & 94.66 & 29.86 & 93.28 \\
                                & + $\approach(m)$               & 32.60 & 92.58 & 47.77 & 88.30 & 9.79  & 97.63 & 29.29 & 94.66 & 29.86 & 93.29 \\
                                & + $\approach(\texttt{md})$     & 48.95 & 89.51 & 59.14 & 85.64 & 65.62 & 83.49 & 63.50 & 83.56 & 59.30 & 85.55 \\
                                & + $\approach(e)$               & 26.33 & 93.86 & 40.71 & 89.65 & 13.87 & 96.60 & 25.12 & 95.22 & 26.51 & 93.83 \\
    \cmidrule(lr){2-12}
                                & \texttt{ASH}                   & 43.63 & 90.02 & 58.85 & 84.73 & 13.12 & 97.10 & 39.13 & 91.94 & 38.68 & 90.95 \\
                                & + $\approach(\mu)$             & 40.05 & 90.86 & 55.47 & 85.83 & 14.49 & 96.77 & 37.05 & 92.59 & 36.76 & 91.51 \\
                                & + $\approach(\sigma)$          & 41.76 & 90.45 & 57.32 & 85.12 & 10.92 & 97.64 & 34.17 & 93.32 & 36.04 & 91.63 \\
                                & + $\approach(m)$               & 42.01 & 90.41 & 57.41 & 85.02 & 11.12 & 97.62 & 36.11 & 92.94 & 36.66 & 91.50 \\
                                & + $\approach(\texttt{md})$     & 43.12 & 89.49 & 57.70 & 83.57 & 20.02 & 95.86 & 45.78 & 88.40 & 41.65 & 89.33 \\
                                & + $\approach(e)$               & 43.33 & 89.97 & 58.82 & 84.47 & 12.71 & 97.24 & 38.66 & 91.97 & 38.38 & 90.91 \\
    \cmidrule(lr){2-12}
                                & \texttt{SCALE}                 & 38.74 & 91.64 & 53.49 & 87.34 & 14.79 & 96.65 & 30.09 & 94.46 & 34.28 & 92.52 \\
                                & + $\approach(\mu)$             & 37.12 & 91.82 & 52.31 & 87.00 & 13.97 & 97.01 & 31.85 & 93.82 & 33.81 & 92.41 \\
                                & + $\approach(\sigma)$          & 39.23 & 91.57 & 54.34 & 87.00 & 11.81 & 97.43 & 31.22 & 94.20 & 34.15 & 92.55 \\
                                & + $\approach(m)$               & 38.70 & 91.66 & 53.53 & 86.95 & 11.33 & 97.56 & 30.97 & 94.27 & 33.63 & 92.61 \\
                                & + $\approach(\texttt{md})$     & 40.23 & 91.51 & 53.09 & 87.71 & 28.53 & 93.86 & 40.72 & 91.62 & 40.64 & 91.18 \\
                                & + $\approach(e)$               & 38.73 & 91.62 & 53.46 & 87.20 & 14.47 & 96.74 & 29.90 & 94.46 & 34.14 & 92.51 \\
    \bottomrule
    \end{tabular}
    }
    \caption{\textit{Detailed results of post-hoc methods combined with $\approach$ on four OOD benchmarks: SUN, Places365, Textures, and iNaturalist using MobileNet-V2 trained on ImageNet-1K. $\boldsymbol{\uparrow}$ indicates higher is better; $\boldsymbol{\downarrow}$ indicates lower is better. The symbols denote the statistic used: $\mu$ (mean), $\sigma$ (std. deviation), $m$ (maximum), $\texttt{md}$ (median), and $e$ (Shannon entropy).}}
    \label{table: combination_MobileNet_imagenet}
\end{table*}

To demonstrate the effectiveness of this synergy on the CIFAR benchmarks, we present detailed experimental results in Table~\ref{table: combination_ResNet-18_CIFAR-10}, ~\ref{table: combination_DenseNet-101_CIFAR-10} (CIFAR-10) and Table~\ref{table: combination_ResNet-18_CIFAR-100}, ~\ref{table: combination_DenseNet-101_CIFAR-100} (CIFAR-100), showing the performance of each baseline method when combined with $\approach$. The results demonstrate consistent performance improvements, validating the benefit of integrating $\approach$ with established baselines.

    \begin{landscape}
        \begin{table*}[ht]
        \centering
        \resizebox{\linewidth}{!}{
        \begin{tabular}{ l l  cc cc cc cc cc cc cc }
            \toprule
             \multirow{2}{*}{\textbf{Model}} & \multirow{2}{*}{\textbf{Combined Method}} & \multicolumn{2}{c}{\textbf{SVHN}} & \multicolumn{2}{c}{\textbf{Place365}} & \multicolumn{2}{c}{\textbf{iSUN}} & \multicolumn{2}{c}{\textbf{Textures}} & \multicolumn{2}{c}{\textbf{LSUN-c}}  & \multicolumn{2}{c}{\textbf{LSUN-r}} & \multicolumn{2}{c}{\textbf{Average}}  \\
            \cmidrule(lr){3-4} \cmidrule(lr){5-6} \cmidrule(lr){7-8} \cmidrule(lr){9-10} \cmidrule(lr){11-12} \cmidrule(lr){13-14} \cmidrule(lr){15-16}
            && \textbf{FPR95} $\downarrow$ & \textbf{AUROC} $\uparrow$ & \textbf{FPR95} $\downarrow$ & \textbf{AUROC} $\uparrow$ & \textbf{FPR95} $\downarrow$ & \textbf{AUROC} $\uparrow$ & \textbf{FPR95} $\downarrow$ & \textbf{AUROC} $\uparrow$ & \textbf{FPR95} $\downarrow$ & \textbf{AUROC} $\uparrow$ & \textbf{FPR95} $\downarrow$ & \textbf{AUROC} $\uparrow$ & \textbf{FPR95} $\downarrow$ & \textbf{AUROC} $\uparrow$ \\
            \midrule
            \multirow{42}{*}{ResNet-18} & MSP	                      & 60.39 & 92.40 & 63.69 & 88.37 & 56.74 & 91.32 & 62.66 & 90.10 & 51.87 & 93.64 & 54.63 & 91.87 & 58.33 & 91.28 \\
                                        & + $\approach(\mu)$          & 28.18 & 95.48 & 62.71 & 85.40 & 48.66 & 93.02 & 49.93 & 92.15 & 14.77 & 97.72 & 46.89 & 92.97 & 41.86 & 92.79 \\
                                        & + $\approach(\sigma)$	      & 11.66 & 97.82 & 54.35 & 90.08 & 30.90 & 95.83 & 25.07 & 96.22 & 4.01  & 98.93 & 31.90 & 95.67 & 26.32 & 95.76 \\
                                        & + $\approach(m)$	          & 10.60 & 97.97 & 51.94 & 90.72 & 25.71 & 96.30 & 22.29 & 96.56 & 3.61  & 99.00 & 27.37 & 96.14 & 23.59 & 96.11 \\
                                        & + $\approach(\texttt{md})$  & 97.77 & 44.48 & 86.80 & 56.39 & 87.60 & 62.30 & 95.18 & 42.70 & 90.74 & 51.01 & 85.24 & 65.28 & 90.55 & 53.70 \\    	      
                                        & + $\approach(e)$            & 58.23 & 93.11 & 62.81 & 89.84 & 55.29 & 92.70 & 61.06 & 91.96 & 48.72 & 94.37 & 53.22 & 93.00 & 56.56 & 92.50 \\ 
            \cmidrule(lr){2-16}
                                        & Energy                      & 44.32 & 94.04 & 41.43 & 91.72 & 35.22 & 94.70 & 50.30 & 91.11 &  9.77 & 98.19 & 31.97 & 95.26 & 35.50 & 94.17 \\
                                        & + $\approach(\mu)$          & 15.73 & 97.32 & 43.65 & 91.25 & 26.26 & 96.08 & 35.98 & 94.25 &  3.20 & 99.26 & 24.26 & 96.30 & 24.85 & 95.74 \\
                                        & + $\approach(\sigma)$       & 10.33 & 98.13 & 37.74 & 92.68 & 16.90 & 97.32 & 24.31 & 96.16 &  1.38 & 99.63 & 15.64 & 97.41 & 17.72 & 96.89 \\
                                        & + $\approach(m)$	          &  9.93 & 98.24 & 36.59 & 92.97 & 14.89 & 97.61 & 23.01 & 96.43 &  1.31 & 99.65 & 13.80 & 97.68 & 16.59 & 97.10 \\
                                        & + $\approach(\texttt{md})$  & 98.21 & 73.52 & 79.74 & 78.51 & 80.55 & 83.81 & 94.36 & 68.85 & 82.37 & 84.90 & 77.77 & 85.43 & 85.50 & 79.17 \\    	      
                                        & + $\approach(e)$            & 41.92 & 94.36 & 41.20 & 91.90 & 34.44 & 94.94 & 49.01 & 91.58 &  9.16 & 98.30 & 31.11 & 95.46 & 34.47 & 94.42\\ 
            \cmidrule(lr){2-16}
                                        & ReAct	                      & 42.31 & 94.12 & 40.70 & 92.25 & 23.07 & 96.37 & 40.44 & 93.69 & 12.27 & 97.90 & 19.78 & 96.80 & 29.76 & 95.19 \\
                                        & + $\approach(\mu)$          & 14.37 & 97.48 & 43.18 & 91.46 & 16.71 & 97.22 & 25.04 & 95.82 &  4.26 & 99.13 & 15.70 & 97.36 & 19.88 & 96.41 \\
                                        & + $\approach(\sigma)$       &  9.38 & 98.29 & 36.51 & 93.10 & 10.82 & 98.10 & 16.86 & 97.27 &  1.57 & 99.61 & 10.38 & 98.14 & 14.25 & 97.42 \\
                                        & + $\approach(m)$	          &  8.86 & 98.39 & 35.04 & 93.38 &  9.08 & 98.32 & 15.64 & 97.48 &  1.52 & 99.63 &  9.00 & 98.35 & 13.19 & 97.59 \\
                                        & + $\approach(\texttt{md})$  & 98.91 & 65.96 & 84.00 & 74.67 & 84.16 & 83.07 & 95.57 & 65.25 & 90.41 & 76.61 & 80.67 & 84.90 & 88.95 & 75.08 \\    	      
                                        & + $\approach(e)$            & 40.32 & 94.25 & 39.76 & 92.54 & 22.24 & 96.49 & 37.87 & 93.99 & 13.64 & 97.70 & 18.89 & 96.92 & 28.79 & 95.31 \\ 
            \cmidrule(lr){2-16}
                                        & DICE	                      & 17.60 & 97.09 & 46.14 & 90.66 & 39.08 & 94.32 & 44.65 & 91.80 & 1.90  & 99.57 & 36.52 & 94.70 & 30.98 & 94.69 \\
                                        & + $\approach(\mu)$          & 8.39  & 98.56 & 58.07 & 88.42 & 29.21 & 95.58 & 30.30 & 95.00 & 0.61  & 99.81 & 30.51 & 95.48 & 26.18 & 95.47 \\
                                        & + $\approach(\sigma)$	      & 6.39  & 98.88 & 44.70 & 91.48 & 15.07 & 97.40 & 19.04 & 96.86 & 0.41  & 99.89 & 16.09 & 97.28 & 16.95 & 96.96 \\
                                        & + $\approach(m)$	          & 5.98  & 98.96 & 42.31 & 92.01 & 12.18 & 97.82 & 17.02 & 97.17 & 0.36  & 99.90 & 12.91 & 97.69 & 15.13 & 97.26 \\
                                        & + $\approach(\texttt{md})$  & 99.11 & 68.64 & 91.50 & 68.78 & 94.43 & 74.02 & 98.37 & 58.86 & 86.39 & 80.37 & 93.52 & 76.03 & 93.89 & 71.12 \\    	      
                                        & + $\approach(e)$            & 15.85 & 97.30 & 44.64 & 90.97 & 36.25 & 94.74 & 41.95 & 92.43 & 1.60  & 99.61 & 33.86 & 95.08 & 29.02 & 95.02 \\ 
            \cmidrule(lr){2-16}
                                        & ReAct+DICE                  & 11.05 & 98.07 & 47.53 & 91.14 & 17.19 & 97.04 & 24.33 & 95.91 &  1.56 & 99.66 & 16.24 & 97.19 & 19.65 & 96.50 \\
                                        & + $\approach(\mu)$          & 7.81  & 98.50 & 64.61 & 86.47 & 22.48 & 96.21 & 21.29 & 96.29 &  0.80 & 99.72 & 23.08 & 96.04 & 23.35 & 95.54 \\
                                        & + $\approach(\sigma)$	      & 5.41  & 98.98 & 47.06 & 91.11 & 11.54 & 97.85 & 12.55 & 97.83 &  0.38 & 99.88 & 12.91 & 97.64 & 14.98 & 97.21 \\
                                        & + $\approach(m)$	          & 5.13  & 99.03 & 45.17 & 91.44 & 9.74  & 98.14 & 11.33 & 98.00 &  0.39 & 99.88 & 10.58 & 97.94 & 13.72 & 97.41 \\
                                        & + $\approach(\texttt{md})$  & 99.48 & 56.14 & 93.84 & 58.15 & 96.45 & 66.78 & 98.76 & 49.11 & 94.49 & 67.60 & 95.53 & 69.35 & 96.42 & 61.19 \\    	      
                                        & + $\approach(e)$            & 10.23 & 98.21 & 46.14 & 91.57 & 15.65 & 97.29 & 22.27 & 96.32 &  1.30 & 99.69 & 15.00 & 97.41 & 18.43 & 96.75 \\ 
            \cmidrule(lr){2-16}
                                        & ASH	                      &  6.24 & 98.80 & 53.83 & 88.05 & 21.61 & 96.44 & 21.81 & 96.41 &  1.94 & 99.52 & 20.31 & 96.49 & 20.96 & 95.95 \\
                                        & + $\approach(\mu)$          &  5.68 & 98.90 & 62.59 & 83.99 & 24.06 & 95.79 & 20.51 & 96.34 &  2.09 & 99.53 & 23.79 & 95.66 & 23.12 & 95.03 \\
                                        & + $\approach(\sigma)$	      &  4.13 & 99.19 & 49.53 & 89.46 & 13.09 & 97.61 & 13.01 & 97.73 &  0.62 & 99.81 & 13.41 & 97.50 & 15.63 & 96.88 \\
                                        & + $\approach(m)$	          &  3.85 & 99.24 & 47.16 & 90.05 & 11.05 & 97.89 & 11.56 & 97.91 &  0.53 & 99.82 & 11.46 & 97.78 & 14.27 & 97.11 \\
                                        & + $\approach(\texttt{md})$  & 94.21 & 79.42 & 83.91 & 72.01 & 83.37 & 79.32 & 92.84 & 67.47 & 40.13 & 92.93 & 81.73 & 80.74 & 79.36 & 78.65 \\    	      
                                        & + $\approach(e)$            & 5.86  & 98.84 & 52.81 & 88.49 & 20.22 & 96.63 & 20.46 & 96.60 &  1.77 & 99.55 & 19.18 & 96.67 & 20.05 & 96.13 \\  
            \cmidrule(lr){2-16}
                                        & \texttt{SCALE}              &  7.73 & 98.54 & 50.51 & 89.81 & 21.43 & 96.62 & 22.29 & 96.27 &  4.18 & 99.18 & 20.17 & 96.75 & 21.05 & 96.19 \\
                                        & + $\approach(\mu)$          &  6.80 & 98.69 & 58.87 & 86.60 & 22.20 & 96.17 & 21.03 & 96.22 &  3.28 & 99.30 & 21.93 & 96.09 & 22.35 & 95.51 \\   
                                        & + $\approach(\sigma)$       &  4.80 & 99.06 & 46.79 & 90.50 & 12.12 & 97.71 & 13.44 & 97.62 &  1.04 & 99.70 & 12.83 & 97.61 & 15.17 & 97.03 \\   
                                        & + $\approach(m)$            &  4.56 & 99.10 & 45.10 & 90.80 & 10.30 & 97.93 & 12.34 & 97.78 &  1.02 & 99.71 & 11.33 & 97.82 & 14.11 & 97.19 \\   
                                        & + $\approach(\texttt{md})$  & 80.58 & 86.08 & 84.81 & 72.82 & 76.89 & 84.83 & 84.50 & 77.21 & 52.92 & 89.48 & 73.54 & 85.92 & 75.54 & 82.72 \\   
                                        & + $\approach(e)$            &  7.39 & 98.59 & 49.40 & 90.11 & 20.22 & 96.78 & 20.98 & 96.45 &  3.82 & 99.23 & 19.24 & 96.89 & 20.17 & 96.34 \\   
            \bottomrule
            \end{tabular}
            }
            \caption{\textit{Detailed results of post-hoc methods combined with $\approach$ on six OOD benchmarks: SVHN, Places365, iSUN, Textures, LSUN-c, and LSUN-r using ResNet-18 trained on CIFAR-10. $\boldsymbol{\uparrow}$ indicates higher is better; $\boldsymbol{\downarrow}$ indicates lower is better. The symbols denote the statistic used: $\mu$ (mean), $\sigma$ (std. deviation), $m$ (maximum), $\texttt{md}$ (median), and $e$ (Shannon entropy).}
            \label{table: combination_ResNet-18_CIFAR-10}}
        \end{table*}
    \end{landscape}

    \begin{landscape}
        \begin{table*}[ht]
        \centering
        \resizebox{\linewidth}{!}{
        \begin{tabular}{ l l  cc cc cc cc cc cc cc }
            \toprule
             \multirow{2}{*}{\textbf{Model}} & \multirow{2}{*}{\textbf{Combined Method}} & \multicolumn{2}{c}{\textbf{SVHN}} & \multicolumn{2}{c}{\textbf{Place365}} & \multicolumn{2}{c}{\textbf{iSUN}} & \multicolumn{2}{c}{\textbf{Textures}} & \multicolumn{2}{c}{\textbf{LSUN-c}}  & \multicolumn{2}{c}{\textbf{LSUN-r}} & \multicolumn{2}{c}{\textbf{Average}}  \\
            \cmidrule(lr){3-4} \cmidrule(lr){5-6} \cmidrule(lr){7-8} \cmidrule(lr){9-10} \cmidrule(lr){11-12} \cmidrule(lr){13-14} \cmidrule(lr){15-16}
            && \textbf{FPR95} $\downarrow$ & \textbf{AUROC} $\uparrow$ & \textbf{FPR95} $\downarrow$ & \textbf{AUROC} $\uparrow$ & \textbf{FPR95} $\downarrow$ & \textbf{AUROC} $\uparrow$ & \textbf{FPR95} $\downarrow$ & \textbf{AUROC} $\uparrow$ & \textbf{FPR95} $\downarrow$ & \textbf{AUROC} $\uparrow$ & \textbf{FPR95} $\downarrow$ & \textbf{AUROC} $\uparrow$ & \textbf{FPR95} $\downarrow$ & \textbf{AUROC} $\uparrow$ \\
            \midrule
            \multirow{36}{*}{ResNet-18} & MSP	                      & 74.26 & 83.20 & 82.37 & 75.31 & 84.13 & 71.57 & 85.04 & 74.02 & 70.79 & 82.78 & 82.96 & 73.10 & 79.92 & 76.66 \\
                                        & + $\approach(\mu)$          & 67.71 & 85.20 & 83.18 & 68.78 & 83.14 & 71.63 & 80.32 & 77.35 & 54.85 & 90.25 & 81.69 & 72.16 & 75.15 & 77.56 \\
                                        & + $\approach(\sigma)$	      & 56.50 & 90.40 & 82.87 & 70.28 & 80.36 & 79.05 & 71.63 & 84.67 & 48.21 & 92.39 & 78.98 & 78.88 & 69.76 & 82.61 \\
                                        & + $\approach(m)$	          & 55.56 & 90.74 & 82.24 & 71.81 & 79.93 & 79.93 & 71.83 & 84.85 & 46.07 & 92.94 & 78.50 & 79.67 & 69.02 & 83.32 \\
                                        & + $\approach(\texttt{md})$  & 81.67 & 67.33 & 83.60 & 66.01 & 87.54 & 52.52 & 87.06 & 62.24 & 67.72 & 80.46 & 85.95 & 54.73 & 82.26 & 63.88 \\    	      
                                        & + $\approach(e)$            & 72.43 & 84.84 & 82.29 & 74.78 & 83.52 & 75.54 & 83.99 & 78.22 & 68.74 & 85.00 & 82.42 & 76.35 & 78.90 & 79.12 \\ 
            \cmidrule(lr){2-16}
                                        & Energy                      & 66.64 & 89.53 & 81.39 & 76.83 & 71.46 & 83.02 & 85.18 & 75.68 & 48.01 & 91.63 & 68.57 & 84.53 & 70.21 & 83.54 \\
                                        & + $\approach(\mu)$          & 31.13 & 95.02 & 81.53 & 76.00 & 64.83 & 85.24 & 62.06 & 85.32 & 16.45 & 97.17 & 61.59 & 86.00 & 52.93 & 87.46 \\
                                        & + $\approach(\sigma)$       & 20.60 & 96.54 & 82.09 & 75.57 & 55.69 & 88.63 & 54.61 & 87.27 & 10.36 & 98.19 & 54.42 & 88.87 & 46.29 & 89.18  \\
                                        & + $\approach(m)$	          & 19.94 & 96.66 & 81.83 & 76.16 & 55.48 & 88.74 & 54.66 & 87.38 &  9.47 & 98.37 & 54.35 & 88.89 & 45.96 & 89.37 \\
                                        & + $\approach(\texttt{md})$  & 86.35 & 81.98 & 82.54 & 74.78 & 81.34 & 76.38 & 87.39 & 71.44 & 35.48 & 93.53 & 78.69 & 78.37 & 75.30 & 79.41 \\    	      
                                        & + $\approach(e)$            & 59.49 & 90.88 & 81.10 & 77.00 & 68.86 & 84.21 & 82.27 & 77.55 & 41.10 & 92.92 & 65.83 & 85.56 & 66.44 & 84.69 \\ 
            \cmidrule(lr){2-16}
                                        & ReAct	                      & 56.62 & 91.69 & 80.38 & 77.28 & 53.40 & 89.25 & 57.27 & 88.63 & 49.29 & 90.69 & 49.59 & 90.27 & 57.76 & 87.97 \\
                                        & + $\approach(\mu)$          & 19.43 & 96.65 & 85.03 & 73.97 & 50.51 & 89.41 & 31.76 & 93.30 & 16.52 & 96.91 & 48.33 & 89.70 & 41.93 & 89.99 \\
                                        & + $\approach(\sigma)$       & 12.01 & 97.78 & 84.81 & 73.90 & 38.70 & 92.53 & 28.69 & 93.87 &  8.56 & 98.30 & 38.15 & 92.51 & 35.15 & 91.48 \\
                                        & + $\approach(m)$	          & 11.47 & 97.85 & 83.96 & 74.67 & 38.04 & 92.72 & 28.58 & 93.96 &  7.65 & 98.46 & 38.23 & 92.55 & 34.66 & 91.70 \\
                                        & + $\approach(\texttt{md})$  & 94.03 & 75.11 & 86.93 & 70.50 & 84.67 & 75.27 & 76.97 & 78.69 & 48.19 & 89.96 & 82.43 & 76.92 & 78.87 & 77.74 \\    	      
                                        & + $\approach(e)$            & 41.81 & 93.69 & 79.65 & 77.65 & 51.03 & 89.82 & 52.52 & 89.67 & 37.41 & 93.16 & 47.55 & 90.74 & 51.66 & 89.12 \\ 
            \cmidrule(lr){2-16}
                                        & DICE	                      & 40.89 & 92.97 & 81.33 & 76.23 & 62.61 & 85.83 & 75.28 & 76.29 & 12.44 & 97.65 & 61.39 & 86.84 & 55.66 & 85.97 \\
                                        & + $\approach(\mu)$          & 18.07 & 96.70 & 85.71 & 73.54 & 63.43 & 87.39 & 50.48 & 87.32 & 7.72  & 98.53 & 63.17 & 87.40 & 48.10 & 88.48 \\
                                        & + $\approach(\sigma)$	      & 17.98 & 96.65 & 87.65 & 70.66 & 52.29 & 90.21 & 49.82 & 87.43 & 6.52  & 98.77 & 54.88 & 89.98 & 44.86 & 88.95 \\
                                        & + $\approach(m)$	          & 17.13 & 96.73 & 86.35 & 72.10 & 50.76 & 90.77 & 50.04 & 87.40 & 5.63  & 98.91 & 53.29 & 90.42 & 43.87 & 89.39 \\
                                        & + $\approach(\texttt{md})$  & 91.36 & 73.12 & 88.97 & 66.71 & 91.18 & 68.13 & 86.83 & 64.19 & 31.66 & 93.47 & 90.88 & 69.56 & 80.15 & 72.53 \\    	      
                                        & + $\approach(e)$            & 32.64 & 94.27 & 81.06 & 76.22 & 58.24 & 87.49 & 70.66 & 78.98 & 9.96  & 98.11 & 57.36 & 88.25 & 51.65 & 87.22 \\ 
            \cmidrule(lr){2-16}
                                        & ReAct+DICE                  & 34.16 & 94.18 & 83.57 & 74.79 & 54.50 & 89.85 & 52.96 & 87.36 & 10.40 & 97.95 & 53.78 & 90.22 & 48.23 & 89.06 \\
                                        & + $\approach(\mu)$          & 24.94 & 95.56 & 89.97 & 68.73 & 66.15 & 86.65 & 37.75 & 89.95 & 13.29 & 97.43 & 68.14 & 86.26 & 50.04 & 87.43 \\
                                        & + $\approach(\sigma)$	      & 21.45 & 95.77 & 90.46 & 65.11 & 55.33 & 88.99 & 39.57 & 89.35 & 9.84  & 98.00 & 59.02 & 88.56 & 45.95 & 87.63 \\
                                        & + $\approach(m)$	          & 20.28 & 96.00 & 89.36 & 67.31 & 53.37 & 89.95 & 39.68 & 89.69 & 8.55  & 98.26 & 57.50 & 89.37 & 44.79 & 88.43 \\
                                        & + $\approach(\texttt{md})$  & 96.81 & 61.79 & 93.08 & 58.01 & 95.51 & 60.58 & 84.40 & 65.93 & 45.31 & 88.46 & 95.54 & 61.27 & 85.11 & 66.00 \\    	      
                                        & + $\approach(e)$            & 27.53 & 95.18 & 83.83 & 74.35 & 49.97 & 91.02 & 47.41 & 88.90 & 8.75  & 98.25 & 50.10 & 91.19 & 44.60 & 89.82 \\ 
            \cmidrule(lr){2-16}
                                        & ASH	                      & 22.00 & 96.16 & 86.10 & 69.25 & 64.55 & 84.17 & 37.87 & 91.77 & 23.39 & 95.57 & 63.19 & 84.25 & 49.52 & 86.86 \\
                                        & + $\approach(\mu)$          & 17.35 & 96.98 & 83.85 & 72.96 & 64.02 & 85.53 & 40.44 & 91.61 & 18.42 & 96.74 & 61.84 & 85.85 & 47.65 & 88.28 \\
                                        & + $\approach(\sigma)$	      & 12.61 & 97.77 & 84.80 & 72.26 & 53.65 & 89.25 & 37.20 & 92.12 & 9.87  & 98.23 & 53.33 & 89.29 & 41.91 & 89.82 \\
                                        & + $\approach(m)$	          & 11.99 & 97.84 & 83.71 & 73.24 & 52.11 & 89.63 & 37.25 & 92.09 & 8.90  & 98.39 & 51.84 & 89.57 & 40.97 & 90.13 \\
                                        & + $\approach(\texttt{md})$  & 62.56 & 88.69 & 85.74 & 69.88 & 80.30 & 76.19 & 66.83 & 82.06 & 28.96 & 94.35 & 77.19 & 77.42 & 66.93 & 81.43 \\    	      
                                        & + $\approach(e)$            & 24.94 & 96.03 & 82.40 & 75.43 & 63.05 & 86.60 & 52.87 & 88.78 & 23.38 & 96.15 & 60.69 & 87.30 & 51.22 & 88.38 \\
            \cmidrule(lr){2-16}
                                        & \texttt{SCALE}                 & 22.12 & 96.38 & 81.96 & 74.95 & 61.62 & 86.65 & 44.50 & 90.72 & 18.62 & 96.78 & 59.76 & 86.74 & 48.10 & 88.70 \\
                                        & + $\approach(\mu)$             & 18.05 & 96.77 & 86.73 & 69.73 & 62.42 & 85.29 & 36.51 & 91.75 & 15.39 & 97.04 & 62.41 & 84.95 & 46.92 & 87.59 \\
                                        & + $\approach(\sigma)$          & 12.08 & 97.67 & 85.87 & 70.50 & 51.94 & 88.95 & 32.80 & 92.64 &  9.75 & 98.11 & 53.54 & 88.47 & 41.00 & 89.39 \\
                                        & + $\approach(m)$               & 11.65 & 97.72 & 85.21 & 71.04 & 51.65 & 89.03 & 32.66 & 92.71 &  9.09 & 98.21 & 53.29 & 88.45 & 40.59 & 89.53 \\
                                        & + $\approach(\texttt{md})$     & 61.84 & 88.20 & 86.79 & 67.62 & 79.72 & 75.40 & 60.34 & 83.57 & 27.00 & 94.26 & 77.43 & 75.95 & 65.52 & 80.83 \\
                                        & + $\approach(e)$               & 19.13 & 96.79 & 81.84 & 74.91 & 58.96 & 87.54 & 41.77 & 91.37 & 16.32 & 97.16 & 57.53 & 87.53 & 45.93 & 89.22 \\
            \bottomrule
            \end{tabular}
            }
            \caption{\textit{Detailed results of post-hoc methods combined with $\approach$ on six OOD benchmarks: SVHN, Places365, iSUN, Textures, LSUN-c, and LSUN-r using ResNet-18 trained on $\gamma(\mathbf{x})$. $\boldsymbol{\uparrow}$ indicates higher is better; $\boldsymbol{\downarrow}$ indicates lower is better. The symbols denote the statistic used: $\mu$ (mean), $\sigma$ (std. deviation), $m$ (maximum), $\texttt{md}$ (median), and $e$ (Shannon entropy).}}
            \label{table: combination_ResNet-18_CIFAR-100}
        \end{table*}
    \end{landscape}

    \begin{landscape}
        \begin{table*}[ht]
        \centering
        \resizebox{\linewidth}{!}{
        \begin{tabular}{ l l  cc cc cc cc cc cc cc }
            \toprule
             \multirow{2}{*}{\textbf{Model}} & \multirow{2}{*}{\textbf{Combined Method}} & \multicolumn{2}{c}{\textbf{SVHN}} & \multicolumn{2}{c}{\textbf{Place365}} & \multicolumn{2}{c}{\textbf{iSUN}} & \multicolumn{2}{c}{\textbf{Textures}} & \multicolumn{2}{c}{\textbf{LSUN-c}}  & \multicolumn{2}{c}{\textbf{LSUN-r}} & \multicolumn{2}{c}{\textbf{Average}}  \\
            \cmidrule(lr){3-4} \cmidrule(lr){5-6} \cmidrule(lr){7-8} \cmidrule(lr){9-10} \cmidrule(lr){11-12} \cmidrule(lr){13-14} \cmidrule(lr){15-16}
            && \textbf{FPR95} $\downarrow$ & \textbf{AUROC} $\uparrow$ & \textbf{FPR95} $\downarrow$ & \textbf{AUROC} $\uparrow$ & \textbf{FPR95} $\downarrow$ & \textbf{AUROC} $\uparrow$ & \textbf{FPR95} $\downarrow$ & \textbf{AUROC} $\uparrow$ & \textbf{FPR95} $\downarrow$ & \textbf{AUROC} $\uparrow$ & \textbf{FPR95} $\downarrow$ & \textbf{AUROC} $\uparrow$ & \textbf{FPR95} $\downarrow$ & \textbf{AUROC} $\uparrow$ \\
            \midrule
            \multirow{36}{*}{DenseNet-101} & MSP	                  & 64.76 & 88.33 & 60.30 & 88.55 & 33.57 & 95.41 & 56.67 & 90.17 & 23.41 & 96.75 & 33.87 & 95.37 & 45.43 & 92.43 \\
                                        & + $\approach(\mu)$          & 29.11 & 95.90 & 54.33 & 86.64 & 10.82 & 98.13 & 30.73 & 93.72 & 1.15  & 99.59 & 12.19 & 97.99 & 23.06 & 95.33 \\
                                        & + $\approach(\sigma)$	      & 13.81 & 97.59 & 50.90 & 89.94 & 9.03  & 98.47 & 18.00 & 97.08 & 2.09  & 99.48 & 10.46 & 98.32 & 17.38 & 96.81 \\
                                        & + $\approach(m)$	          & 13.24 & 97.59 & 50.14 & 90.02 & 9.17  & 98.42 & 16.91 & 97.15 & 2.59  & 99.40 & 10.49 & 98.26 & 17.09 & 96.81 \\
                                        & + $\approach(\texttt{md})$  & 88.99 & 59.01 & 73.67 & 72.19 & 54.88 & 88.98 & 89.11 & 55.92 & 8.64  & 97.86 & 51.84 & 89.42 & 61.19 & 77.23 \\    	      
                                        & + $\approach(e)$            & 63.89 & 90.76 & 59.99 & 89.29 & 32.43 & 95.72 & 55.53 & 92.04 & 22.36 & 96.98 & 32.83 & 95.68 & 44.50 & 93.41 \\ 
            \cmidrule(lr){2-16}
                                        & Energy                      & 37.91 & 93.59 & 36.42 & 92.38 &  7.33 & 98.27 & 43.87 & 90.48 & 1.95 & 99.47 &  6.97 & 98.38 & 22.41 & 95.43 \\
                                        & + $\approach(\mu)$          & 15.12 & 97.51 & 33.93 & 92.75 &  3.32 & 98.98 & 25.71 & 94.88 & 0.61 & 99.79 & 3.70 & 98.92 & 13.73 & 97.14 \\
                                        & + $\approach(\sigma)$       & 10.95 & 98.11 & 32.51 & 92.99 &  1.73 & 99.47 & 18.26 & 96.34 & 0.26 & 99.90 & 1.88 & 99.43 & 10.93 & 97.71 \\
                                        & + $\approach(m)$	          & 10.86 & 98.13 & 31.69 & 93.16 &  1.64 & 99.48 & 17.93 & 96.51 & 0.30 & 99.89 & 1.83 & 99.43 & 10.71 & 97.77 \\
                                        & + $\approach(\texttt{md})$  & 82.84 & 79.69 & 63.50 & 84.13 & 41.24 & 94.40 & 85.67 & 72.20 & 3.12 & 99.23 & 38.13 & 94.80 & 52.42 & 87.41 \\    	      
                                        & + $\approach(e)$            & 36.00 & 93.90 & 35.84 & 92.49 &  6.36 & 98.37 & 42.18 & 90.89 & 1.81 & 99.50 &  6.30 & 98.46 & 21.42 & 95.60 \\ 
            \cmidrule(lr){2-16}
                                        & ReAct	                      & 23.18 & 96.28 & 33.96 & 92.97 &  5.56 & 98.49 & 32.23 & 93.98 & 2.47 & 99.33 &  5.37 & 98.59 & 17.13 & 96.61 \\
                                        & + $\approach(\mu)$          &  5.82 & 98.76 & 31.59 & 93.50 &  2.87 & 99.15 & 16.91 & 96.83 & 0.91 & 99.75 &  3.32 & 99.09 & 10.24 & 97.85 \\
                                        & + $\approach(\sigma)$       &  5.82 & 98.83 & 30.35 & 93.71 &  1.49 & 99.54 & 11.26 & 97.78 & 0.34 & 99.88 &  1.69 & 99.51 &  8.49 & 98.21 \\
                                        & + $\approach(m)$	          &  5.86 & 98.86 & 29.97 & 93.89 &  1.49 & 99.55 & 11.06 & 97.88 & 0.48 & 99.87 &  1.68 & 99.51 &  8.42 & 98.26 \\
                                        & + $\approach(\texttt{md})$  & 85.79 & 78.26 & 65.55 & 82.53 & 49.13 & 92.97 & 88.55 & 69.17 & 3.63 & 99.09 & 45.31 & 93.49 & 56.33 & 85.92 \\    	      
                                        & + $\approach(e)$            & 21.19 & 96.49 & 33.40 & 93.11 &  5.06 & 98.59 & 30.64 & 94.32 & 2.30 & 99.38 &  4.99 & 98.68 & 16.26 & 96.76 \\ 
            \cmidrule(lr){2-16}
                                        & DICE	                      & 16.66 & 96.98 & 37.59 & 92.04 &  2.31 & 99.42 & 27.98 & 92.71 & 0.15 & 99.94 &  2.44 & 99.36 & 14.52 & 96.74 \\
                                        & + $\approach(\mu)$          &  6.23 & 98.80 & 44.96 & 91.11 &  2.35 & 99.30 & 20.18 & 95.57 & 0.07 & 99.94 &  2.67 & 99.22 & 12.74 & 97.32 \\
                                        & + $\approach(\sigma)$	      &  5.33 & 98.95 & 41.71 & 91.84 &  1.88 & 99.48 & 15.28 & 96.95 & 0.09 & 99.95 &  2.06 & 99.43 & 11.06 & 97.77 \\
                                        & + $\approach(m)$	          &  4.95 & 99.01 & 39.57 & 92.20 &  1.58 & 99.53 & 14.08 & 97.18 & 0.11 & 99.95 &  1.88 & 99.48 & 10.36 & 97.89 \\
                                        & + $\approach(\texttt{md})$  & 83.46 & 75.99 & 74.72 & 78.40 & 53.58 & 92.11 & 86.81 & 67.01 & 2.24 & 99.42 & 53.74 & 92.19 & 59.09 & 84.19 \\    	      
                                        & + $\approach(e)$            & 15.57 & 97.17 & 37.28 & 92.17 &  2.15 & 99.44 & 26.81 & 93.08 & 0.14 & 99.95 &  2.27 & 99.39 & 14.04 & 96.86 \\ 
            \cmidrule(lr){2-16}
                                        & ReAct+DICE                  &  4.60 & 99.02 & 35.94 & 92.91 &  1.78 & 99.51 & 17.07 & 96.78 & 0.12 & 99.95 &  2.02 & 99.47 & 10.26 & 97.94 \\
                                        & + $\approach(\mu)$          &  2.68 & 99.30 & 48.65 & 90.81 &  2.76 & 99.20 & 14.27 & 97.29 & 0.09 & 99.93 &  3.19 & 99.13 & 11.94 & 97.61 \\
                                        & + $\approach(\sigma)$	      &  3.78 & 99.21 & 43.99 & 91.70 &  2.02 & 99.44 & 10.80 & 98.01 & 0.11 & 99.93 &  2.38 & 99.39 & 10.51 & 97.95 \\
                                        & + $\approach(m)$	          &  3.59 & 99.24 & 42.01 & 92.03 &  1.83 & 99.49 &  9.72 & 98.17 & 0.15 & 99.93 &  2.24 & 99.45 &  9.92 & 98.05\\
                                        & + $\approach(\texttt{md})$  & 85.94 & 75.87 & 78.12 & 75.30 & 57.04 & 91.52 & 88.16 & 65.34 & 2.45 & 99.38 & 56.79 & 91.52 & 61.42 & 83.16 \\    	      
                                        & + $\approach(e)$            &  4.49 & 99.06 & 35.95 & 93.02 &  1.73 & 99.52 & 16.38 & 96.99 & 0.13 & 99.95 &  1.99 & 99.49 & 10.11 & 98.01\\ 
            \cmidrule(lr){2-16}
                                        & ASH	                      &  5.18 & 98.90 & 42.80 & 90.42 &  2.97 & 99.27 & 15.80 & 97.04 & 0.45 & 99.80 &  3.06 & 99.25 & 11.71 & 97.44 \\
                                        & + $\approach(\mu)$          &  7.02 & 98.61 & 38.76 & 91.80 &  2.72 & 99.22 & 17.98 & 96.62 & 0.29 & 99.85 &  3.18 & 99.15 & 11.66 & 97.54 \\
                                        & + $\approach(\sigma)$	      &  5.31 & 98.91 & 35.92 & 92.56 &  1.70 & 99.48 & 12.85 & 97.61 & 0.26 & 99.89 &  1.92 & 99.44 &  9.66 & 97.98 \\
                                        & + $\approach(m)$	          &  5.23 & 98.95 & 35.13 & 92.77 &  1.56 & 99.50 & 12.32 & 97.73 & 0.32 & 99.89 &  1.89 & 99.46 &  9.41 & 98.05\\
                                        & + $\approach(\texttt{md})$  & 68.46 & 84.08 & 63.48 & 83.16 & 31.93 & 95.40 & 76.86 & 75.56 & 1.83 & 99.56 & 30.24 & 95.54 & 45.47 & 88.88 \\    	      
                                        & + $\approach(e)$            &  4.85 & 98.93 & 42.53 & 90.54 &  2.86 & 99.29 & 15.21 & 97.15 & 0.44 & 99.80 &  3.01 & 99.27 & 11.48 & 97.50\\
            \cmidrule(lr){2-16}
                                        & \texttt{SCALE}                 & 29.23 & 95.23 & 37.86 & 92.14 &  6.71 & 98.46 & 36.99 & 92.28 & 1.71 & 99.50 &  6.80 & 98.48 & 19.88 & 96.01 \\
                                        & + $\approach(\mu)$             & 11.85 & 97.95 & 35.98 & 92.36 &  3.68 & 98.94 & 22.91 & 95.56 & 0.65 & 99.75 &  3.94 & 98.86 & 13.17 & 97.23 \\
                                        & + $\approach(\sigma)$          &  8.99 & 98.36 & 34.54 & 92.59 &  1.88 & 99.38 & 16.74 & 96.74 & 0.32 & 99.87 &  2.17 & 99.32 & 10.77 & 97.71 \\
                                        & + $\approach(m)$               &  8.96 & 98.38 & 34.02 & 92.76 &  1.87 & 99.38 & 16.47 & 96.89 & 0.38 & 99.86 &  2.23 & 99.32 & 10.66 & 97.76 \\
                                        & + $\approach(\texttt{md})$     & 69.80 & 84.20 & 54.23 & 87.24 & 23.01 & 96.29 & 75.48 & 76.37 & 1.76 & 99.51 & 18.90 & 96.73 & 40.53 & 90.06 \\
                                        & + $\approach(e)$               & 28.81 & 95.31 & 37.44 & 92.21 &  6.51 & 98.48 & 36.60 & 92.41 & 1.63 & 99.51 &  6.61 & 98.50 & 19.60 & 96.07 \\                            
            \bottomrule
            \end{tabular}
            }
            \caption{\textit{Detailed results of post-hoc methods combined with $\approach$ on six OOD benchmarks: SVHN, Places365, iSUN, Textures, LSUN-c, and LSUN-r using DenseNet-101 trained on CIFAR-10. $\boldsymbol{\uparrow}$ indicates higher is better; $\boldsymbol{\downarrow}$ indicates lower is better. The symbols denote the statistic used: $\mu$ (mean), $\sigma$ (std. deviation), $m$ (maximum), $\texttt{md}$ (median), and $e$ (Shannon entropy).}}
            \label{table: combination_DenseNet-101_CIFAR-10}
        \end{table*}
    \end{landscape}

    \begin{landscape}
        \begin{table*}[ht]
        \centering
        \resizebox{\linewidth}{!}{
        \begin{tabular}{ l l  cc cc cc cc cc cc cc }
            \toprule
             \multirow{2}{*}{\textbf{Model}} & \multirow{2}{*}{\textbf{Combined Method}} & \multicolumn{2}{c}{\textbf{SVHN}} & \multicolumn{2}{c}{\textbf{Place365}} & \multicolumn{2}{c}{\textbf{iSUN}} & \multicolumn{2}{c}{\textbf{Textures}} & \multicolumn{2}{c}{\textbf{LSUN-c}}  & \multicolumn{2}{c}{\textbf{LSUN-r}} & \multicolumn{2}{c}{\textbf{Average}}  \\
            \cmidrule(lr){3-4} \cmidrule(lr){5-6} \cmidrule(lr){7-8} \cmidrule(lr){9-10} \cmidrule(lr){11-12} \cmidrule(lr){13-14} \cmidrule(lr){15-16}
            && \textbf{FPR95} $\downarrow$ & \textbf{AUROC} $\uparrow$ & \textbf{FPR95} $\downarrow$ & \textbf{AUROC} $\uparrow$ & \textbf{FPR95} $\downarrow$ & \textbf{AUROC} $\uparrow$ & \textbf{FPR95} $\downarrow$ & \textbf{AUROC} $\uparrow$ & \textbf{FPR95} $\downarrow$ & \textbf{AUROC} $\uparrow$ & \textbf{FPR95} $\downarrow$ & \textbf{AUROC} $\uparrow$ & \textbf{FPR95} $\downarrow$ & \textbf{AUROC} $\uparrow$ \\
            \midrule
            \multirow{36}{*}{DenseNet-101} & MSP	                  & 81.38 & 75.71 & 82.68 & 74.06 & 82.52 & 70.50 & 87.11 & 68.39 & 51.82 & 87.93 & 79.31 & 72.21 & 77.47 & 74.80\\
                                        & + $\approach(\mu)$          & 67.10 & 85.98 & 82.63 & 72.87 & 76.61 & 77.17 & 73.87 & 80.77 & 22.08 & 96.58 & 74.63 & 77.39 & 66.15 & 81.79 \\
                                        & $\approach(\sigma)$	      & 61.17 & 88.41 & 82.09 & 73.45 & 73.98 & 80.04 & 64.43 & 86.18 & 22.22 & 96.42 & 72.44 & 79.91 & 62.72 & 84.07 \\
                                        & + $\approach(m)$	          & 60.15 & 88.82 & 81.77 & 74.22 & 73.23 & 80.90 & 62.82 & 86.96 & 22.51 & 96.35 & 71.72 & 80.72 & 62.03 & 84.66 \\
                                        & + $\approach(\texttt{md})$  & 94.34 & 47.23 & 86.86 & 66.68 & 93.52 & 45.23 & 97.41 & 30.76 & 40.44 & 90.25 & 90.66 & 50.25 & 83.87 & 55.07 \\    	      
                                        & + $\approach(e)$            & 78.98 & 79.30 & 82.44 & 74.49 & 81.08 & 74.98 & 84.65 & 75.59 & 46.99 & 90.11 & 77.79 & 75.99 & 75.32 & 78.41 \\ 
            \cmidrule(lr){2-16}
                                        & Energy                      & 70.99 & 86.66 & 77.28 & 76.94 & 59.39 & 85.68 & 83.49 & 67.47 & 11.45 & 97.89 & 50.90 & 88.57 & 58.92 & 83.87\\
                                        & + $\approach(\mu)$          & 22.45 & 96.11 & 78.72 & 77.16 & 48.77 & 89.75 & 52.09 & 83.58 &  1.54 & 99.68 & 44.92 & 90.44 & 41.42 & 89.45 \\
                                        & + $\approach(\sigma)$       & 21.13 & 96.30 & 78.19 & 77.16 & 42.78 & 91.48 & 44.34 & 86.19 &  1.18 & 99.72 & 40.25 & 92.02 & 37.98 & 90.48 \\
                                        & + $\approach(m)$	          & 19.90 & 96.45 & 77.30 & 77.67 & 41.02 & 91.81 & 42.48 & 87.12 &  1.28 & 99.70 & 38.78 & 92.25 & 36.79 & 90.83 \\
                                        & + $\approach(\texttt{md})$  & 98.09 & 53.27 & 88.48 & 68.06 & 96.62 & 55.42 & 99.40 & 30.52 & 20.27 & 95.70 & 93.82 & 61.85 & 82.78 & 60.81 \\    	      
                                        & + $\approach(e)$            & 58.57 & 89.86 & 76.92 & 77.59 & 53.03 & 88.03 & 76.86 & 72.19 &  7.68 & 98.66 & 45.28 & 90.32 & 53.06 & 86.11 \\ 
            \cmidrule(lr){2-16}
                                        & ReAct	                      & 69.82 & 86.30 & 79.23 & 74.09 & 41.50 & 92.40 & 72.09 & 80.38 & 18.14 & 96.26 & 36.53 & 93.64 & 52.89 & 87.18\\
                                        & + $\approach(\mu)$          & 11.73 & 97.67 & 83.17 & 74.06 & 25.66 & 95.16 & 26.12 & 93.93 &  1.69 & 99.52 & 27.77 & 94.98 & 29.36 & 92.56 \\
                                        & + $\approach(\sigma)$       & 14.13 & 97.32 & 83.87 & 74.36 & 25.15 & 95.51 & 23.21 & 94.55 &  1.55 & 99.55 & 26.41 & 95.37 & 29.05 & 92.78 \\
                                        & + $\approach(m)$	          & 13.70 & 97.36 & 83.00 & 75.14 & 23.26 & 95.83 & 21.68 & 94.95 &  2.07 & 99.44 & 24.63 & 95.64 & 28.06 & 93.06 \\
                                        & + $\approach(\texttt{md})$  & 99.29 & 36.35 & 92.19 & 60.93 & 98.68 & 41.22 & 99.70 & 19.61 & 41.86 & 91.88 & 97.16 & 47.82 & 88.15 & 49.63 \\    	      
                                        & + $\approach(e)$            & 46.52 & 91.94 & 78.18 & 75.18 & 25.51 & 95.11 & 49.66 & 87.42 & 10.23 & 97.95 & 23.32 & 95.79 & 38.90 & 90.56 \\ 
            \cmidrule(lr){2-16}
                                        & DICE	                      & 32.93 & 94.09 & 79.90 & 75.43 & 35.50 & 92.50 & 64.84 & 71.95 & 1.93 & 99.57 & 30.81 & 93.96 & 40.98 & 87.92 \\
                                        & + $\approach(\mu)$          & 16.64 & 96.95 & 84.89 & 74.34 & 33.11 & 94.02 & 46.44 & 84.49 & 1.14 & 99.65 & 32.25 & 94.27 & 35.74 & 90.62 \\
                                        & $\approach(\sigma)$	      & 17.27 & 96.77 & 83.48 & 74.87 & 32.32 & 93.97 & 48.19 & 83.15 & 1.15 & 99.67 & 30.98 & 94.38 & 35.57 & 90.47 \\
                                        & + $\approach(m)$	          & 17.95 & 96.62 & 82.44 & 75.04 & 32.04 & 93.84 & 48.97 & 82.26 & 1.11 & 99.67 & 30.23 & 94.34 & 35.46 & 90.30 \\
                                        & + $\approach(\texttt{md})$  & 97.75 & 56.78 & 90.67 & 65.03 & 97.62 & 54.76 & 99.27 & 29.22 & 9.86 & 97.71 & 95.27 & 60.53 & 81.74 & 60.67 \\    	      
                                        & + $\approach(e)$            & 25.98 & 95.41 & 79.38 & 77.04 & 37.18 & 92.38 & 56.77 & 77.17 & 0.80 & 99.75 & 33.23 & 93.48 & 38.89 & 89.21 \\ 
            \cmidrule(lr){2-16}
                                        & ReAct+DICE                  & 25.10 & 95.70 & 84.17 & 73.56 & 27.98 & 95.06 & 41.79 & 87.82 &  1.06 & 99.70 & 27.76 & 95.16 & 34.64 & 91.17 \\
                                        & + $\approach(\mu)$          & 16.40 & 96.93 & 86.87 & 72.82 & 31.87 & 94.55 & 31.65 & 92.18 &  1.41 & 99.56 & 33.77 & 94.34 & 33.66 & 91.73 \\
                                        & $\approach(\sigma)$	      & 17.69 & 96.80 & 84.49 & 74.41 & 30.06 & 94.83 & 33.94 & 91.27 &  0.97 & 99.69 & 30.73 & 94.84 & 32.98 & 91.97 \\
                                        & + $\approach(m)$	          & 17.50 & 96.83 & 84.73 & 74.34 & 30.34 & 94.81 & 33.76 & 91.33 &  0.99 & 99.68 & 31.03 & 94.80 & 33.06 & 91.97 \\
                                        & + $\approach(\texttt{md})$  & 98.37 & 49.23 & 92.03 & 60.76 & 98.36 & 45.84 & 99.40 & 25.45 & 11.36 & 97.48 & 96.42 & 51.22 & 82.66 & 55.00 \\    	      
                                        & + $\approach(e)$            & 23.39 & 96.01 & 83.50 & 74.62 & 28.59 & 95.03 & 40.64 & 89.05 &  1.06 & 99.73 & 28.83 & 95.13 & 34.34 & 91.60 \\ 
            \cmidrule(lr){2-16}
                                        & ASH	                      & 10.32 & 97.99 & 85.80 & 71.97 & 37.68 & 92.45 & 35.48 & 91.77 & 5.43 & 98.98 & 40.35 & 91.96 & 35.84 & 90.85 \\
                                        & + $\approach(\mu)$          &  9.00 & 98.12 & 87.63 & 70.70 & 38.40 & 92.42 & 27.70 & 93.91 & 5.38 & 98.94 & 42.25 & 91.72 & 35.06 & 90.97 \\
                                        & $\approach(\sigma)$	      &  8.82 & 98.16 & 86.33 & 71.57 & 37.41 & 92.38 & 29.27 & 93.53 & 5.43 & 98.95 & 40.76 & 91.78 & 34.67 & 91.06 \\
                                        & + $\approach(m)$	          & 10.63 & 97.85 & 87.59 & 70.89 & 38.11 & 92.67 & 26.45 & 94.16 & 4.40 & 99.09 & 42.01 & 92.02 & 34.87 & 91.11\\
                                        & + $\approach(\texttt{md})$  & 67.65 & 85.31 & 90.45 & 67.10 & 89.48 & 71.59 & 90.98 & 59.95 & 4.17 & 99.08 & 86.53 & 73.09 & 71.54 & 76.02 \\    	      
                                        & + $\approach(e)$            &  9.50 & 98.08 & 85.59 & 71.97 & 35.82 & 92.84 & 32.34 & 92.69 & 4.89 & 99.07 & 38.53 & 92.33 & 34.45 & 91.16 \\ 
            \cmidrule(lr){2-16}
                                        & \texttt{SCALE}                 & 16.26 & 97.05 & 78.54 & 76.97 & 43.56 & 91.21 & 45.60 & 87.23 & 3.23 & 99.30 & 42.69 & 91.02 & 38.31 & 90.46 \\
                                        & + $\approach(\mu)$             & 10.37 & 97.96 & 83.56 & 74.76 & 42.16 & 91.85 & 33.35 & 92.48 & 1.84 & 99.52 & 44.44 & 91.07 & 35.95 & 91.27 \\
                                        & + $\approach(\sigma)$          & 11.08 & 97.85 & 83.49 & 74.93 & 38.60 & 92.77 & 30.53 & 93.09 & 1.58 & 99.57 & 41.10 & 92.17 & 34.40 & 91.73 \\
                                        & + $\approach(m)$               & 10.79 & 97.90 & 83.16 & 75.44 & 37.69 & 93.03 & 29.08 & 93.56 & 1.89 & 99.52 & 40.53 & 92.39 & 33.86 & 91.97 \\
                                        & + $\approach(\texttt{md})$     & 82.79 & 78.26 & 87.59 & 70.09 & 92.87 & 64.90 & 95.57 & 50.40 & 5.80 & 98.79 & 89.44 & 68.02 & 75.68 & 71.75 \\
                                        & + $\approach(e)$               & 14.85 & 97.27 & 78.26 & 77.29 & 41.02 & 91.93 & 42.84 & 88.33 & 2.86 & 99.37 & 40.71 & 91.70 & 36.76 & 90.98 \\                            
            \bottomrule
            \end{tabular}
            }
            \caption{\textit{Detailed results of post-hoc methods combined with $\approach$ on six OOD benchmarks: SVHN, Places365, iSUN, Textures, LSUN-c, and LSUN-r using DenseNet-101 trained on CIFAR-100. $\boldsymbol{\uparrow}$ indicates higher is better; $\boldsymbol{\downarrow}$ indicates lower is better. The symbols denote the statistic used: $\mu$ (mean), $\sigma$ (std. deviation), $m$ (maximum), $\texttt{md}$ (median), and $e$ (Shannon entropy).}}
            \label{table: combination_DenseNet-101_CIFAR-100}
        \end{table*}
    \end{landscape}

\end{document}